\documentclass{article}

\usepackage[main, final,nonatbib]{neurips_2026}
\usepackage[utf8]{inputenc} % allow utf-8 input
\usepackage[T1]{fontenc}    % use 8-bit T1 fonts
\usepackage{hyperref}       % hyperlinks
\usepackage{url}            % simple URL typesetting
\usepackage{booktabs}       % professional-quality tables
\usepackage{amsfonts}       % blackboard math symbols
\usepackage{nicefrac}       % compact symbols for 1/2, etc.
\usepackage{microtype}      % microtypography
\usepackage{xcolor}         % colors

\usepackage{amsmath}
\usepackage{amssymb}
\usepackage{mathtools}
\usepackage{amsthm}
\usepackage{algorithm}
\usepackage{algorithmic}
\usepackage{graphicx}
\usepackage{subcaption}
\usepackage{rotating}
\usepackage{placeins}

\usepackage[numbers]{natbib}
\usepackage{cleveref}
\let\cref\Cref
\usepackage{thmtools}
\usepackage{thm-restate}

\newcommand{\calF}{\mathcal{F}}

\newcommand{\calW}{\mathcal{W}}
\newcommand{\E}{\mathbb{E}}
\newcommand{\var}{\mathrm{var}}
\newcommand{\bbP}{\mathbb{P}}
\newcommand{\bbQ}{\mathbb{Q}}

\newcommand{\bbR}{\mathbb{R}}
\newcommand{\bbE}{\mathbb{E}}

\newcommand{\bbH}{\mathbb{H}}

\newcommand{\iid}{i.i.d.\ }

\newcommand{\Wassone}{\calW_1}
\newcommand{\Ptrue}{\bbP^\ast}
\newcommand{\mutrue}{\mu^\ast}
\newcommand{\Wnopt}{W_n^\ast(X_{1:n},\mu)}
\newcommand{\lambdaopt}{\lambda^\ast}
\newcommand{\Wn}{W_n(X_{1:n},\mu)}
\newcommand{\CS}{C_{\alpha,n}(X_{1:n})}
\newcommand{\lambdahat}{\widehat\lambda}
\newcommand{\predictive}{\bbQ_{n|n-1}(\cdot\mid X_{1:n-1})}

\DeclareMathOperator{\KL}{KL}

\newtheorem{proposition}{Proposition}

\newtheorem{theorem}{Theorem}
\newtheorem{lemma}{Lemma}

\title{Asymptotically Log-Optimal Bayes-Assisted Confidence Sequences for Bounded Means}

\author{%
  Valentin Kilian\thanks{Equal contribution. Order decided by the outcome of 10 judo fights.} \\
  Department of Statistics,\\
  University of Oxford\\
  \texttt{kilian@stats.ox.ac.uk} \\
   \And
  Stefano Cortinovis\footnotemark[1]\\
  Department of Statistics,\\
  University of Oxford\\
  \texttt{cortinovis@stats.ox.ac.uk} \\
  \And
  François Caron \\
  Department of Statistics,\\
  University of Oxford\\
  \texttt{caron@stats.ox.ac.uk} \\
}

\begin{document}

\maketitle

%%%%%%%%%%%%%%%%%%%%%%%%%%%%%%%%%%%%%%%%%%%%%%%%%%%%%%%%%%%%%%%%%%%%%%%%%%%%%%%
% ABSTRACT
%%%%%%%%%%%%%%%%%%%%%%%%%%%%%%%%%%%%%%%%%%%%%%%%%%%%%%%%%%%%%%%%%%%%%%%%%%%%%%%

\begin{abstract}
  Confidence sequences based on test martingales provide time-uniform uncertainty quantification for the mean of bounded \iid observations without parametric distributional assumptions.
  Their practical efficiency, however, depends strongly on the choice of martingale updates, and many existing constructions do not exploit prior information about plausible data-generating distributions or mean values.
  We propose a Bayes-assisted framework that uses a Bayesian working predictive model to adaptively construct confidence sequences.
  For each candidate mean and time point, the predictive distribution selects, among valid one-step martingale factors, the update maximising predictive expected log-growth;
  validity is therefore preserved even when the prior or working model is misspecified.
  We prove that if the predictive distribution is Wasserstein-consistent, the resulting procedure is asymptotically log-optimal, matching the per-sample log-growth of an oracle procedure with access to the true distribution.
  We instantiate the framework using robust predictives based on Dirichlet-process mixtures and Bayesian exponentially tilted empirical likelihood.
  Experiments on synthetic data, sequential best-arm identification for LLM evaluation, and prediction-powered inference show that informative priors can substantially reduce confidence-sequence width and sampling effort while retaining anytime-valid coverage.
\end{abstract}

%%%%%%%%%%%%%%%%%%%%%%%%%%%%%%%%%%%%%%%%%%%%%%%%%%%%%%%%%%%%%%%%%%%%%%%%%%%%%%%
% MAIN
%%%%%%%%%%%%%%%%%%%%%%%%%%%%%%%%%%%%%%%%%%%%%%%%%%%%%%%%%%%%%%%%%%%%%%%%%%%%%%%

% !TEX root = bacs-main.tex

\section{Introduction}
\label{sec:introduction}

Modern data analysis is often sequential: online experiments are monitored
during data collection, adaptive benchmarks allocate samples to promising
systems, and bandit algorithms act before a fixed terminal sample size is
known. Repeated inspection and data-dependent stopping can invalidate ordinary
fixed-time confidence intervals. The appropriate object is a \emph{confidence
sequence}, a sequence $(C_{\alpha,n})_{n\geq 1}$ satisfying
\begin{equation}
\label{eq:coverage}
\Ptrue\left(
  \mutrue \in C_{\alpha,n} \text{ for all } n\geq 1
\right)
\geq 1-\alpha,
\end{equation}
where $\mutrue$ is the mean of the data-generating distribution $\Ptrue$. We
study the nonparametric bounded-mean problem: $X_1,X_2,\ldots$ are \iid in
$[0,1]$, $\mutrue=\E_{\Ptrue}[X_1]$, and no parametric form is assumed.

Confidence sequences can be constructed by sequentially testing each candidate
mean $\mu$. For each $\mu$, one builds a nonnegative evidence process that is
controlled if $\mu=\mutrue$ and tends to grow otherwise; values not yet
rejected form the confidence sequence. Test martingales and e-processes
justify this construction through Ville's inequality
\citep{Ville1939,Darling1967,Lai1976,Howard2021,WaudbySmith2024,RamdasWang2025}.
For bounded observations, this can be done with predictable linear factors
$1+\lambda_i(X_i-\mu)$, constrained to be nonnegative on $[0,1]$. The
constraints give validity; the choice of $\lambda_i$ determines efficiency.

The key practical question is how to choose these predictable coefficients.
The Growth Rate Adaptive to the Particular Alternative (GRAPA) strategy
\citep{WaudbySmith2024a} uses the empirical distribution of past observations
as a forecast for the next one. We follow the same principle, but replace the
purely empirical forecast by a Bayesian or pseudo-Bayesian predictive
distribution, allowing prior information from previous experiments, expert
knowledge, simulations, or structure to guide the evidence process.

This differs from fully parametric Bayesian sequential testing. Under a
specified likelihood, mixture likelihood ratios or Bayes factors yield valid
sequential tests under that model
\citep{Kumon2008,Kumon2011,KaufmannKoolen2021,Cortinovis2025a}. Here, the
Bayesian model is used only to choose the next coefficient. The one-step factor
itself remains in a distribution-free bounded-mean martingale family. Hence
validity holds for any predictable working predictive distribution, even under
model or prior misspecification. Poor prior information may widen the
confidence sequence, but cannot by itself invalidate \eqref{eq:coverage}.

We propose \emph{Bayes-assisted confidence sequences} for bounded means. At
each time and candidate mean $\mu$, a working predictive distribution is
computed from the past and used to choose the next linear coefficient. This
separates \emph{validity}, which follows from martingales and Ville's
inequality, from \emph{efficiency}, which depends on predictive quality. We
show that if the working predictive learns the true distribution in
Wasserstein distance, then the method is asymptotically log-optimal: for each
false candidate mean, it attains the same limiting per-sample growth rate as
an oracle method that knew $\Ptrue$.

We instantiate the framework with predictives that use prior information early
while becoming data-driven over time. A parametric Bayesian predictive is
effective when its working model is accurate. For robustness, we also study a
mixture Dirichlet-process predictive, interpolating between a parametric
Bayesian predictive and the empirical distribution
\citep{Ferguson1973,Blackwell1973,Antoniak1974}; without the prior component,
it recovers the empirical plug-in idea underlying GRAPA. Finally, we develop
Bayesian exponentially tilted empirical likelihood and regularized variants,
which place prior information on the mean parameter of interest without specifying a full likelihood
\citep{Owen2001,Lazar2003,Schennach2005,Kim2026}.

\looseness=-1
Our contributions are as follows.
\begin{itemize}
\item We introduce a Bayes-assisted framework in which
Bayesian or pseudo-Bayesian predictives guide bounded-mean confidence
sequences.
\looseness=-1
\item We prove nonasymptotic time-uniform coverage for any
predictable working predictive distribution, including misspecified models and
inaccurate priors.
\looseness=-1
\item  We show asymptotic log-optimality when the predictive
distribution learns the true data-generating distribution.
\looseness=-1
\item  We propose asymptotically log-optimal constructions
based on mixture Dirichlet-process, BETEL, and RETEL predictives.
\looseness=-1
\item We demonstrate empirically that informative priors can
reduce confidence-sequence width and sampling effort, while misspecification
affects efficiency rather than validity.
\end{itemize}

The remainder of the paper is organised as follows. \Cref{sec:relatedwork}
discusses related work. \Cref{sec:background} reviews bounded-mean martingales
and confidence sequences by test inversion.
\Cref{sec:predictive_assisted_test_martingales} introduces the generic
predictive-assisted construction and efficiency theory.
\Cref{sec:predictive_constructions} develops the parametric, mixture
Dirichlet-process, BETEL, and RETEL predictives. \Cref{sec:experiments}
studies synthetic examples, best-arm identification for LLM evaluation, and
prediction-powered inference. \Cref{sec:extensions,sec:discussion} discuss
extensions and limitations. The appendices contain proofs, implementation
details, prior specifications, and additional experiments.

% !TEX root = bacs-main.tex

\section{Related work}
\label{sec:relatedwork}

\textbf{Prior-assisted fixed-time inference.}
Fixed-time confidence intervals for bounded means are often obtained by inverting Hoeffding, Bennett--Bernstein, Bentkus, or empirical-Bernstein inequalities \citep{Hoeffding1963,Bennett1962,Bernstein1927,Bentkus2004,Audibert2007,Maurer2009}. A separate line asks how prior or side information can shorten frequentist intervals while preserving coverage, from Pratt's expected-length-optimal intervals to modern ``frequentist, assisted by Bayes'' (FAB) procedures \citep{Pratt1961,Pratt1963,Brown1995,Yu2018,Farchione2008,Kabaila2013,Kabaila2022,Hoff2019,Hoff2023,Cortinovis2024,Cortinovis2025b}. These methods are primarily parametric and fixed-sample-size; our target is a nonparametric, anytime-valid confidence sequence for bounded means.

\textbf{Confidence sequences, e-processes, and betting.}
Confidence sequences originate in \citet{Darling1967,Robbins1970a,Robbins1970,Lai1976}, building on martingale and sequential-testing foundations \citep{Ville1939,Wald1945}. Modern treatments use nonnegative supermartingales, e-values, and e-processes to obtain validity under optional stopping and continuation \citep{Howard2021,WaudbySmith2024,Shafer2001,Shafer2019,Ramdas2023,RamdasWang2025,Gruenwald2024}. For bounded means, betting martingales are especially relevant: \citet{WaudbySmith2024a} introduced betting CIs and CSs, with GRAPA choosing stakes by empirical log-growth optimization, while \citet{Orabona2023} obtained PRECiSE sequences via universal-portfolio regret bounds \citep{Cover1991,Cover1996}. Near-optimality and optimality of betting confidence sets have since been studied by \citet{ShekharRamdas2023,Clerico2025}, and closed-form empirical-Bernstein sequences give strong variance-adaptive baselines \citep{ChuggRamdas2025}. Our method follows the Kelly log-growth principle \citep{Kelly1956}, but replaces GRAPA's empirical predictive with a Bayesian working predictive; validity still follows solely from the bounded-mean martingale and Ville's inequality.

\textbf{Bayesian and empirical-likelihood predictives.}
Bayesian sequential tests, mixture likelihood ratios, and Bayes factors yield martingales when a probabilistic model is fully specified, and have been used to encode prior information in parametric or model-based sequential inference \citep{Kumon2008,Kumon2011,KaufmannKoolen2021,Cortinovis2025a}. The nonparametric Bayes-assisted confidence sequences of \citet{Kilian2025} are related, but their guarantees are asymptotic, whereas ours are nonasymptotic for bounded observations. We instantiate the betting rule using Dirichlet-process posterior predictives \citep{Ferguson1973,Blackwell1973,Antoniak1974} and Bayesian empirical or exponentially tilted empirical likelihood \citep{Owen2001,Lazar2003,Schennach2005,Chen2008,Emerson2009,Zhu2009,Liu2010,Kim2026}. These Bayesian objects are used only to select predictable stakes, not to assert Bayesian posterior coverage.

% !TEX root = bacs-main.tex

\section{Background on bounded-mean test martingales and confidence sequences}
\label{sec:background}

This section gives the general construction used throughout the paper, see e.g. \cite{RamdasWang2025}.  Let $X_1,X_2,\ldots$ be \iid random variables in $[0,1]$ from $\Ptrue$ with mean $\mutrue\in(0,1)$. Let $\calF_n=\sigma(X_1,\ldots,X_n)$. For each candidate mean
$\mu\in(0,1)$, consider the composite null
hypothesis
\begin{equation}
\label{eq:mean_null}
H_\mu:\mutrue=\mu.
\end{equation}

\paragraph{Bounded-mean test martingales.} For $\mu\in(0,1)$ and a parameter $c\in(0,1)$, let
\begin{equation}
\label{eq:I_mu_c}
I_{\mu,c}:=\left[-\frac{c}{1-\mu},\frac{c}{\mu}\right].
\end{equation}
For any predictable (that is, $\calF_{i-1}$--measurable) sequence of coefficients
$\lambdahat_i(X_{1:i-1},\mu)\in I_{\mu,c}$, define
\begin{equation}
\label{Wn}
\Wn
=
\prod_{i=1}^n
\left(1+\lambdahat_i(X_{1:i-1},\mu)(X_i-\mu)\right),
\qquad W_0(\mu)=1 .
\end{equation}
The interval $I_{\mu,c}$ is chosen so that each one-step factor in
\eqref{Wn} is nonnegative for all possible observations $x\in[0,1]$. In the betting interpretation, $W_n$ is the bettor's wealth and
$\lambdahat_i(X_{1:i-1},\mu)$ is the predictable betting coefficient,
chosen before observing $X_i$, on the centered payoff $X_i-\mu$; such bet is fair under $H_\mu$. Such nonnegative multiplicative wealth processes are the
standard linear betting martingales used for bounded-mean confidence
sequences \citep{WaudbySmith2024a,Clerico2025}.

\begin{restatable}[Test-martingale property]{proposition}{testmarginaleproperty}
\label{prop:test_martingale_property}
Fix $\mu\in(0,1)$ and $c\in(0,1)$.  If
$\lambdahat_i(X_{1:i-1},\mu)$ is predictable and belongs to
$I_{\mu,c}$ for every $i\geq1$, then $(W_n(X_{1:n},\mu))_{n\geq0}$ is
nonnegative.  Moreover, under the null hypothesis $H_\mu$ in
\eqref{eq:mean_null}, $(W_n(X_{1:n},\mu))_{n\geq0}$ is a martingale with
initial value one.
\end{restatable}

\paragraph{Confidence sequence and anytime-valid coverage.} For each $\mu\in(0,1)$ and each $\alpha\in(0,1)$, the process $W_n(X_{1:n},\mu)$ can be thresholded at level $1/\alpha$ to obtain an anytime-valid sequential test for $H_\mu$ with type-I error $\alpha$. Inversion of this collection of tests yields the confidence sequence
\begin{equation}
\label{cs}
\CS
=
\left\{
\mu\in(0,1):\; \Wn \leq \frac{1}{\alpha}
\right\}
\subseteq [0,1].
\end{equation}

\begin{restatable}[Anytime-valid coverage]{proposition}{anytimevalidity}
\label{prop:anytime_validity}
For any $\alpha\in(0,1)$, $(\CS)_{n\geq 1}$ in \cref{cs} is a $1-\alpha$ confidence sequence for the mean parameter $\mu^*$. That is, it satisfies \cref{eq:coverage}.

\end{restatable}

\paragraph{Oracle log-optimality.} The anytime-coverage result above holds for every predictable choice of the
coefficients.  The remaining question is how to choose them so that the
sequential tests have high power against false candidate means. For a fixed
candidate $\mu\in(0,1)$, define the expected log-evidence of a constant
coefficient $\lambda\in I_{\mu,c}$ under the true distribution $\Ptrue$ by
\begin{equation}
\label{eq:M_oracle}
M(\lambda,\mu)
:=
\E_{X_1\sim\Ptrue}
\left[
\log\left(1+\lambda(X_1-\mu)\right)
\right].
\end{equation}
If $\Ptrue$ were known, the oracle coefficient would be
\begin{equation}
\label{optimallambda}
\lambdaopt(\mu)
=
{\arg\max}_{\lambda\in I_{\mu,c}}
M(\lambda,\mu),
\end{equation}
and the corresponding oracle test martingale would be
\begin{equation}
\label{eq:oracle}
\Wnopt
=
\prod_{i=1}^n
\left(1+\lambdaopt(\mu)(X_i-\mu)\right).
\end{equation}
If $\mu=\mutrue$, then $\lambdaopt(\mu)=0$ and the oracle log-growth is zero. For a candidate mean $\mu\neq \mutrue$, $M(\lambdaopt(\mu),\mu)$ is the oracle
per-observation log-evidence against $H_\mu$ within the class
\eqref{Wn}.  Since $\Ptrue$ is unknown, $\lambdaopt(\mu)$ is unavailable.  We
call a data-dependent construction asymptotically log-optimal if it
matches this oracle log-evidence asymptotically, in the sense that
\begin{equation}
\label{eq:logoptimal}
\frac{1}{n}\E_{\Ptrue}\left[\log \Wn\right]
\to
M(\lambdaopt(\mu),\mu).
\end{equation}

\section{Predictive-Assisted Test Martingales for Bounded Means}
\label{sec:predictive_assisted_test_martingales}

We now describe our generic predictive-based approach to choose the predictable betting coefficient $(\widehat \lambda_i)$ in \cref{Wn} and state a simple sufficient condition for the associated test martingales to be asymptotically log-optimal.

Consider a sequence of $\calF_{n-1}$--measurable predictive distributions $\bbQ_{n\mid n-1}(\cdot\mid X_{1:n-1})$, $n\geq1$, on $[0,1]$.
 For a fixed candidate mean $\mu$, the $n$th predictive distribution is used to choose the next
betting coefficient by maximising the conditional expected logarithm of
the $n$th factor of $W_n$ in \cref{Wn}:
\begin{equation}
\label{eq:bayes_assisted_lambda}
\lambdahat_n(X_{1:n-1},\mu)
\in
{\arg\max}_{\lambda\in I_{\mu,c}}
\left\{
\E_{\widetilde X_n\sim \predictive}
\left[
\log\left(1+\lambda(\widetilde X_n-\mu)\right)
\right]
\right\}.
\end{equation}
When the maximiser is not unique, we fix an arbitrary measurable selection.  A
process of the form \eqref{Wn}, with coefficients chosen by
\eqref{eq:bayes_assisted_lambda}, is called a \emph{predictive-assisted test
martingale}. The role of $\bbQ_{n\mid n-1}$ in \eqref{eq:bayes_assisted_lambda} is purely
algorithmic: it guides the choice of a one-step test factor.  The predictive
model is not assumed to be correctly specified for the purpose of coverage: anytime coverage holds for any choice of predictive sequence.

The next theorem shows that the predictive-assisted construction is asymptotically log-optimal
whenever the sequence of predictive distributions consistently estimates $\Ptrue$ in
Wasserstein distance. Let $\Wassone$ denote the Wasserstein-1 distance on probability distributions
on $[0,1]$.

\begin{restatable}[Asymptotic log-optimality]{theorem}{thmlogoptimality}
\label{thm:logoptimality}
Let $(X_{n})_{n\geq 1}$ be \iid in $[0,1]$ from a non-degenerate
distribution $\Ptrue$, and let $\calF_n=\sigma(X_1,\ldots,X_n)$. Let
$(\bbQ_{n\mid n-1}(\cdot\mid X_{1:n-1}))_{n\geq 1}$ be an $\calF_{n-1}$-measurable sequence of
predictive distributions on $[0,1]$ such that
\begin{align}
\Wassone(\bbQ_{n\mid n-1}(\cdot\mid X_{1:n-1}),\Ptrue)
\to 0
\qquad \text{in probability}.
\label{assump:weakconvpred}
\end{align}
Fix $\mu\in(0,1)$ and $c\in(0,1)$.  Define
\begin{align}
M(\lambda,\mu)
&:=
\E_{X_1\sim\Ptrue}
\left[
\log\left(1+\lambda(X_1-\mu)\right)
\right],
\label{eq:M_def_section3}
\\
M_n(\lambda,X_{1:n-1},\mu)
&:=
\E_{X'_n\sim{\bbQ}_{n\mid n-1}}
\left[
\log\left(1+\lambda(X'_{n}-\mu)\right)
\mid X_{1:n-1}
\right].
\label{eq:Mn_def_section3}
\end{align}
Let $\lambdahat_n(X_{1:n-1},\mu)$ be an $\calF_{n-1}$-measurable maximiser of
$\lambda\mapsto M_n(\lambda,X_{1:n-1},\mu)$ on $I_{\mu,c}$.  Then
$\lambda\mapsto M(\lambda,\mu)$ admits a unique maximiser
$\lambdaopt(\mu)$ over $I_{\mu,c}$, and
\begin{align}
\widehat\lambda_{n}(X_{1:n-1},\mu)
\rightarrow
\lambdaopt(\mu)
\qquad \text{in probability}.
\label{eq:lambda_as}
\end{align}
Moreover, for
$W_n(X_{1:n},\mu)=\prod_{i=1}^n
(1+\widehat\lambda_i(X_{1:i-1},\mu)(X_i-\mu))$,
\begin{align}
\frac{1}{n}\log W_n(X_{1:n},\mu)
&\to
M(\lambdaopt(\mu),\mu)
\qquad\text{in probability},
\label{eq:Wn_proba}
\\
\frac{1}{n}\E_{\Ptrue}\left[\log W_n(X_{1:n},\mu)\right]
&\to
M(\lambdaopt(\mu),\mu).
\label{eq:Wn_L1}
\end{align}
\end{restatable}

\Cref{prop:anytime_validity} gives anytime-valid
coverage for any sequence of predictive distributions.  In contrast,
\Cref{thm:logoptimality} states that, when the predictive distributions
converge to $\Ptrue$, the induced sequential tests attain the same first-order
log-evidence as the oracle construction that knows $\Ptrue$.

The following finite-sample comparison makes explicit how the loss relative to
the oracle is controlled by the predictive error.

\begin{theorem}[Predictive error and expected log-evidence]
\label{thm:finitn}
With the notation of \Cref{thm:logoptimality},
\begin{equation}
M(\lambdaopt(\mu),\mu)
-
\frac{1}{n}\E_{\Ptrue}\left[\log W_n(X_{1:n},\mu)\right]
\leq
2L_{\mu,c}
\times
\frac{1}{n}
\sum_{i=1}^n
\E_{\Ptrue}\left[
\Wassone\left(
\bbQ_{i\mid i-1}(\cdot\mid X_{1:i-1}),\Ptrue
\right)
\right],
\end{equation}
where $L_{\mu,c}=\frac{c}{(1-c)\min(\mu,1-\mu)}<\infty .$
\end{theorem}

The proposed method can therefore use any sequence of predictive distributions
satisfying \eqref{assump:weakconvpred}.  In the next section, we describe
constructions of $\predictive$ based on Bayesian nonparametric and
empirical-likelihood ideas.

% !TEX root = bacs-main.tex

\section{Bayesian and pseudo-Bayesian predictive constructions}
\label{sec:predictive_constructions}

\Cref{prop:anytime_validity} implies that anytime coverage is
obtained for any predictable sequence of predictive distributions
$\bbQ_{n\mid n-1}(\cdot\mid X_{1:n-1})$ on $[0,1]$.  The predictive
model is used only to choose the one-step betting coefficient in
\Cref{eq:bayes_assisted_lambda}.  \Cref{thm:logoptimality}
then shows that the same construction is
asymptotically log-optimal whenever \cref{assump:weakconvpred} holds.

This section gives several choices of $\bbQ_{n\mid n-1}$. A first natural approach is to take a fully Bayesian approach to this problem, by specifying a parametric or nonparametric prior distribution on the unknown $\Ptrue$. $\bbQ_{n\mid n-1}$ is then the posterior predictive distribution under this Bayesian model. A parametric version is presented in \cref{subsec:parametric_predictives}, and a log-optimal nonparametric approach in \cref{subsec:dirichlet_process}.

In many applications, prior information may be available for the mean parameter of interest $\mu^*$, but the elicitation of a prior for the unknown distribution $\Ptrue$ is challenging. To this end, we consider a Bayesian empirical likelihood in \cref{subsec:betel}, which bypasses the need to fully specify a Bayesian model. This approach is also shown to be asymptotically log-optimal.

\subsection{Parametric Bayesian predictive}
\label{subsec:parametric_predictives}

The most direct construction starts from a parametric working model.  Let
$G_0(\cdot;\phi)$ be a probability distribution on $[0,1]$ indexed by
$\phi\in\Phi$, and let $\pi(d\phi)$ be a prior on $\Phi$.  The posterior
predictive distribution is
\begin{equation}
\label{eq:param}
\bbQ_{n+1\mid n}(dx\mid X_{1:n})
=
\int_\Phi G_0(dx;\phi)\,\pi(d\phi\mid X_{1:n}),
\end{equation}
with $\bbQ_{1\mid0}(dx)=\int_\Phi G_0(dx;\phi)\,\pi(d\phi)$. While such approach has anytime coverage, and can provide shorter confidence sequences for small sample sizes, it need not be asymptotically log-optimal if misspecified. This motivates the use of a Bayesian nonparametric approach, as introduced in the next section.

\subsection{Mixture Dirichlet process-style predictive}
\label{subsec:dirichlet_process}

Let $\bbP_n=\frac{1}{n}\sum_{i=1}^n\delta_{X_i}$ denote the empirical distribution, and
\[
\bbH_{n+1\mid n}(dx\mid X_{1:n})
=
\int_\Phi G_0(dx;\phi)\,\pi(d\phi\mid X_{1:n})
\]
be a parametric posterior predictive distribution of the form
\eqref{eq:param}.  For $\kappa>0$, define
\begin{equation}
\label{eq:dirichlet}
\bbQ_{n+1\mid n}(dx\mid x_{1:n})
=
\frac{\kappa}{\kappa+n}\,
\bbH_{n+1\mid n}(dx\mid X_{1:n})
+
\frac{n}{\kappa+n}\,\bbP_n(dx),
\qquad n\geq1.
\end{equation}
The form in \eqref{eq:dirichlet} is inspired by the Pólya-urn posterior
predictive of a Dirichlet process \citep{Ferguson1973,Blackwell1973}.
When $G_0$ is non-atomic and all observations are distinct, it corresponds to the posterior predictive in Antoniak's mixture of Dirichlet processes~\cite{Antoniak1974}. We call \eqref{eq:dirichlet} a mixture-DP-style (MDP) predictive. It can also be viewed as a robust interpolation between the parametric
predictive and the empirical distribution. For $\kappa=0$, one recovers GRAPA~\citet{WaudbySmith2024a}.

The parameter $\kappa$ controls how quickly the construction moves from the
prior-driven component to the data-driven component.  Large $\kappa$ gives more
weight to the parametric predictive at early sample sizes, while small
$\kappa$ moves quickly toward the empirical distribution.  For every fixed
$\kappa$, the empirical component has weight tending to one.  Thus the
construction can exploit informative prior information early, while recovering
predictive consistency even if the parametric likelihood is misspecified.
This type of prior-data interpolation is closely related to robust Bayesian
nonparametric and general Bayesian ideas for mitigating model misspecification
\citep{Walker2013,Bissiri2016,Lyddon2018}.

\begin{proposition}[Predictive consistency of the MDP predictive]
\label{prop:wass_rate}
Suppose $X_1,X_2,\ldots$ are \iid from a distribution $\Ptrue$ supported on
$[0,1]$.  Let $\bbQ_{n+1\mid n}$ be defined by \eqref{eq:dirichlet}, with the
parametric component $\bbH_{n+1\mid n}$ also supported on $[0,1]$.  Then
\begin{equation}
\label{eq:wass_rate_dirichlet}
\E_{\Ptrue}\!\left[
\Wassone\!\left(
\bbQ_{n+1\mid n}(\cdot\mid X_{1:n}),\Ptrue
\right)
\right]
=
O(n^{-1/2}).
\end{equation}
Consequently, $\Wassone\!\left(
\bbQ_{n+1\mid n}(\cdot\mid X_{1:n}),\Ptrue
\right)
\to 0$ in $\Ptrue$-probability.

\end{proposition}

By \Cref{thm:logoptimality}, the Bayes-assisted confidence sequence based on
the MDP predictive is therefore asymptotically log-optimal.

\subsection{Bayesian empirical likelihood predictive}
\label{subsec:betel}

The MDP predictive in \Cref{eq:dirichlet} is robust to likelihood
misspecification, but still requires a prior model for the full
data-generating distribution.  When prior information is available only for
the mean, we instead use a Bayesian empirical likelihood: the prior is placed
on \(\mu\), while the likelihood is estimated from the moment restriction
\(\E_{\Ptrue}[X-\mu]=0\).  We consider Bayesian exponentially tilted empirical
likelihood (BETEL) \citep{Lazar2003,Schennach2005} and a regularized version
(RETEL), related to adjusted and regularized empirical likelihood methods
\citep{Chen2008,Emerson2009,Liu2010,Zhu2009,Kim2026}.

Let \(\pi\) be a prior on \([0,1]\).  For each \(\mu\in(0,1)\), let
\(\bbH_\mu\) be a regularizing distribution on \([0,1]\), and let
\(\tau_n\geq0\) with \(\tau_n=o(n)\).  Recall that $\bbP_n$ is the empirical distribution, and define
$
    \bbP_{n,\mu}
    =
    \frac{n}{n+\tau_n}\bbP_n
    +
    \frac{\tau_n}{n+\tau_n}\bbH_\mu.
$
For a candidate mean \(\mu\), consider as an empirical likelihood distribution
\begin{equation}
\label{eq:retel-projection}
    \widetilde{\bbP}_{n,\mu}
    =
    \arg\min_{P\ll \bbP_{n,\mu}}
    \left\{
        \KL(P\,\|\,\bbP_{n,\mu})
        :
        \int (z-\mu)\,dP(z)=0
    \right\}.
\end{equation}
Thus \(\widetilde{\bbP}_{n,\mu}\) is the closest tilt of the regularized
empirical measure whose mean is exactly \(\mu\).  Writing
$
    \psi_{n,\mu}(\gamma)
    =
    \log\int \exp\{\gamma(z-\mu)\}\,d\bbP_{n,\mu}(z),
$
the solution is
\begin{equation}
%\label{eq:retel-density}
    \frac{d\widetilde{\bbP}_{n,\mu}}{d\bbP_{n,\mu}}(z)
    =
    \exp\{\gamma_{n,\mu}(z-\mu)-\psi_{n,\mu}(\gamma_{n,\mu})\},
    \qquad
    \psi_{n,\mu}'(\gamma_{n,\mu})=0 .
\label{eq:retel-gamma}
\end{equation}
The case \(\tau_n=0\) gives BETEL, with likelihood set to zero outside the interior of the
empirical convex hull.  If \(\tau_n>0\) and
\(\operatorname{conv}(\operatorname{supp}\bbH_\mu)=[0,1]\), the construction is
defined for every \(\mu\in(0,1)\). Let
\[
    M_{\bbH_\mu}(\gamma)=\int e^{\gamma(z-\mu)}\,d\bbH_\mu(z),
    \qquad
    Z_{n,\mu}
    =
    \sum_{j=1}^n e^{\gamma_{n,\mu}(X_j-\mu)}
    +
    \tau_n M_{\bbH_\mu}(\gamma_{n,\mu}),
\]
and define the empirical-likelihood weights $w_{n,i}(\mu) =
    e^{\gamma_{n,\mu}(X_i-\mu)}/Z_{n,\mu},$
    $i=1,\ldots,n$.
For the two-point regularizer used below, the corresponding
empirical-likelihood contribution is, up to a \(\mu\)-free factor,
$
    L_n(\mu)=\prod_{i=1}^n w_{n,i}(\mu).
$
The pseudo-posterior and pseudo-predictive are
\begin{equation}
    \label{eq:retel-posterior_and_predictive}
    \Pi_n(d\mu)
    =
    \frac{L_n(\mu)\pi(d\mu)}
    {\int_0^1 L_n(u)\pi(du)},
    \qquad
%\label{eq:retel-predictive}
    \bbQ_{n+1\mid n}(A\mid X_{1:n})
    =
    \int_0^1 \widetilde{\bbP}_{n,\mu}(A)\,\Pi_n(d\mu).
\end{equation}
The next proposition verifies the Wasserstein consistency condition needed
for \Cref{thm:logoptimality}; its proof is given in
\Cref{app:proof-retel}.

\begin{restatable}[Consistency of the RETEL predictive]{proposition}{propretel}
\label{prop:retel}
Assume that \(X_1,X_2,\ldots\) are \iid from a non-degenerate distribution
\(\Ptrue\) supported on \([0,1]\), with mean \(\mutrue\in(0,1)\) and variance
\(\sigma^2>0\).  Assume that \(\pi\) is absolutely continuous on \([0,1]\),
with density bounded away from zero in a neighbourhood of \(\mutrue\), that
\(\tau_n=o(n)\), and that each \(\bbH_\mu\) is supported on \([0,1]\) with
\(\operatorname{conv}(\operatorname{supp}\bbH_\mu)=[0,1]\).  Then, for every
\(\epsilon>0\),
\begin{equation}
    \Pi_n\bigl(|\mu-\mutrue|>\epsilon\bigr)\to0,\quad\text{and}\quad    \Pi_n\bigl(|\gamma_{n,\mu}|>\epsilon\bigr)\to0,
    \label{eq:retel-consistency}
\end{equation}
in probability.  Moreover, in probability, $ \Wassone\bigl(\bbQ_{n+1\mid n}(\cdot\mid X_{1:n}),\Ptrue\bigr)\to0.$
\end{restatable}

By \Cref{thm:logoptimality}, the Bayes-assisted confidence sequence based on
the RETEL predictive is asymptotically log-optimal.  In the experiments we use
two cases: unregularized BETEL, \(\tau_n=0\), and the two-point RETEL
regularizer $\tau_n=1$, $\bbH_\mu=\bbH=\frac12(\delta_0+\delta_1)$. The latter acts as two endpoint half-pseudo-observations, stabilizing the empirical
likelihood while remaining asymptotically data-dominated.  Computational
details are given in \Cref{app:betel-computation}.

% !TEX root = bacs-main.tex

\section{Experiments}
\label{sec:experiments}

We evaluate our Bayes-assisted CS methods on synthetic data and two real-world applications, showing that informative working predictives can shorten confidence sequences and reduce stopping times while preserving anytime validity under misspecification.
The Bayes-assisted methods use the robust predictives from \cref{sec:predictive_constructions}: the mixture Dirichlet process (MDP) predictive \eqref{eq:dirichlet} with a beta base distribution, BETEL \eqref{eq:retel-posterior_and_predictive} with $\tau_n=0$, and RETEL \eqref{eq:retel-posterior_and_predictive} with $\tau_n=1$ and the two-point regulariser.
We compare against the empirical predictive and a parametric beta predictive within the same framework, and against prior-free baselines outside it: Hedged-Capital (HC) \citep{WaudbySmith2024a}, PRECiSE-CO96 \citep{Orabona2023}, and the full Cover--Ordentlich universal portfolio (UP) method \citep{Cover1996}.
The main text reports the proposed methods together with the empirical predictive and UP, which we find to be the strongest prior-free baselines in our experiments; full comparisons with HC, PRECiSE-CO96, and parametric Bayes are deferred to \cref{app:additional_results}.
We report CS width and, in decision-making experiments, stopping time; in synthetic experiments we also report empirical cumulative miscoverage and widths normalised by the oracle CS whose coefficient maximises expected log-growth under the true distribution.
Unless otherwise specified, we set $\alpha=0.1$, invert tests on the 499 interior points of the $1/500$ grid on $[0,1]$, and use $c=0.95$ for the empirical and Bayes-assisted methods;
method- and baseline-specific details are provided in \cref{app:implementation_details}.

\paragraph{Synthetic experiments.}

We first consider synthetic data, where $\Ptrue$ is known and the oracle coefficient $\lambdaopt(\mu)$ in \cref{optimallambda}, and hence the oracle confidence sequence $C_{\alpha,n}^\star$, can be computed.
This lets us assess how close each method comes to the log-optimal benchmark suggested by \eqref{eq:oracle}.
We generate \iid observations on $[0,1]$ from six distributions: $\mathrm{Bernoulli}(p)$ with $p\in\{0.1,0.5\}$, $\mathrm{Beta}(a,b)$ with $(a,b)\in\{(0.5,0.5),(1,1),(10,30)\}$, and the asymmetric mixture $0.25\,\mathrm{Beta}(5,15) + 0.75\,\mathrm{Beta}(15,5)$.
The Bernoulli and beta-mixture examples test departures from the beta working model used by the parametric and MDP predictives.
In the main text, we report the informative-prior comparison for the mixture example.
The remaining synthetic distributions under informative priors are shown in \cref{app:synthetic}, while the non-informative and misspecified-prior regimes are considered in \cref{app:additional_priors}.
Detailed prior specifications are provided in \cref{app:prior_specifications_synthetic}.

\Cref{fig:simulation_beta_mixture_main} compares methods on the beta-mixture experiment in terms of oracle-normalised CS width and empirical cumulative miscoverage.
\begin{figure}
    \centering
    \includegraphics[width=0.9\textwidth]{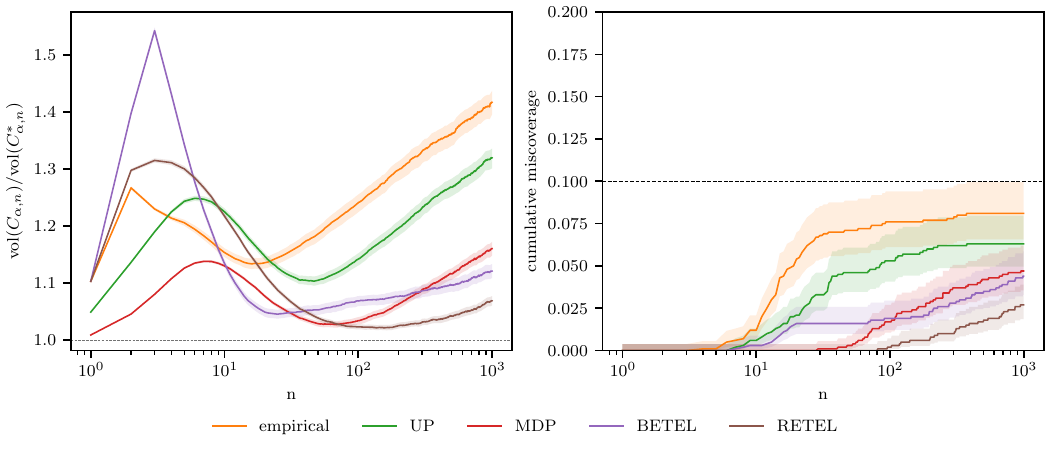}
    \caption{
        Synthetic beta-mixture experiment.
        The left and right panels show average CS length relative to the oracle and empirical cumulative miscoverage, respectively, over $1000$ replicates;
        the dashed line marks $\alpha = 0.1$. Width ratios use replicates where both the method and oracle CSs are non-empty. At $n=1000$, at most 2\% are empty (0.8\% on average). Shading shows pointwise 95\% paired-bootstrap intervals for width ratios and Wilson intervals for cumulative miscoverage.
    }
    \label{fig:simulation_beta_mixture_main}
\end{figure}
With informative priors, Bayes-assisted methods are closer to the oracle than the empirical predictive and UP baselines over much of the sample-size range.
The same overall qualitative pattern is observed for the remaining synthetic distributions in \cref{fig:simulation_bernoulli_baselines,fig:simulation_beta_baselines}: when the prior is informative, Bayes-assisted predictives usually reduce oracle-normalised width relative to prior-free baselines.
Among the Bayes-assisted methods in \cref{fig:simulation_beta_mixture_main}, MDP benefits strongly from prior information at small $n$, while BETEL and RETEL gradually catch up as the sample size grows.
The prior-sensitivity experiments in \cref{app:additional_priors} show the complementary effect: diffuse or misspecified priors can increase oracle-normalised width, but empirical miscoverage remains controlled.
Thus prior quality affects efficiency, whereas coverage is guaranteed by the underlying test-martingale construction.
Sensitivity to the truncation level $c$ and the MDP mixing parameter $\kappa$ is studied in \cref{app:c,app:kappa}.

\paragraph{Best-arm identification for LLM evaluation.}
% \label{sec:LLM_eval}

As a first real-data application, we consider adaptive LLM evaluation for identifying the best of four LLMs on a sorting task.
We frame the problem as \emph{best-arm identification}: each LLM is a bandit arm, and pulling an arm queries the model on a random array of length $20$ and returns a rescaled Spearman rank-correlation score in $[0,1]$ measuring the quality of the sorted output.
We use LUCB \citep{Kalyanakrishnan2012,Kaufmann2013} with arm-wise confidence sequences, stopping once the leading model is separated from its competitors.
This avoids fixing the same benchmark size for all models and requires anytime-valid bounds because LUCB repeatedly inspects confidence sequences, adaptively samples arms, and stops at a data-dependent time, thereby potentially reducing the total number of LLM calls.
For this example, we elicit priors from cheaper auxiliary evaluations: for each LLM, we evaluate the same task on smaller arrays, fit an exponential decay trend in array length, and centre the arm prior at the extrapolated score for length $20$.
Details on LUCB, auxiliary array lengths, and prior parameters are given in \cref{app:LLM_eval}. \looseness=-1

\Cref{tab:method_comparison_spear} reports the average LUCB stopping time over $50$ replicates, paired ranks, and win/tie/loss counts against the empirical predictive baseline, the strongest prior-free method in this experiment.
MDP performs best in terms of both average calls and paired rank, stopping earlier than the empirical predictive in $38$ of $50$ paired replicates, while BETEL is close to empirical. Together, these observations suggest that the Beta prior is a good choice for this experiment.
UP requires more calls on average than the empirical predictive.
Full-baseline results and paired-comparison analyses are reported in \cref{app:LLM_eval_results}, while results under deliberately misspecified priors are reported in \cref{app:LLM_bad_prior}.

\begin{table*}[t]
    \centering
    \footnotesize
    \caption{
        Sequential stopping-time comparisons in the two real-data applications.
        Rank is computed within each replicate; W/T/L is against the empirical predictive baseline.
    }
    \label{tab:real_data_stopping}

    \begin{subtable}[t]{0.48\textwidth}
        \centering
        \caption{LLM best-arm identification. Avg calls is the LUCB stopping time. Across all methods, no CS is empty at the stopping time.}
        \label{tab:method_comparison_spear}
        \begin{tabular}{lccc}
            \toprule
            Method & Avg calls (s.e.) & Avg rank (s.e.)  & W/T/L \\
            \midrule
            MDP & 91.68 (3.90) & 1.95 (0.13) & 38/5/7 \\
            BETEL & 99.36 (3.62) & 3.28 (0.17) & 23/7/20 \\
            Emp. & 101.64 (4.03) & 3.58 (0.21) & -- \\
            RETEL & 109.76 (3.63) & 4.95 (0.23) & 18/2/30 \\
            UP & 110.68 (3.53) & 5.18 (0.16) & 4/3/43 \\
            \bottomrule
        \end{tabular}
    \end{subtable}
    \hfill
    \begin{subtable}[t]{0.48\textwidth}
        \centering
        \caption{\textsc{attributedQA} PPI test. Avg labels is time to rejection. Across all methods, no CS is empty at the stopping time.}
        \label{tab:ppi_attributedqa_test}
        \begin{tabular}{lccc}
            \toprule
            Method & Avg labels (s.e.) & Avg rank (s.e.) & W/T/L \\
            \midrule
            BETEL & 50.25 (3.37) & 3.33 (0.18) & 71/15/14 \\
            MDP & 52.18 (3.50) & 3.49 (0.18) & 82/16/2 \\
            RETEL & 55.03 (2.49) & 4.24 (0.22) & 62/8/30 \\
            Emp. & 65.24 (4.12) & 5.88 (0.25) & -- \\
            UP & 69.73 (3.24) & 6.48 (0.16) & 38/4/58 \\
            \bottomrule
        \end{tabular}
    \end{subtable}
\end{table*}

\paragraph{Prediction-powered inference.}
% \label{sec:ppi}

As a second real-data application, we consider prediction-powered inference (PPI) \citep{angelopoulos2023prediction} for bounded mean estimation.
Given a fixed prediction rule $f$, PPI decomposes the target mean as $\theta^\star = \mathbb E[Y] = \mathbb E[f(X)] + \mathbb E[Y-f(X)] = m^\star + \Delta^\star$,
where $m^\star$ is estimated from unlabelled data and the rectifier $\Delta^\star$ from labelled residuals.
When labelled and unlabelled data both arrive sequentially, one can build confidence sequences for $m^\star$ and $\Delta^\star$ separately and combine them by a union bound.
As noted by \citet{Kilian2025}, however, this merging step can be loose and can obscure the contribution of the confidence-sequence method used for the rectifier.
We therefore follow their simplified setting: a large pool of $N$ unlabelled examples is available at the start, $m^\star$ is treated as fixed through the plug-in estimate $\tfrac{1}{N}\sum_{j=1}^N f(X_j)$, and sequential uncertainty is concentrated on the labelled residuals $Y_i-f(X_i)$.
Since these residuals are bounded, we apply our bounded-mean CSs after rescaling to $[0,1]$;
see \cref{app:ppi_construction} for details.
For this application, we use priors centred at zero for $\Delta^\star$, encoding the belief that the prediction rule is approximately unbiased.
After rescaling the residual range to $[0,1]$, this corresponds to symmetric beta priors centred at $1/2$, with concentration controlling the strength of the prior belief; full specifications are given in \cref{app:ppi_prior_specification}.
 
\Cref{fig:ppi_main} compares methods on two PPI tasks: estimating the proportion of galaxies with spiral arms in \textsc{galaxies}, and estimating the difference in attribution reliability between two retrieve-then-read QA systems, RTR-10 and RTR-4, on \textsc{attributedQA}.
For \textsc{attributedQA}, both labels and predictions are differences between the two systems' binary attribution outcomes/probabilities; dataset details are in \cref{app:ppi_datasets}.
\begin{figure}
    \centering
    \includegraphics[width=0.82\textwidth]{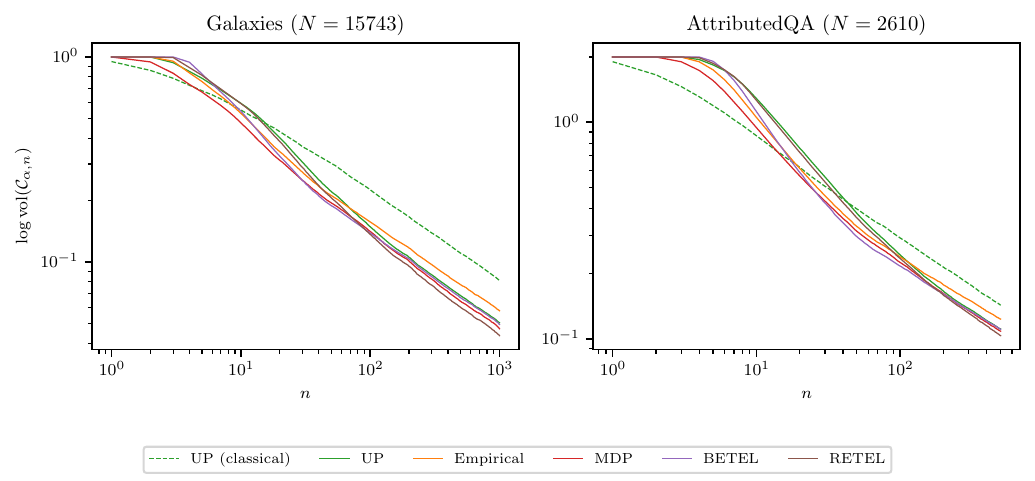}
    \caption{
        Prediction-powered mean estimation.
        The left and right panels show average CS length over 100 replicates for the \textsc{galaxies} and \textsc{attributedQA} datasets.
        Curves labelled ``classical'' use only labelled observations, while the remaining curves use the PPI construction. Empty CSs are counted as width zero (at the final sample size, at most 1\% are empty, 0.2\% on average).}
    \label{fig:ppi_main}
\end{figure}
PPI-based CSs are substantially narrower than classical labelled-only CSs on both datasets.
Among PPI methods, the Bayes-assisted procedures are competitive with the empirical and UP baselines: MDP performs best at small sample sizes, while RETEL is strongest at larger sample sizes in these experiments.

These width reductions translate into earlier anytime-valid decisions.
On \textsc{attributedQA}, testing $H_0:\theta^\star \leq 0$ asks whether RTR-10 has higher attribution reliability than RTR-4; rejection occurs once $0$ is excluded from the CS.
\Cref{tab:real_data_stopping} shows that MDP, BETEL, and RETEL all reject using fewer labelled examples on average than the empirical predictive and UP baselines.
The paired comparisons against the empirical predictive are strongest for MDP and BETEL, while RETEL also wins more often than it loses.
Thus an informative prior on the rectifier can shorten CSs and reduce the labelled sample size needed for sequential conclusions.
Full comparisons with additional baselines are reported in \cref{app:ppi_results}, while the corresponding misspecified-prior experiment is reported in \cref{app:ppi_bad_prior}.

% !TEX root = bacs-main.tex

\section{Extensions beyond univariate \iid bounded means}
\label{sec:extensions}

While our results focus on \iid observations in $[0,1]$, the principle of using a predictive distribution to select a valid predictable bet also suggests extensions beyond this setting.

\paragraph{Quantiles.}
Let $(X_n)_{n\geq1}$ be \iid from a distribution with cdf $F$, without requiring bounded observations. For a candidate quantile $q$, define $Z_n(q)=\mathbf{1}\{X_n\leq q\}$. The variables $(Z_n(q))_{n\geq1}$ are \iid and bounded with mean $F(q)$, so our construction can be applied to $F(q)$ and inverted over $q$. This connects directly with existing sequential inference for quantiles \cite{Howard2022}.

\paragraph{Relaxing the \iid assumption.}
If $\bbE[X_n\mid\mathcal{F}_{n-1}]=\mutrue$ for all $n\geq1$, independence is unnecessary for validity: every predictable admissible stake still yields a nonnegative martingale. Consequently, the Bayes-assisted procedures presented in this work remain anytime-valid under this conditional-mean assumption, irrespective of predictive-model misspecification.
However, the log-optimality results in \cref{thm:logoptimality,thm:finitn} rely on the \iid assumption and do not directly extend to this larger class. For example, suppose $(X_n)_{n\geq1}$ is Markov with transition kernel $K$ and constant conditional mean. The relevant oracle stake is then state-dependent, maximising $\int \log\{1+\lambda(y-\mu)\}\,K(x,dy)$ over $\lambda\in I_{\mu,c}$, whereas our methods learn a single stake from the marginal distribution. A natural Bayes-assisted extension would approximate the conditional law, e.g. the Markov transition kernel, rather than the marginal law, and select $\lambda_i$ using the resulting conditional posterior predictive. Establishing log-optimality would then require a new analysis relative to this richer, state-dependent oracle.

\paragraph{Relaxing boundedness.}
Without additional moment or tail constraints, there is no nontrivial universal e-variable for the mean of arbitrary unbounded observations. Suppose instead that, under a candidate null, $\mathbb{E}[X_i \mid \mathcal{F}_{i-1}]=\mu$ and $\operatorname{Var}(X_i \mid \mathcal{F}_{i-1}) \leq v$. Then $E_i(\mu)=(X_i-\mu)^2/v$ is a conditional e-variable. Hence, for any predictable $\eta_i \in[0,1]$, $M_n(\mu)=\prod_{i=1}^n\{1-\eta_i+\eta_i E_i(\mu)\}$ is a test supermartingale. This bounded-variance construction is developed by \cite{Fan2025}. Our approach could be extended to select $\eta_i$ by maximising posterior-predictive expected log-growth. The MDP predictive could combine the empirical distribution with a Gaussian or Student-$t$ parametric component, while a possible empirical-likelihood analogue is the Gaussian-regularized RETEL construction of \cite{Kim2026}.

\paragraph{Multivariate extensions.}
Suppose $X_i \in K \subset \mathbb{R}^d$, where $K$ is known, compact, and convex. For each candidate mean $\mu \in K$, replace the scalar stake by a predictable vector $\lambda_i(\mu)$ and define $W_n(\mu)=\prod_{i=1}^n\{1+\lambda_i(\mu)^{\top}(X_i-\mu)\}$. A suitable admissible set is $\Lambda_{\mu,c}=\{\lambda \in \mathbb{R}^d: \inf_{x \in K}[1+\lambda^{\top}(x-\mu)] \geq 1-c\}$. This constraint keeps every wealth factor nonnegative and bounded away from zero. If $\mathbb{E}[X_i \mid \mathcal{F}_{i-1}]=\mu^{\star}$, then $W_n(\mu^{\star})$ is a nonnegative martingale for every predictable choice $\lambda_i(\mu^{\star}) \in \Lambda_{\mu^{\star},c}$. Inverting these tests therefore gives an anytime-valid confidence sequence for the vector mean. Given a predictive distribution $Q_i$ on $K$, the Bayes-assisted vector stake maximises $\int_K \log\{1+\lambda^{\top}(x-\mu)\}\,Q_i(dx)$ over $\lambda\in\Lambda_{\mu,c}$. For fixed $\mu$, this is a concave optimization problem over a convex admissible set. A parametric predictive can represent dependence between coordinates, while an MDP can mix that predictive with the empirical joint distribution. BETEL and RETEL also accommodate vector-valued moment restrictions. \Cref{thm:multivariate_logoptimality} extends \cref{thm:logoptimality}
to vector means.

\section{Discussion and limitations}
\label{sec:discussion}

This paper introduced a Bayes-assisted framework for constructing anytime-valid confidence sequences for bounded means: Bayesian or pseudo-Bayesian predictives are used to choose powerful predictable betting updates, while validity remains guaranteed by test-martingale inversion even under prior or working-model misspecification. The theory separates validity from efficiency, showing oracle-level asymptotic log-growth when the predictive distribution is consistent, and the MDP, BETEL, and RETEL constructions illustrate how informative prior information can yield shorter confidence sequences in practice. Several limitations remain: computing the sets can be substantially more expensive than closed-form baselines because it requires evaluating martingales over candidate means and repeatedly updating predictives (see \cref{app:computational_complexity}); while poor priors do not invalidate coverage, the efficiency gains depend on the prior or auxiliary predictive being sufficiently informative (see \cref{app:additional_priors}); the inverted sets are not guaranteed to be single intervals, which can complicate reporting or downstream use; the pointwise numerical inversion used in the main experiments guarantees coverage only at the tested grid values, although we provide an alternative continuum-valid construction in \Cref{app:cellwise}. These limitations are natural targets for future work on faster algorithms, interval-valued approximations, and adaptive or data-driven prior construction.

\newpage
\begin{ack}
The authors would like to acknowledge the use of the University of Oxford Advanced Research Computing (ARC) facility \citep{Richards2015} in carrying out this work.
Valentin Kilian is supported by the Clarendon Funds Scholarship.
Stefano Cortinovis is supported by the EPSRC Centre for Doctoral Training in Modern Statistics and Statistical Machine Learning (EP/S023151/1).
\end{ack}

\bibliographystyle{unsrtnat}
\bibliography{bacs-biblio}

%%%%%%%%%%%%%%%%%%%%%%%%%%%%%%%%%%%%%%%%%%%%%%%%%%%%%%%%%%%%
\newpage
\appendix

\section*{Technical Appendices and Supplementary Material}

% !TEX root = bacs-main.tex

\section{Proofs}

\subsection{Proof of \cref{prop:test_martingale_property}}

\testmarginaleproperty*

\begin{proof}
For $\lambda\in I_{\mu,c}$ and $x\in[0,1]$,
$1+\lambda(x-\mu)\geq 1-c\geq0$, so the process is nonnegative.  Since
$\lambdahat_i(X_{1:i-1},\mu)$ is $\calF_{i-1}$-measurable, under $H_\mu$,
\begin{align*}
\E\left[
1+\lambdahat_i(X_{1:i-1},\mu)(X_i-\mu)
\mid \calF_{i-1}
\right]
&=
1+\lambdahat_i(X_{1:i-1},\mu)
\E[X_i-\mu\mid \calF_{i-1}] \\
&=1 .
\end{align*}
Hence $\E[W_i(X_{1:i},\mu)\mid\calF_{i-1}]=W_{i-1}(X_{1:i-1},\mu)$.
\end{proof}

\subsection{Proof of \cref{prop:anytime_validity}}

\anytimevalidity*

\begin{proof}
By \Cref{prop:test_martingale_property},
$(W_n(X_{1:n},\mutrue))_{n\geq0}$ is a nonnegative martingale with initial value one.  Ville's inequality \citep{Ville1939} gives
\[
\Ptrue\left(
\sup_{n\geq1} W_n(X_{1:n},\mutrue) > \frac{1}{\alpha}
\right)
\leq \alpha .
\]
On the complement of this event, $\mutrue$ belongs to $\CS$ for every
$n\geq1$, which proves the claim.
\end{proof}

\subsection{Proof of \cref{thm:logoptimality}}

We first recall the theorem.

\thmlogoptimality*

\textbf{Proof.} All probabilities and expectations in this proof are taken under the true data-generating law $\Ptrue$ of the \iid sequence $(X_i)_{i\geq 1}$, unless explicitly stated otherwise. For readability, write
\[
I_{\mu,c}:=\left[-\frac{c}{1-\mu},\frac{c}{\mu}\right],
\qquad
M_n(\lambda):=M_n(\lambda,X_{1:n-1},\mu),
\qquad
M(\lambda):=M(\lambda,\mu),
\]
and define
\[
f_\lambda(x):=\log\{1+\lambda(x-\mu)\},\qquad \lambda\in I_{\mu,c},\ x\in[0,1].
\]
For every $(\lambda,x)\in I_{\mu,c}\times[0,1]$, we have $1+\lambda(x-\mu)\geq 1-c>0$. Hence $f_\lambda$ is jointly continuous and uniformly bounded: with
\[
B_{\mu,c}:=\max\left\{ |\log(1-c)|,\log\left(1+c\frac{1-\mu}{\mu}\right),
\log\left(1+c\frac{\mu}{1-\mu}\right)\right\},
\]
we have $|f_\lambda(x)|\leq B_{\mu,c}$ for all $(\lambda,x)\in I_{\mu,c}\times[0,1]$. Moreover, for fixed $\lambda\in I_{\mu,c}$, the map $x\mapsto f_\lambda(x)$ is Lipschitz with a constant that is uniform in $\lambda$:
\[
\left|\partial_x f_\lambda(x)\right|
=\left|\frac{\lambda}{1+\lambda(x-\mu)}\right|
\leq \frac{c}{(1-c)\min\{\mu,1-\mu\}}
=:L_{\mu,c}.
\]
Therefore, by Kantorovich--Rubinstein duality for $\Wassone$ \cite[Sec.~1.2]{Santambrogio2015},
\begin{align}
Z_n
:=\sup_{\lambda\in I_{\mu,c}} |M_n(\lambda)-M(\lambda)|
\leq L_{\mu,c}\,
\Wassone\left(\bbQ_{n\mid n-1}(\cdot\mid X_{1:n-1}),\Ptrue\right).
\label{eq:uniform_Mn_M}
\end{align}
Under the assumption of \cref{thm:logoptimality}, $Z_n\to0$ in probability. If the convergence in Wasserstein distance is strengthened to almost sure convergence, then the same display gives $Z_n\to0$ almost surely.

We next show that $M$ has a unique maximiser. For each $x\in[0,1]$, the function $\lambda\mapsto f_\lambda(x)$ is concave on $I_{\mu,c}$, and it is strictly concave whenever $x\neq\mu$. Since $\Ptrue$ is non-degenerate, $\Ptrue(X\neq\mu)>0$. Hence, for any distinct $\lambda_1,\lambda_2\in I_{\mu,c}$ and $t\in(0,1)$,
\[
M(t\lambda_1+(1-t)\lambda_2)>tM(\lambda_1)+(1-t)M(\lambda_2),
\]
so $M$ is strictly concave. Since $M$ is continuous on the compact interval $I_{\mu,c}$, it admits a unique maximiser $\lambda^*=\lambda^*(\mu)$.

We now prove convergence of the maximisers. Fix $\epsilon>0$ and set
\[
A_\epsilon:=\{\lambda\in I_{\mu,c}: |\lambda-\lambda^*|\geq\epsilon\}.
\]
If $A_\epsilon=\varnothing$, there is nothing to prove. Otherwise, by compactness of $A_\epsilon$ and uniqueness of $\lambda^*$,
\[
\Delta_\epsilon:=M(\lambda^*)-\sup_{\lambda\in A_\epsilon}M(\lambda)>0.
\]
On the event $\{\widehat\lambda_n\in A_\epsilon\}$, the maximality of $\widehat\lambda_n$ for $M_n$ implies
\begin{align*}
\Delta_\epsilon
&\leq M(\lambda^*)-M(\widehat\lambda_n)\\
&=\{M(\lambda^*)-M_n(\lambda^*)\}+\{M_n(\lambda^*)-M_n(\widehat\lambda_n)\}
  +\{M_n(\widehat\lambda_n)-M(\widehat\lambda_n)\}\\
&\leq 2Z_n.
\end{align*}
Consequently,
\[
\bbP\{|\widehat\lambda_n-\lambda^*|\geq\epsilon\}
\leq \bbP\{Z_n\geq \Delta_\epsilon/2\}\to0,
\]
which proves \cref{eq:lambda_as}. The same argument also shows that if $Z_n\to0$ almost surely, then $\widehat\lambda_n\to\lambda^*$ almost surely.

Define the log increment
\begin{align}
Y_n:=\log\{1+\widehat\lambda_n(X_n-\mu)\},
\label{eq:logincrementYn}
\end{align}
and let
\[
D_n:=Y_n-\E[Y_n\mid\calF_{n-1}],
\qquad
S_n:=\sum_{i=1}^n D_i.
\]
Then $(S_n)_{n\geq1}$ is a martingale with respect to $(\calF_n)_{n\geq1}$. Since $|Y_n|\leq B_{\mu,c}$, we have $|D_n|\leq 2B_{\mu,c}$ and
\[
\sum_{i=1}^\infty \frac{\E[D_i^2\mid\calF_{i-1}]}{i^2}
\leq 4B_{\mu,c}^2\sum_{i=1}^\infty \frac{1}{i^2}<\infty.
\]
The martingale law of large numbers \cite[Theorem~2.18]{Hall1980} gives
\begin{align}
\frac{S_n}{n}\to0\qquad\text{a.s.}
\label{eq:convSn}
\end{align}
Moreover, since $X_i$ is independent of $\calF_{i-1}$ and $\widehat\lambda_i$ is $\calF_{i-1}$-measurable,
\[
\E[Y_i\mid\calF_{i-1}]
=\E\left[\log\{1+\widehat\lambda_i(X_i-\mu)\}\mid\calF_{i-1}\right]
=M(\widehat\lambda_i).
\]
Thus
\begin{align}
\frac{1}{n}\log W_n
=\frac{1}{n}\sum_{i=1}^n M(\widehat\lambda_i)+\frac{S_n}{n}.
\label{eq:decomplogWn}
\end{align}
Because $\widehat\lambda_n\to\lambda^*$ in probability and $M$ is continuous and bounded on $I_{\mu,c}$, $M(\widehat\lambda_n)\to M(\lambda^*)$ in $L^1$. Hence, by Ces\`aro summability,
\[
\frac{1}{n}\sum_{i=1}^n M(\widehat\lambda_i)\to M(\lambda^*)
\qquad\text{in }L^1,
\]
and therefore also in probability. Combining this convergence with \cref{eq:convSn,eq:decomplogWn} gives
\[
\frac{1}{n}\log W_n\to M(\lambda^*)
\qquad\text{in probability},
\]
which proves \cref{eq:Wn_proba}.

Finally,
\[
\frac{1}{n}\E[\log W_n]
=\frac{1}{n}\sum_{i=1}^n \E[Y_i]
=\frac{1}{n}\sum_{i=1}^n \E[M(\widehat\lambda_i)]
\to M(\lambda^*)
\]
by the same boundedness, $L^1$ convergence, and Ces\`aro argument. This proves \cref{eq:Wn_L1}.

\paragraph{Almost-sure counterpart.}
If the assumption $\Wassone(\bbQ_{n\mid n-1}(\cdot\mid X_{1:n-1}),\Ptrue)\to0$ holds almost surely rather than merely in probability, then \cref{eq:uniform_Mn_M} gives uniform convergence $Z_n\to0$ almost surely. The separation argument for the argmax then yields $\widehat\lambda_n\to\lambda^*$ almost surely. Consequently, $M(\widehat\lambda_n)\to M(\lambda^*)$ almost surely, and Ces\`aro summability gives $n^{-1}\sum_{i=1}^n M(\widehat\lambda_i)\to M(\lambda^*)$ almost surely. Together with the martingale law of large numbers in \cref{eq:convSn}, this yields
\[
\frac{1}{n}\log W_n\to M(\lambda^*)
\qquad\text{almost surely}.
\]
The expectation convergence in \cref{eq:Wn_L1} remains valid by the bounded $L^1$ argument above.

\subsection{Proof of \cref{thm:finitn}}

We already proved that
$$
\sup_{\lambda\in I_{\mu,c}} |M_n(\lambda)-M(\lambda)|\leq L_{\mu,c}\times \Wassone\left( \bbQ_{n|n-1}(\cdot\mid X_{1:n-1}),\Ptrue\right)
$$
We have
\begin{align*}
0\leq M(\lambda^*) - M(\widehat\lambda_n) &=(M(\lambda^*)-M_n(\lambda^*)) + (M_n(\lambda^*)-M_n(\widehat\lambda_n)) + (M_n(\widehat\lambda_n) - M(\widehat\lambda_n) )\\
&\leq 2\sup_{\lambda\in I_{\mu,c}} |M_n(\lambda)-M(\lambda)|\\
&\leq 2L_{\mu,c}\times \Wassone\left( \bbQ_{n|n-1}(\cdot\mid X_{1:n-1}),\Ptrue\right).
\end{align*}
Additionally,
$$
\E[\log W_n]=\sum_{i=1}^n \E[Y_i]=\sum_{i=1}^n \E[\E[ Y_i\mid \calF_{i-1}]]=\sum_{i=1}^n \E[M(\widehat\lambda_i)],
$$
where $Y_n$ is the log increment defined in \cref{eq:logincrementYn}. It follows that
$$
0\leq M(\lambda^*) - \frac{1}{n}\E[\log W_n]\leq 2L_{\mu,c}\times \frac{1}{n}\sum_{i=1}^n \E\left[\Wassone\left( \bbQ_{i|i-1}(\cdot\mid X_{1:i-1}),\Ptrue\right)\right].
$$

\subsection{Proof of \Cref{prop:wass_rate}}

Let
\[
\omega_n=\frac{\kappa}{\kappa+n}.
\]
By convexity of \(\Wassone\) in its first argument,
\begin{align*}
\Wassone(\bbQ_{n+1\mid n},\Ptrue)
&\le
\omega_n \Wassone(\bbH_{n+1\mid n},\Ptrue)
+
(1-\omega_n)\Wassone(\bbP_n,\Ptrue) .
\end{align*}
Since both \(\bbH_{n+1\mid n}\) and \(\Ptrue\) are supported on \([0,1]\),
\(\Wassone(\bbH_{n+1\mid n},\Ptrue)\le 1\). Therefore
\[
\E_{\Ptrue}\{\Wassone(\bbQ_{n+1\mid n},\Ptrue)\}
\le
\frac{\kappa}{\kappa+n}
+
\E_{\Ptrue}\{\Wassone(\bbP_n,\Ptrue)\}.
\]
The second term is \(O(n^{-1/2})\) for distributions supported on a bounded
interval, so the displayed expectation is \(O(n^{-1/2})\). Markov's
inequality gives convergence in \(\Ptrue\)-probability, and
\Cref{thm:logoptimality} gives the final claim.

\subsection{Proof of \Cref{prop:retel}}
\label{app:proof-retel}

We first recall the proposition.
\propretel*

\textbf{Proof.} Write
\[
    \bar X_n=\frac1n\sum_{i=1}^n X_i,
    \qquad
    s_n^2=\frac1n\sum_{i=1}^n (X_i-\bar X_n)^2,
\]
and define the rescaled pseudo-likelihood
\[
    R_n(\mu)=\prod_{i=1}^n n w_{n,i}(\mu).
\]
The factor \(n^n\) is independent of \(\mu\), so \(R_n\) and
\(L_n\) induce the same pseudo-posterior.  In the unregularized case
\(\tau_n=0\), we use the convention from the main text that
\(R_n(\mu)=0\) whenever \(\mu\notin(\min_i X_i,\max_i X_i)\).  At points
where the likelihood is zero, the value assigned to \(\gamma_{n,\mu}\) is
irrelevant for the posterior statements below.

\begin{lemma}[Separation from the empirical mean]
\label{lem:retel-separation}
For every \(t\in(0,1]\),
\[
    \sup_{\mu\in[0,1]:\,|\mu-\bar X_n|\geq t} R_n(\mu)
    \leq
    \exp\{-n c(t)\},
    \qquad
    c(t)=\min\left\{-\log(1-t/2),\,\frac{t^2}{2}\right\}>0 .
\]
\end{lemma}

\begin{proof}
Fix \(\mu\).  If \(R_n(\mu)=0\), there is nothing to prove.  Otherwise,
write
\[
    a=\sum_{i=1}^n w_{n,i}(\mu)=1-w_{n,H}(\mu)>0,
    \qquad
    p_i=\frac{w_{n,i}(\mu)}{a},\quad i=1,\ldots,n.
\]
Then \(p=(p_1,\ldots,p_n)\) is a probability vector.  Since
\(\widetilde{\bbP}_{n,\mu}\) has mean \(\mu\), and its regularizing
component is supported on \([0,1]\), there is an \(m_H\in[0,1]\) such
that
\[
    \mu=a\sum_{i=1}^n p_iX_i+(1-a)m_H .
\]
Let \(m_p=\sum_i p_iX_i\) and let \(u_i=1/n\).  The last display implies
\(|\mu-m_p|\leq 1-a\).  Moreover,
\[
    \log R_n(\mu)
    =
    \sum_{i=1}^n\log(nw_{n,i}(\mu))
    =
    n\log a+
    \sum_{i=1}^n\log(np_i)
    =
    n\log a-n\KL(u\|p),
\]
with the usual convention that \(\KL(u\|p)=\infty\) if \(p_i=0\) for
some \(i\).

Suppose \(|\mu-\bar X_n|\geq t\).  If \(1-a\geq t/2\), then
\[
    \log R_n(\mu)
    \leq
    n\log a
    \leq
    n\log(1-t/2).
\]
If instead \(1-a<t/2\), then \(|m_p-\bar X_n|\geq t/2\).  Since
\(X_i\in[0,1]\),
\[
    |m_p-\bar X_n|
    =
    \left|\sum_{i=1}^n(p_i-u_i)X_i\right|
    \leq \mathrm{TV}(p,u),
\]
where \(\mathrm{TV}(p,u)=\frac12\sum_i|p_i-u_i|\).  Pinsker's inequality
therefore gives
\[
    \KL(u\|p)
    \geq 2\,\mathrm{TV}(u,p)^2
    \geq \frac{t^2}{2}.
\]
Thus \(\log R_n(\mu)\leq -nt^2/2\) in the second case.  Combining the two
cases proves the lemma.
\end{proof}

\begin{lemma}[Local control of the tilting parameter]
\label{lem:retel-local-gamma}
For every \(v_0>0\) there exist constants \(\eta>0\) and \(C<\infty\),
depending only on \(v_0\), such that, for all sufficiently large \(n\),
whenever \(s_n^2\geq v_0\),
\[
    |\gamma_{n,\mu}|
    \leq
    C\left(|\mu-\bar X_n|+\frac{\tau_n}{n}\right)
\]
for every \(\mu\) for which \(|\mu-\bar X_n|\leq\eta\) and the tilt equation
is defined.
\end{lemma}

\begin{proof}
For fixed \(\mu\), write
\[
    \psi_{n,\mu}(\gamma)
    =
    \log\int e^{\gamma(z-\mu)}\,d\bbP_{n,\mu}(z).
\]
The tilting parameter solves \(\psi_{n,\mu}'(\gamma_{n,\mu})=0\).  Let
\(h_{n,\mu}=\int z\,d\bbH_\mu(z)\in[0,1]\).  The mean of the base measure is
\[
    m_{n,\mu}
    =
    \int z\,d\bbP_{n,\mu}(z)
    =
    \frac{n}{n+\tau_n}\bar X_n
    +
    \frac{\tau_n}{n+\tau_n}h_{n,\mu},
\]
so
\begin{equation}
\label{eq:retel-local-score-bound}
    |\psi_{n,\mu}'(0)|
    =
    |m_{n,\mu}-\mu|
    \leq
    |\bar X_n-\mu|+\frac{\tau_n}{n+\tau_n}.
\end{equation}
Moreover, by the variance decomposition for mixtures,
\[
    \psi_{n,\mu}''(0)
    =
    \var_{\bbP_{n,\mu}}(Z)
    \geq
    \frac{n}{n+\tau_n}s_n^2 .
\]
If \(s_n^2\geq v_0\), then, since \(\tau_n=o(n)\), the last quantity is
bounded below by \(v_0/2\) for all sufficiently large \(n\).

Fix a small constant \(\delta>0\).  For \(|\gamma|\leq\delta\), let
\(\bbP_{n,\mu}^\gamma\) be the exponentially tilted version of
\(\bbP_{n,\mu}\).  Since \(z-\mu\in[-1,1]\),
\[
    \frac{d\bbP_{n,\mu}^\gamma}{d\bbP_{n,\mu}}(z)
    =
    \exp\{\gamma(z-\mu)-\psi_{n,\mu}(\gamma)\}
    \geq e^{-2\delta}.
\]
Using
\[
    \var_Q(Z)=\frac12\int\!\int (z-z')^2\,dQ(z)dQ(z'),
\]
we obtain
\[
    \psi_{n,\mu}''(\gamma)
    =
    \var_{\bbP_{n,\mu}^\gamma}(Z)
    \geq
    e^{-4\delta}\var_{\bbP_{n,\mu}}(Z)
    \geq
    e^{-4\delta}v_0/2
    =:c_0
\]
for all \(|\gamma|\leq\delta\) and all large \(n\).

Choose \(\eta>0\) such that \(\eta<c_0\delta/2\).  Since \(\tau_n/n\to0\),
\eqref{eq:retel-local-score-bound} implies that, uniformly over
\(|\mu-\bar X_n|\leq\eta\),
\[
    |\psi_{n,\mu}'(0)|<c_0\delta
\]
for all large \(n\).  Hence
\(\psi_{n,\mu}'(-\delta)<0<\psi_{n,\mu}'(\delta)\), and the root lies in
\((-\delta,\delta)\).  By the mean-value theorem,
\[
    |\gamma_{n,\mu}|
    \leq
    c_0^{-1}|\psi_{n,\mu}'(0)|
    \leq
    c_0^{-1}\left(|\mu-\bar X_n|+\frac{\tau_n}{n+\tau_n}\right),
\]
which proves the claim, after increasing the constant if necessary.
\end{proof}

\begin{lemma}[Subexponential denominator lower bound]
\label{lem:retel-denominator}
Let
\[
    D_n=\int_0^1 R_n(\mu)\pi(d\mu).
\]
Under the assumptions of \Cref{prop:retel}, with \(\Ptrue\)-probability
tending to one,
\[
    D_n\geq \frac{c}{n}\exp\{-C(1+\tau_n)\}
\]
for constants \(c,C>0\) independent of \(n\).  In particular,
\(\log D_n=o_{\Ptrue}(n)\).
\end{lemma}

\begin{proof}
Since \(\bar X_n\to\mutrue\) and \(s_n^2\to\sigma^2>0\) almost surely,
and since \(\mutrue\in(0,1)\), with probability tending to one the interval
\[
    B_n=\left[\bar X_n-\frac1n,\bar X_n+\frac1n\right]
\]
is contained in \((0,1)\), lies in a neighbourhood on which the prior
density is bounded below, and has \(s_n^2\geq\sigma^2/2\).  Applying
\Cref{lem:retel-local-gamma} with \(v_0=\sigma^2/2\), we also have the
local bound in that lemma on this event.  In the unregularized case
\(\tau_n=0\), the same event may also be taken to include
\(B_n\subset(\min_i X_i,\max_i X_i)\): this holds with probability tending
to one because \(\sigma^2>0\) implies positive probability both below and
above \(\mutrue\).

On this event, for every \(\mu\in B_n\),
\[
    |\gamma_{n,\mu}|
    \leq
    C_1\frac{1+\tau_n}{n}.
\]
Using \(|X_i-\mu|\leq1\) and
\(M_{\bbH_\mu}(\gamma_{n,\mu})\leq e^{|\gamma_{n,\mu}|}\), we have
\[
    n w_{n,i}(\mu)
    =
    \frac{n e^{\gamma_{n,\mu}(X_i-\mu)}}
    {\sum_{j=1}^n e^{\gamma_{n,\mu}(X_j-\mu)}
      +\tau_n M_{\bbH_\mu}(\gamma_{n,\mu})}
    \geq
    \frac{n}{n+\tau_n}e^{-2|\gamma_{n,\mu}|}.
\]
Therefore, uniformly over \(\mu\in B_n\),
\[
    R_n(\mu)
    \geq
    \left(\frac{n}{n+\tau_n}\right)^n
    \exp\{-2n|\gamma_{n,\mu}|\}
    \geq
    \exp\{-C_2(1+\tau_n)\},
\]
where we used \(\log(1+x)\leq x\).  The prior mass of \(B_n\) is at least
\(c/n\) on the same event, so
\[
    D_n
    \geq
    \int_{B_n}R_n(\mu)\pi(d\mu)
    \geq
    \frac{c}{n}\exp\{-C_2(1+\tau_n)\}.
\]
Since \(\tau_n=o(n)\), this gives \(\log D_n=o_{\Ptrue}(n)\).
\end{proof}

We now prove the posterior concentration assertions.  Fix
\(\epsilon>0\).  If \(\epsilon>1\), then
\(\{\mu\in[0,1]:|\mu-\mutrue|>\epsilon\}\) is empty, so assume
\(\epsilon\leq1\).  Let
\[
    A_\epsilon=\{\mu\in[0,1]:|\mu-\mutrue|>\epsilon\}.
\]
On the event \(|\bar X_n-\mutrue|\leq\epsilon/2\),
\(A_\epsilon\subseteq\{\mu:|\mu-\bar X_n|\geq\epsilon/2\}\).  Hence, by
\Cref{lem:retel-separation},
\[
    \int_{A_\epsilon}R_n(\mu)\pi(d\mu)
    \leq
    \exp\{-n c(\epsilon/2)\}.
\]
Combining this exponential numerator bound with the subexponential lower
bound from \Cref{lem:retel-denominator} gives
\[
    \Pi_n(A_\epsilon)
    =
    \frac{\int_{A_\epsilon}R_n(\mu)\pi(d\mu)}{D_n}
    \to0
\]
in \(\Ptrue\)-probability.  This proves the first part of
\eqref{eq:retel-consistency}.

Next fix \(\epsilon>0\).  Let \(\eta\) and \(C\) be as in
\Cref{lem:retel-local-gamma}, and choose \(\rho\in(0,\eta)\) such that
\(C\rho<\epsilon/2\).  Since \(\tau_n/n\to0\), with probability tending to
one,
\[
    |\mu-\bar X_n|\leq\rho
    \quad\Longrightarrow\quad
    |\gamma_{n,\mu}|
    \leq
    C\left(\rho+\frac{\tau_n}{n}\right)
    <\epsilon.
\]
Thus
\[
    \Pi_n(|\gamma_{n,\mu}|>\epsilon)
    \leq
    \Pi_n(|\mu-\bar X_n|>\rho)+o_{\Ptrue}(1).
\]
Furthermore,
\[
    \Pi_n(|\mu-\bar X_n|>\rho)
    \leq
    \Pi_n(|\mu-\mutrue|>\rho/2)
    +
    \mathbf 1\{|\bar X_n-\mutrue|>\rho/2\}.
\]
The first term tends to zero by the first part of \eqref{eq:retel-consistency}, and
the second term by the law of large numbers.  Hence the second part of
\eqref{eq:retel-consistency} follows.

It remains to prove predictive consistency.  We first record a deterministic
bound.

\begin{lemma}
For every \(\mu\) for which the fitted distribution is defined,
\begin{equation}
\label{eq:retel-wass-bound-app}
    \Wassone(\widetilde{\bbP}_{n,\mu},\bbP_n)
    \leq
    2|\gamma_{n,\mu}|
    +
    \frac{\tau_n}{n+\tau_n}
    +
    \frac{\tau_n}{n}e^{2|\gamma_{n,\mu}|}.
\end{equation}

\end{lemma}
\begin{proof}
Indeed, by Kantorovich--Rubinstein duality, or by the one-dimensional CDF
formula for \(\Wassone\), and because all distributions are supported on
\([0,1]\),
\[
    \Wassone(\widetilde{\bbP}_{n,\mu},\bbP_n)
    \leq
    \sum_{i=1}^n\left|w_{n,i}(\mu)-\frac1n\right|
    +
    w_{n,H}(\mu).
\]
The regularizer weight is bounded by
\[
    w_{n,H}(\mu)
    \leq
    \frac{\tau_n e^{|\gamma_{n,\mu}|}}{n e^{-|\gamma_{n,\mu}|}}
    =
    \frac{\tau_n}{n}e^{2|\gamma_{n,\mu}|}.
\]
For the empirical weights, view
\[
    w_{n,i}(\mu;\gamma)
    =
    \frac{e^{\gamma(X_i-\mu)}}
    {\sum_{j=1}^n e^{\gamma(X_j-\mu)}+\tau_n M_{\bbH_\mu}(\gamma)}.
\]
Then
\[
    \frac{d}{d\gamma}w_{n,i}(\mu;\gamma)
    =
    w_{n,i}(\mu;\gamma)\{(X_i-\mu)-m(\gamma)\},
\]
where \(m(\gamma)\in[-1,1]\).  Therefore
\[
    \sum_{i=1}^n\left|\frac{d}{d\gamma}w_{n,i}(\mu;\gamma)\right|
    \leq2,
\]
and hence
\[
    \sum_{i=1}^n
    \left|w_{n,i}(\mu;\gamma)-w_{n,i}(\mu;0)\right|
    \leq 2|\gamma|.
\]
Since \(w_{n,i}(\mu;0)=1/(n+\tau_n)\),
\[
    \sum_{i=1}^n\left|w_{n,i}(\mu)-\frac1n\right|
    \leq
    2|\gamma_{n,\mu}|
    +
    \frac{\tau_n}{n+\tau_n},
\]
which proves \eqref{eq:retel-wass-bound-app}.
\end{proof}

Finally,
\begin{align*}
    \Wassone(\bbQ_{n+1\mid n},\Ptrue)
    &\leq
    \Wassone(\bbP_n,\Ptrue)
    +
    \Wassone(\bbQ_{n+1\mid n},\bbP_n)\\
    &\leq
    \Wassone(\bbP_n,\Ptrue)
    +
    \int_0^1
    \Wassone(\widetilde{\bbP}_{n,\mu},\bbP_n)\Pi_n(d\mu).
\end{align*}
The empirical term satisfies \(\Wassone(\bbP_n,\Ptrue)\to0\) almost surely,
because \([0,1]\) is compact and \(\bbP_n\) converges weakly to \(\Ptrue\)
almost surely.  For the second term, fix \(\delta>0\).  Since
\(\Wassone\leq1\) on \([0,1]\), \eqref{eq:retel-wass-bound-app} gives
\[
    \int_0^1
    \Wassone(\widetilde{\bbP}_{n,\mu},\bbP_n)\Pi_n(d\mu)
    \leq
    2\delta
    +
    \frac{\tau_n}{n+\tau_n}
    +
    \frac{\tau_n}{n}e^{2\delta}
    +
    \Pi_n(|\gamma_{n,\mu}|>\delta).
\]
Taking \(n\to\infty\) and using \(\tau_n=o(n)\) and the second part of
\eqref{eq:retel-consistency}, the limsup of the second term is at
most \(2\delta\) in \(\Ptrue\)-probability.  Since \(\delta>0\) is
arbitrary, this term converges to zero in \(\Ptrue\)-probability, proving the consistency of the RETEL predictive in Wasserstein distance.

% !TEX root = bacs-main.tex
\section{Asymptotic log-optimality for multivariate bounded means}
\label{app:multivariate}

We extend \cref{thm:logoptimality} to the vector betting rule of
\cref{sec:extensions}. For a fixed candidate mean $\mu$ and $c\in(0,1)$, write
\[
\Lambda_{\mu,c}
=\bigl\{\lambda\in\bbR^d:\inf_{x\in K}
       \{1+\lambda^\top(x-\mu)\}\geq 1-c\bigr\},
\qquad
f_\lambda(x)=\log\{1+\lambda^\top(x-\mu)\}.
\]

\begin{theorem}[Multivariate asymptotic log-optimality]
\label{thm:multivariate_logoptimality}
Let $(X_n)_{n\geq1}$ be \iid from a distribution $\Ptrue$ supported on a known
compact convex set $K\subset\bbR^d$, and suppose that $X_1$ has positive
definite covariance. Fix $\mu$ in the interior of $K$ and $c\in(0,1)$.
Let $\calF_n=\sigma(X_1,\ldots,X_n)$ and let $\bbQ_{n\mid n-1}$ be an
$\calF_{n-1}$-measurable predictive distribution on $K$ such that
$\Wassone(\bbQ_{n\mid n-1},\Ptrue)\to0$ in probability, where $\Wassone$
uses Euclidean distance. Define
\[
M(\lambda,\mu)=\E_{\Ptrue}[f_\lambda(X_1)],
\qquad
M_n(\lambda,X_{1:n-1},\mu)
=\int_K f_\lambda(x)\,\bbQ_{n\mid n-1}(dx\mid X_{1:n-1}).
\]
Let $\lambdahat_n(\mu)$ be an $\calF_{n-1}$-measurable maximiser of $M_n$
over $\Lambda_{\mu,c}$, and set
$\Wn=\prod_{i=1}^n\{1+\lambdahat_i(\mu)^\top(X_i-\mu)\}$.
Then $M(\cdot,\mu)$ has a unique maximiser $\lambdaopt(\mu)$ over
$\Lambda_{\mu,c}$, and $\lambdahat_n(\mu)\to\lambdaopt(\mu)$ in probability.
Moreover,
\begin{align*}
\frac{1}{n}\log\Wn &\longrightarrow M(\lambdaopt(\mu),\mu)
    \qquad\text{in probability},\\
\frac{1}{n}\E_{\Ptrue}[\log\Wn] &\longrightarrow M(\lambdaopt(\mu),\mu).
\end{align*}
\end{theorem}

\begin{proof}
Fix $\mu$ and $c$ throughout. Write $Q_n=\bbQ_{n\mid n-1}(\cdot\mid X_{1:n-1})$,
$M_n(\lambda)=M_n(\lambda,X_{1:n-1},\mu)$ and $M(\lambda)=M(\lambda,\mu)$.
Expectations and probabilities below are under the true law of the data.
All norms are Euclidean.

\emph{1. The admissible set and uniform bounds.}
Since $\mu$ is in the interior of $K$, there is an $r>0$ such that the closed
ball of radius $r$ about $\mu$ lies in $K$. If
$\lambda\in\Lambda_{\mu,c}$ is nonzero, then
$x=\mu-r\lambda/\|\lambda\|$ belongs to that ball. Admissibility at this
$x$ gives $1-r\|\lambda\|\geq 1-c$, so $\|\lambda\|\leq R:=c/r$.
Moreover,
\[
\Lambda_{\mu,c}=\bigcap_{x\in K}
  \{\lambda\in\bbR^d:\lambda^\top(x-\mu)\geq-c\}.
\]
This intersection of closed half-spaces is closed and convex, contains
zero, and is bounded by the preceding calculation. It is therefore compact.

Put $D=\operatorname{diam}(K)$. For all $\lambda\in\Lambda_{\mu,c}$ and
$x\in K$, the lower bound follows from admissibility and the upper bound
from Cauchy--Schwarz. Together they give
\[
 1-c\leq 1+\lambda^\top(x-\mu)\leq 1+RD.
\]
Thus $|f_\lambda(x)|\leq B:=\max\{-\log(1-c),\log(1+RD)\}$.
On $[1-c,\infty)$ the derivative of the logarithm is at most $1/(1-c)$.
Applying the mean-value theorem, first with $x$ varying and then with
$\lambda$ varying, gives
\begin{align*}
 |f_\lambda(x)-f_\lambda(y)|
 &\leq \frac{|\lambda^\top(x-y)|}{1-c}
 \leq L\|x-y\|,\qquad L:=\frac{R}{1-c},\\
 |f_\lambda(x)-f_{\lambda'}(x)|
 &\leq \frac{|(\lambda-\lambda')^\top(x-\mu)|}{1-c}
 \leq \frac{D}{1-c}\|\lambda-\lambda'\|.
\end{align*}
The second inequality shows that $M$ and, at each history, $M_n$ are
continuous in $\lambda$. Then, both attain their maxima on the compact $\Lambda_{\mu,c}$.
For each fixed $\lambda$, $M_n(\lambda)$ is $\calF_{n-1}$-measurable because
$Q_n$ is a measurable probability kernel. The measurable maximum theorem \cite[Theorem 18.19]{Aliprantis2006} for a continuous objective on a fixed compact metric space therefore gives
an $\calF_{n-1}$-measurable choice from its maximisers.

\emph{2. A uniform bound on the predictive objective.}
Fix a realised history, so that $Q_n$ is a fixed probability measure.
Let $\Gamma(Q_n,\Ptrue)$ be the set of probability measures $\gamma$ on
$K\times K$ whose first marginal is $Q_n$ and whose second marginal is
$\Ptrue$. For any $\gamma\in\Gamma(Q_n,\Ptrue)$ and any admissible $\lambda$,
the two marginal identities and the first inequality above give
\begin{align*}
 |M_n(\lambda)-M(\lambda)|
 &=\left|\int_K f_\lambda\,dQ_n-\int_K f_\lambda\,d\Ptrue\right|\\
 &=\left|\int_{K\times K}\{f_\lambda(x)-f_\lambda(y)\}\,\gamma(dx,dy)\right|\\
 &\leq \int_{K\times K}|f_\lambda(x)-f_\lambda(y)|\,\gamma(dx,dy)\\
 &\leq L\int_{K\times K}\|x-y\|\,\gamma(dx,dy).
\end{align*}
The left-hand side is the same for every coupling. By definition, the infimum
of the last integral over $\Gamma(Q_n,\Ptrue)$ is $\Wassone(Q_n,\Ptrue)$.
The constant $L$ is the same for every admissible stake. Taking first the
infimum over couplings and then the supremum over stakes yields
\begin{equation}
 Z_n:=\sup_{\lambda\in\Lambda_{\mu,c}}|M_n(\lambda)-M(\lambda)|
 \leq L\Wassone(Q_n,\Ptrue).
 \label{eq:mv_uniform_objectives}
\end{equation}
By continuity, the supremum equals the supremum over any fixed countable
dense subset of $\Lambda_{\mu,c}$, so $Z_n$ is measurable. Predictive
consistency implies $Z_n\to0$ in probability.

\emph{3. Uniqueness of the oracle and convergence of the stakes.}
Write $\Sigma$ for the covariance matrix of $X_1$. For distinct admissible
stakes $\lambda,\lambda'$, put $v=\lambda-\lambda'$. Then
$\var(v^\top(X_1-\mu))=v^\top\Sigma v>0$, so the two positive wealth factors
$1+\lambda^\top(X_1-\mu)$ and $1+{\lambda'}^\top(X_1-\mu)$ differ with
positive probability. For $0<t<1$ the wealth factor at $t\lambda+(1-t)\lambda'$
is their convex combination. Concavity of the logarithm gives a pointwise
concavity inequality which is strict whenever the two factors differ.
Taking expectations therefore gives
\[
 M(t\lambda+(1-t)\lambda')>tM(\lambda)+(1-t)M(\lambda').
\]
Thus $M$ is strictly concave on the convex admissible set and its maximiser
$\lambda^\ast:=\lambdaopt(\mu)$ is unique.

Write $\lambdahat_n=\lambdahat_n(\mu)$. Since
$M_n(\lambda^\ast)-M_n(\lambdahat_n)\leq0$, adding and subtracting both
predictive objective values yields the pathwise bound
\begin{align}
 0\leq M(\lambda^\ast)-M(\lambdahat_n)
  &= [M(\lambda^\ast)-M_n(\lambda^\ast)]
     +[M_n(\lambda^\ast)-M_n(\lambdahat_n)] \notag\\
  &\quad +[M_n(\lambdahat_n)-M(\lambdahat_n)]
  \leq 2Z_n. \label{eq:mv_oracle_gap}
\end{align}
The same uniform bound applies at the random stake because it holds for
every stake at each realised history. Fix $\varepsilon>0$ and let
$A_\varepsilon=\{\lambda\in\Lambda_{\mu,c}:
\|\lambda-\lambda^\ast\|\geq\varepsilon\}$. If $A_\varepsilon$ is empty,
then $\|\lambdahat_n-\lambda^\ast\|<\varepsilon$ for all $n$. Otherwise
$A_\varepsilon$ is compact and excludes the unique maximiser, so
$\delta_\varepsilon=M(\lambda^\ast)-\max_{\lambda\in A_\varepsilon}M(\lambda)>0$.
On the event $\{\lambdahat_n\in A_\varepsilon\}$, the left-hand gap in
\eqref{eq:mv_oracle_gap} is at least $\delta_\varepsilon$. Hence
\[
 \bbP(\|\lambdahat_n-\lambda^\ast\|\geq\varepsilon)
 \leq \bbP(Z_n\geq\delta_\varepsilon/2)\longrightarrow0.
\]

\emph{4. Realised and expected log-growth.}
Let $Y_i=f_{\lambdahat_i}(X_i)$. Conditional on $\calF_{i-1}$, the stake
$\lambdahat_i$ is known and $X_i$ has law $\Ptrue$ by independence. Thus
\[
 \E[Y_i\mid\calF_{i-1}]
 =\int_K f_{\lambdahat_i}(x)\,\Ptrue(dx)=M(\lambdahat_i).
\]
The centred increments $D_i=Y_i-M(\lambdahat_i)$ have conditional mean zero.
For $i<j$, $D_i$ is $\calF_{j-1}$-measurable, so
$\E[D_iD_j]=\E[D_i\,\E(D_j\mid\calF_{j-1})]=0$.
Also $\E[D_i^2]=\E[\var(Y_i\mid\calF_{i-1})]\leq\E[Y_i^2]\leq B^2$.
Expanding the square and using these two facts gives
\[
 \E\left[\left(\frac1n\sum_{i=1}^n D_i\right)^2\right]
 =\frac1{n^2}\sum_{i=1}^n\E[D_i^2]\leq \frac{B^2}{n}.
\]
Since $\log\Wn=\sum_{i=1}^nY_i$, subtracting $M(\lambda^\ast)$ gives
\[
 \frac1n\log\Wn-M(\lambda^\ast)
 =\frac1n\sum_{i=1}^n D_i
  -\frac1n\sum_{i=1}^n\{M(\lambda^\ast)-M(\lambdahat_i)\}.
\]
The terms in braces are nonnegative. Applying Cauchy--Schwarz to the
first average and \eqref{eq:mv_uniform_objectives}--\eqref{eq:mv_oracle_gap}
to the second yields
\[
 \E\left|\frac1n\log\Wn-M(\lambda^\ast)\right|
 \leq \frac{B}{\sqrt n}
   +\frac{2L}{n}\sum_{i=1}^n\E[\Wassone(Q_i,\Ptrue)].
\]
Both laws are supported on $K$, so $\Wassone(Q_i,\Ptrue)\leq D$.
For every $\eta>0$,
$\E[\Wassone(Q_i,\Ptrue)]\leq\eta+D\,\bbP(\Wassone(Q_i,\Ptrue)>\eta)$.
First letting $i\to\infty$ and then $\eta\downarrow0$ shows that these
expectations tend to zero. Their Ces\`aro averages therefore tend to zero
as well. This proves $L^1$ convergence of $n^{-1}\log\Wn$ to $M(\lambda^\ast)$.
Convergence in probability follows by Markov's inequality. The expected
log-growth conclusion follows from
$|\E[n^{-1}\log\Wn]-M(\lambda^\ast)|\leq
\E|n^{-1}\log\Wn-M(\lambda^\ast)|$.
\end{proof}

\paragraph{Scalar specialisation.}
For $d=1$, $K=[0,1]$ and $0<\mu<1$, the affine wealth factor attains its
minimum at an endpoint. Its two constraints are $1-\lambda\mu\geq1-c$ and
$1+\lambda(1-\mu)\geq1-c$, giving exactly
$\Lambda_{\mu,c}=[-c/(1-\mu),c/\mu]=I_{\mu,c}$.
Positive definite covariance reduces to positive variance, and taking
$r=\min\{\mu,1-\mu\}$ gives $L=c/[(1-c)\min\{\mu,1-\mu\}]$.
Thus the theorem yields all three conclusions of \cref{thm:logoptimality}
with the same admissible interval and Lipschitz constant.

\section{Computational details for BETEL/RETEL predictives}
\label{app:betel-computation}

For fixed \(\mu\), let
\[
    M_{\bbH_\mu}(\gamma)=\int e^{\gamma(z-\mu)}\,d\bbH_\mu(z),
    \qquad
    Z_{n,\mu}
    =
    \sum_{j=1}^n e^{\gamma_{n,\mu}(X_j-\mu)}
    +
    \tau_n M_{\bbH_\mu}(\gamma_{n,\mu}).
\]
Then the tilted distribution in \eqref{eq:retel-gamma} can be written as
\begin{equation}
\label{eq:retel-mixture}
    \widetilde{\bbP}_{n,\mu}
    =
    \sum_{i=1}^n w_{n,i}(\mu)\delta_{X_i}
    +
    w_{n,H}(\mu)\widetilde{\bbH}_{n,\mu},
\end{equation}
where
\begin{align}
\label{eq:retel-weights}
    w_{n,i}(\mu)
    &=
    \frac{e^{\gamma_{n,\mu}(X_i-\mu)}}{Z_{n,\mu}},
    \qquad i=1,\ldots,n,\\
    w_{n,H}(\mu)
    &=
    \frac{\tau_n M_{\bbH_\mu}(\gamma_{n,\mu})}{Z_{n,\mu}},
\end{align}
and
\[
    \widetilde{\bbH}_{n,\mu}(A)
    =
    \frac{\int_A e^{\gamma_{n,\mu}(z-\mu)}\,d\bbH_\mu(z)}
    {M_{\bbH_\mu}(\gamma_{n,\mu})}.
\]
Some regularized ETEL likelihoods include the atom correction
\[
    \prod_{i=1}^n \{1+\tau_n\bbH_\mu(\{X_i\})\}.
\]
This factor is equal to one for continuous regularizers and is independent of
\(\mu\) for the two-point regularizer used in our experiments, so it does not
change the pseudo-posterior.

For the two-point choice
\(\bbH=\frac12(\delta_0+\delta_1)\), define
\[
    S_0(\gamma)=\sum_{i=1}^n e^{\gamma X_i},
    \qquad
    S_1(\gamma)=\sum_{i=1}^n X_i e^{\gamma X_i}.
\]
The tilting equation \eqref{eq:retel-gamma} becomes
\begin{equation}
\label{eq:two-point-gamma}
    \mu
    =
    \frac{S_1(\gamma_{n,\mu})+\frac{\tau_n}{2}e^{\gamma_{n,\mu}}}
    {S_0(\gamma_{n,\mu})+\frac{\tau_n}{2}(1+e^{\gamma_{n,\mu}})}.
\end{equation}
The corresponding weights are
\begin{equation}
\label{eq:two-point-weights}
    w_{n,i}(\mu)
    =
    \frac{e^{\gamma_{n,\mu}X_i}}
    {S_0(\gamma_{n,\mu})+\frac{\tau_n}{2}(1+e^{\gamma_{n,\mu}})},
    \qquad
    w_{n,H}(\mu)
    =
    \frac{\frac{\tau_n}{2}(1+e^{\gamma_{n,\mu}})}
    {S_0(\gamma_{n,\mu})+\frac{\tau_n}{2}(1+e^{\gamma_{n,\mu}})}.
\end{equation}
Finally, the tilted regularizer remains supported on \(\{0,1\}\):
\[
    \widetilde{\bbH}_{n,\mu}
    =
    (1-p_{n,\mu})\delta_0+p_{n,\mu}\delta_1,
    \qquad
    p_{n,\mu}
    =
    \frac{e^{\gamma_{n,\mu}}}{1+e^{\gamma_{n,\mu}}}.
\]
\section{Implementation details}
\label{app:implementation_details}

The methods proposed in this paper were implemented in Python 3.12. The code to reproduce all experiments is available at \url{https://github.com/ValentinKil/bacsbm.git}. 
Except for the LLM experiment, all experiments were run on a CPU-only cluster with a mix of Intel and AMD nodes (Intel Xeon Platinum 8268, AMD EPYC 9634, and AMD EPYC 9645). The LLM experiment was run on an Apple M4 Max system with 36 GB of RAM.
A runtime comparison for a subset of the synthetic experiments is reported separately by processor in \cref{tab:runtime}.

\subsection{Computation of predictive-assisted confidence sequences}
\label{app:implementation_cs}

All confidence sequences based on the predictive-assisted construction in \cref{sec:predictive_assisted_test_martingales} are computed by discretising the test inversion \eqref{cs}.
For $G$ equal subdivisions of $[0,1]$, we use the equally spaced grid
\[
    \mathcal G
    =
    \left\{
        \frac{1}{G},
        \frac{2}{G},
        \ldots,
        \frac{G-1}{G}
    \right\}
    \subset (0,1).
\]
For each candidate mean $\mu \in \mathcal G$ and each time $n$, the method computes a predictable coefficient
\[
    \widehat\lambda_n(\mu)
    \in
    I_{\mu,c}
    =
    \left[
        -\frac{c}{1-\mu},
        \frac{c}{\mu}
    \right],
\]
for a given truncation parameter $c \in (0, 1)$, using only the history $X_{1:n-1}$.
The corresponding test-martingale value is then updated through the one-step factor
\[
    1 + \widehat\lambda_t(\mu)(X_t-\mu).
\]
For numerical stability, the implementation works with cumulative log-values,
\[
    \log W_n(X_{1:n},\mu)
    =
    \sum_{i=1}^n
    \log\{1+\widehat\lambda_i(\mu)(X_i-\mu)\},
    \qquad
    \mu\in\mathcal G .
\]
At time $n$, the retained grid points are
\[
    \mathcal A_n
    =
    \left\{
        \mu\in\mathcal G:
        \log W_n(X_{1:n},\mu)\leq \log(1/\alpha)
    \right\}.
\]
The reported discretised confidence set is the smallest interval covering the retained grid points, enlarged by one grid step and clipped to $[0,1]$:
\[
    \widetilde C_{\alpha,n}
    =
    \left[
        \max\{0,\min\mathcal A_n-G^{-1}\},
        \min\{1,\max\mathcal A_n+G^{-1}\}
    \right].
\]
The enlargement by one grid step accounts for the finite grid resolution when reporting an interval from the retained grid points.
Unless otherwise stated, the displayed confidence sequence is the running intersection
\[
    C_{\alpha,n}^{\mathrm{run}}
    =
    \bigcap_{s=1}^n \widetilde C_{\alpha,s},
\]
which enforces nested intervals while preserving anytime-validity \citep[see, e.g.,][]{WaudbySmith2024a}.

For a generic predictable distribution $\mathbb Q_{n\mid n-1}(\cdot\mid X_{1:n-1})$, the coefficient is computed according to \eqref{eq:bayes_assisted_lambda}: for each $\mu\in\mathcal G$,
\[
\lambdahat_n(X_{1:n-1},\mu)
\in
{\arg\max}_{\lambda\in I_{\mu,c}}
\left\{
\E_{\widetilde X_n\sim \predictive}
\left[
\log\left(1+\lambda(\widetilde X_n-\mu)\right)
\right]
\right\},
\]
which is a concave objective in $\lambda$.
In the implementation, the maximiser is obtained by solving the first-order condition
\[
\E_{\widetilde X_n\sim \predictive}
\left[
\frac{\widetilde X_n-\mu}{1+\lambda(\widetilde X_n-\mu)} = 0
\right]
\]
whenever the solution lies in the interior of $I_{\mu,c}$.
Otherwise, the appropriate endpoint of $I_{\mu,c}$ is selected.
For discrete predictive distributions, the integral above is evaluated as a finite weighted sum.

\subsection{Predictive distributions}
\label{app:implementation_predictives}

The implemented methods differ only in the construction of the predictable distribution $\mathbb Q_{n\mid n-1}(\cdot\mid X_{1:n-1})$ used in \eqref{eq:bayes_assisted_lambda}.

\paragraph{Empirical predictive.}
The empirical method uses the empirical distribution of the previous observations,
\[
    \mathbb P_{n-1}
    =
    \frac{1}{n-1}\sum_{i=1}^{n-1}\delta_{X_i},
    \qquad n\geq 2,
\]
as the predictive distribution. Thus, for each candidate value $\mu$, the
coefficient is computed by maximising
\[
    \frac{1}{n-1}
    \sum_{i=1}^{n-1}
    \log\{1+\lambda(X_i-\mu)\}
\]
over $\lambda\in I_{\mu,c}$. Equivalently, when the optimum is interior, it
solves
\begin{equation}
    \frac{1}{n-1}
    \sum_{i=1}^{n-1}
    \frac{X_i-\mu}{1+\lambda(X_i-\mu)}
    =
    0.
    \label{eq:empirical_lambda_condition}
\end{equation}

\paragraph{Parametric Bayesian predictive.}
The parametric Bayesian method uses a beta working model in mean--concentration parametrisation,
\begin{equation}
    X_i\mid \rho,\nu
    \sim
    \mathrm{Beta}(\rho\nu,(1-\rho)\nu),
    \qquad
    \rho\in(0,1),\quad \nu>0.
    \label{eq:beta_working_model}
\end{equation}
In experiments, observations supplied to the parametric and MDP predictive calculations are clipped to $[\epsilon_{\mathrm{mach}},1-\epsilon_{\mathrm{mach}}]$, where $\epsilon_{\mathrm{mach}}$ is double-precision machine epsilon, preventing singular beta-likelihood evaluations; the resulting stakes are evaluated on the original observations in the test martingale.
Posterior inference is approximated using the IBIS/waste-free SMC sampler from the \texttt{particles} package \citep{chopin2020introduction}.
With $N_\text{smc}$ particles and Markov chains of length $L_\text{smc}$, each time point stores
\[
    M
    =
    N_\text{smc} \times L_\text{smc}
\]
posterior samples. Conditional on these samples, the predictive score
\[
    \mathbb E
    \left[
        \frac{X-\mu}{1+\lambda(X-\mu)}
        \,\middle|\,
        X\sim \mathrm{Beta}(a,b)
    \right]
\]
is evaluated for each sampled beta distribution using the closed-form
hypergeometric expression, and the resulting scores are averaged over the SMC
particle cloud. The root of the averaged score is then solved independently for
each $\mu\in\mathcal G$.

\paragraph{MDP predictive.}
The MDP method implements the predictive distribution in \eqref{eq:dirichlet}.
Let $\mathbb{H}_{n\mid n-1}(\cdot \mid X_{1:n-1})$ and $\mathbb{P}_{n-1}$ denote the beta posterior predictive and the empirical distribution after observing $X_{1:n-1}$, respectively.
For $\kappa>0$, the implemented predictive is
\[
    \mathbb Q_{n\mid n-1}
    =
    \frac{\kappa}{\kappa+n-1}\mathbb H_{n\mid n-1}
    +
    \frac{n-1}{\kappa+n-1}\mathbb P_{n-1}.
\]
Accordingly, the first-order condition for
$\widehat\lambda_t(\mu)$ is evaluated as the same convex combination of the
parametric Bayesian score and the empirical score. As $n$ increases, the empirical component
dominates, so the predictive distribution becomes increasingly data-driven.

\paragraph{BETEL predictive.}
The BETEL method implements the unregularised empirical-likelihood predictive from
\cref{subsec:betel}, corresponding to $\tau_n=0$.
After observing $X_{1:n-1}$, the pseudo-posterior integral \eqref{eq:retel-posterior_and_predictive} is approximated on a deterministic grid of mean values $\mu_0\in(0,1)$.
In the unregularised case, the implementation constructs this grid inside the interior of the empirical convex hull of $X_{1:n-1}$, corresponding to the convention that the empirical likelihood is zero outside this hull.
For each retained $\mu_0$, the tilting parameter $\gamma_{n,\mu_0}$ and the weights $w_{n,i}(\mu_0)$ are computed numerically as described in \cref{subsec:betel}.
The pseudo-posterior weights are then normalised over the retained grid points, using a log-sum-exp calculation for numerical stability.

The resulting predictive distribution is the grid approximation to
\[
    \mathbb Q_{n\mid n - 1}(\cdot\mid X_{1:n-1})
    =
    \int_0^1 \widetilde P_{n,\mu_0}(\cdot)\,\Pi_{n-1}(d\mu_0),
\]
and is represented in the code as a discrete distribution supported on the observed sample points $X_1,\ldots,X_{n-1}$.
This discrete predictive is then used in \eqref{eq:bayes_assisted_lambda} to compute $\widehat\lambda_n(\mu)$ for each candidate value $\mu$ in the confidence-sequence inversion grid.
If the history is empty or degenerate, the implementation uses the safe default $\widehat\lambda_n(\mu)=0$ for all candidate values $\mu$.

\paragraph{RETEL predictive.}
The RETEL method implements the regularised empirical-likelihood predictive from
\cref{subsec:betel}.
In all reported experiments, we use the two-point regulariser
\[
    \mathbb H_{\mu_0} = \mathbb H = \frac{1}{2}(\delta_0+\delta_1),
    \qquad
    \tau_n=1.
\]
After observing $X_{1:n-1}$, the pseudo-posterior integral \eqref{eq:retel-posterior_and_predictive} is approximated on a deterministic grid of values $\mu_0\in(0,1)$.
For each grid value of $\mu_0$, the tilting parameter $\gamma_{n,\mu_0}$, the sample-point weights, and the endpoint weights are computed using the formulas in \cref{app:betel-computation}.
The pseudo-posterior weights are then normalised on the grid, again using log-sum-exp normalisation.

The resulting predictive distribution is the grid approximation to
\[
    \mathbb Q_{n \mid n-1}(\cdot \mid X_{1:n-1})
    =
    \int_0^1 \widetilde P_{n,\mu_0}(\cdot)\,\Pi_{n-1}(d\mu_0).
\]
In contrast to BETEL, this predictive has support on both the observed sample points and the two endpoints $\{0,1\}$.
It is then substituted into \eqref{eq:bayes_assisted_lambda} to compute $\widehat\lambda_{n}(\mu)$ over the confidence-sequence inversion grid.
The endpoint masses stabilise the empirical-likelihood approximation and make the regularised construction well defined over the full mean range.

\subsection{Computational complexity}
\label{app:computational_complexity}

Let $n$ denote the sequence length and let $G$ be the number of equal subdivisions used for test inversion, giving $G-1$ interior candidate means.
All CS methods based on \eqref{cs}, like the ones above, evaluate wealth processes over the $n\times (G-1)$ time--mean grid.
Thus, after the predictable coefficients have been computed, test inversion itself contributes a common $\mathcal O(nG)$ time cost.

However, the dominant cost is usually the computation of the predictable coefficients itself.
For the empirical predictive, exact evaluation of the first-order condition \eqref{eq:empirical_lambda_condition} requires summing over the current history $X_{1:t-1}$ at each time $t$ and each grid value.
This gives $\mathcal O(G n^2)$ operations up to root-solver iterations.
The MDP method inherits the same empirical-support term, and also averages the parametric predictive contribution over posterior particles.
BETEL and RETEL also have growing empirical-support dependence: at each time they solve empirical-likelihood tilting equations and form pseudo-predictive weights using sums over $X_{1:t-1}$.
With a fixed pseudo-posterior grid size $G_\text{etel}$, their exact implementation therefore contains terms of order $\mathcal O(G_\text{etel} n^2)$ for the tilting and pseudo-posterior construction, as well as $\mathcal O(G n^2)$ for the subsequent weighted empirical betting optimisation, again up to solver iterations.
Consequently, the empirical, MDP, BETEL, and RETEL methods have quadratic dependence on $n$ and linear dependence on $G$ in the current exact implementation.

By contrast, the parametric Bayesian SMC method avoids the growing empirical support in the coefficient computation.
With a fixed particle budget, its coefficient-computation cost is linear in $n$ and $G$, up to particle, root-solver, and SMC-update constants.

The quadratic dependence on sequence length is determined by the predictive representation and its numerical implementation, rather than by the use of prior information itself. Nonetheless, wall-clock costs can still differ substantially among the Bayes-assisted methods proposed in this work because of the additional grids, posterior computations, and numerical optimizations. \cref{tab:runtime} reports method-only wall-clock times for computing the CS until $n=1000$ in the beta-mixture experiment for the empirical predictive, parametric Bayes, and the proposed Bayes-assisted methods. Since these methods share the predictive-assisted CS construction, the comparison isolates the computational cost of their different predictive representations. Results are stratified by processor across all 1000 paired replicates.

\begin{table}
\centering
\caption{Wall-clock time in seconds of one method call computing the CS at $n=1000$ on the Beta mixture scenario, by CPU: mean (standard error) over replicates.}
\label{tab:runtime}
\begin{tabular}{llccc}
\toprule
Method & Prior & \shortstack{Intel Xeon Platinum\\8268} & \shortstack{AMD EPYC\\9634} & \shortstack{AMD EPYC\\9645} \\
\midrule
Empirical & - & 18.63 (0.31) & 7.78 (0.12) & 5.90 (0.65) \\
MDP & Informative & 2223.19 (25.77) & 1445.81 (12.11) & 1276.17 (74.21) \\
BETEL & Informative & 189.96 (2.77) & 108.80 (1.38) & 106.04 (11.48) \\
RETEL & Informative & 551.69 (7.30) & 432.89 (5.96) & 448.47 (40.94) \\
Parametric & Informative & 1864.92 (20.55) & 1336.11 (14.09) & 1208.72 (82.34) \\
\midrule
Replicates & - & 265 & 694 & 41 \\
\bottomrule
\end{tabular}
\end{table}

Thus, in their current implementations, the Bayes-assisted methods are substantially slower than the empirical predictive. This additional computation may nevertheless be acceptable when the cost of acquiring observations dominates the cost of computing the CS.

\subsection{Default numerical parameters}
\label{app:default_numerical_parameters}
Unless otherwise specified, all experiments use confidence level $\alpha=0.1$ and $G=500$ equal grid subdivisions, giving spacing $1/500$ and 499 interior candidate means for test inversion. The empirical and Bayes-assisted methods use $c=0.95$, HC uses its default $c=0.5$, and PRECiSE-CO96 and UP use their default full-range constructions, corresponding to $c=1$.
For the SMC-based beta predictives we use $N_\text{smc} = 20$ and $L_\text{smc} = 50$, giving $M=1000$ posterior samples at each time point.
For BETEL and RETEL we use $G_\text{etel}=1000$ grid points for the pseudo-posterior integral.
For the MDP predictive, we use $\kappa=50$; RETEL uses the two-point regulariser with $\tau_n=1$, while BETEL corresponds to $\tau_n=0$.

Moreover, the Bayes-assisted methods, including the parametric Bayesian, MDP, BETEL, and RETEL predictives, require prior specifications.
In particular, the parametric Bayesian and MDP methods require a prior distribution over the beta working model mean--concentration parameters $\rho$ and $\nu$ in \cref{eq:beta_working_model}. This usually takes the form
\[
    \rho \sim \mathrm{Beta}(a_\rho,b_\rho),
    \qquad
    \nu \sim \mathrm{Gamma}(a_\nu,b_\nu)
\]
for some hyperparameters $a_\rho,b_\rho,a_\nu,b_\nu$.
The BETEL and RETEL methods require a prior distribution directly over the mean $\mu$, which we take to be a beta distribution of the form
\[
    \mu \sim \mathrm{Beta}(a_m,b_m).
\]
In practice, we usually set $a_m = a_\rho$ and $b_m = b_\rho$.

\section{Continuum-valid cellwise inversion}
\label{app:cellwise}

The main experiments invert test martingales only at finitely many candidate means, so their validity guarantee applies only to those grid points. Taking the hull of the accepted points and enlarging it by one grid step is a useful reporting convention, but does not by itself certify coverage between grid points. Here we instead partition the parameter space into cells, use stakes that are valid uniformly within each cell, and reject a cell only when endpoint evaluations certify rejection throughout it. The union of the retained cells is therefore a confidence sequence over all $\mu\in[0,1]$.

\subsection{Cellwise Betting Construction}

Fix a partition of \([0,1]\):
$
0=a_0<a_1<\cdots<a_G=1,
$
and define cells
\[
B_j=[a_{j-1},a_j],
\qquad j=1,\ldots,G.
\]
Write \(B_j=[\ell_j,u_j]\). For each cell, define the cellwise-safe stake
interval
\[
J_j
=
\bigcap_{\mu\in B_j\cap(0,1)} I_{\mu,c}
=
\left[
-\frac{c}{1-\ell_j},
\frac{c}{u_j}
\right].
\]
Thus every \(\lambda\in J_j\) is valid simultaneously for every
\(\mu\in B_j\cap(0,1)\).
At time \(t\), before observing \(X_t\), choose a predictable stake
$
\lambda_{t,j}\in J_j.
$
This choice can be arbitrary from the perspective of validity. For efficiency,
one natural choice is to optimize predictive log-growth at the cell midpoint
$
m_j=\frac{\ell_j+u_j}{2}.
$
Given a predictable distribution \(Q_{t\mid t-1}\), set
\[
\lambda_{t,j}
\in
\arg\max_{\lambda\in J_j}
\mathbb E_{X\sim Q_{t\mid t-1}}
\log\{1+\lambda(X-m_j)\}.
\]

This is a valid predictable stake for all $\mu\in B_j$, hence validity is preserved.  

\subsection{Endpoint Certificate}

For each cell \(B_j\), define the cellwise log-wealth
\[
\Lambda_{n,j}(\mu)
=
\sum_{t=1}^n
\log\{1+\lambda_{t,j}(X_t-\mu)\},
\qquad \mu\in B_j.
\]
For fixed \(\lambda_{t,j}\), each summand is concave in \(\mu\). Therefore \(\Lambda_{n,j}\) is concave on \(B_j\). A concave function on an
interval attains its minimum at an endpoint, so
\[
\inf_{\mu\in B_j} \Lambda_{n,j}(\mu)
=
\min\{\Lambda_{n,j}(\ell_j),\Lambda_{n,j}(u_j)\}.
\]

We reject the entire cell \(B_j\) at time \(n\) only if
\[
\min\{\Lambda_{n,j}(\ell_j),\Lambda_{n,j}(u_j)\}>\log(1/\alpha).
\]
Equivalently, the reported confidence set is
\[
\widehat C_{\alpha,n}
=
\bigcup_{j:\,
\min\{\Lambda_{n,j}(\ell_j),\Lambda_{n,j}(u_j)\}\le \log(1/\alpha)}
B_j.
\]
If a single interval is desired, one may report the convex hull of
\(\widehat C_{\alpha,n}\), which remains conservative.

An empirical comparison with pointwise grid inversion is reported in \cref{app:cellwise_experiments}.

% !TEX root = bacs-main.tex

\section{Experimental details}
\label{app:experimental_details}

This section provides additional details on the experiments presented in \cref{sec:experiments} and \cref{app:additional_results}.

\subsection{Synthetic experiments}
\subsubsection{Prior specification}
\label{app:prior_specifications_synthetic}

For each of the six distributions considered for the synthetic experiments, we consider three prior regimes: informative, non-informative, and misspecified.
Informative priors place significant mass around the true mean and represent the case in which the prior information is directionally correct. 
Non-informative priors are diffuse over $[0,1]$. 
Misspecified priors are concentrated away from the true mean and are used to test sensitivity to misleading prior information.
The specific prior specifications for each distribution and regime are as follows.

\paragraph{Bernoulli distributions.}
Let $X_n \sim \mathrm{Bernoulli}(p)$, $p \in \{0.1, 0.5\}$. For Parametric/MDP, we set:
\begin{itemize}
      \item Informative prior: $\rho \sim \mathrm{Beta}(500p, 500(1-p)),\quad \nu \sim \mathrm{Gamma}(1, 100)$.
      \item Non-informative prior: $\rho \sim \mathrm{Beta}(1, 1),\quad \nu \sim \mathrm{Gamma}(1.5, 1)$.
      \item Misspecified prior: $\rho \sim \mathrm{Beta}(500 p_\text{bad}, 500(1-p_\text{bad})),\quad \nu \sim \mathrm{Gamma}(7.5, 1.0)$, where $p_\text{bad} = 0.5$ if $p=0.1$ and $p_\text{bad} = 0.1$ if $p=0.5$.
\end{itemize}
For BETEL/RETEL, we set the $\mu$ priors to be equal to the corresponding $\rho$ priors above. 

\paragraph{Beta distributions.}
Let $X_n \sim \mathrm{Beta}(a, b)$. Consider the three cases $(a, b) = (0.5, 0.5),(1, 1), (10, 30)$, which give $\mu^\star = a / (a + b) = 0.5, 0.5, 0.25$. For Parametric/MDP, we set:
\begin{itemize}
      \item Informative prior: $\rho \sim \mathrm{Beta}(500\mu^\star, 500(1-\mu^\star)),\quad \nu \sim \mathrm{Gamma}(a_\nu, b_\nu)$, where $(a_\nu, b_\nu) = (2, 1), (7.5, 1), (2, 0.1)$ in the three cases.
      \item Non-informative prior: $\rho \sim \mathrm{Beta}(1, 1),\quad \nu \sim \mathrm{Gamma}(1.5, 1)$.
      \item Misspecified prior: $\rho \sim \mathrm{Beta}(500 \mu_\text{bad}, 500(1-\mu_\text{bad})),\quad \nu \sim \mathrm{Gamma}(a_\nu, b_\nu)$, where $\mu_\text{bad} = 0.9$ and $(a_\nu, b_\nu) = (2, 1), (7.5, 1), (2, 0.1)$ in the three cases.
\end{itemize}
For BETEL/RETEL, we set the $\mu$ priors to be equal to the corresponding $\rho$ priors above. 

\paragraph{Beta mixture distribution.}
Let $X_n$ be distributed as $0.25\,\mathrm{Beta}(5,15) + 0.75\,\mathrm{Beta}(15,5)$, which gives $\mu^\star = 0.625$. For Parametric/MDP, we set:
\begin{itemize}
      \item Informative prior: $\rho \sim \mathrm{Beta}(500\mu^\star, 500(1-\mu^\star)),\quad \nu \sim \mathrm{Gamma}(2, 2)$.
      \item Non-informative prior: $\rho \sim \mathrm{Beta}(1, 1),\quad \nu \sim \mathrm{Gamma}(1.5, 1)$.
      \item Misspecified prior: $\rho \sim \mathrm{Beta}(500 \mu_\text{bad}, 500(1-\mu_\text{bad})),\quad \nu \sim \mathrm{Gamma}(2, 2)$, where $\mu_\text{bad} = 0.1$.
\end{itemize}
For BETEL/RETEL, we set the $\mu$ priors to be equal to the corresponding $\rho$ priors above. 

\subsection{Best-arm identification for LLM evaluation}
\label{app:best_arm_details}

\Cref{app:LUCB} provides a detailed description of the LUCB algorithm \citep{Kalyanakrishnan2012, Kaufmann2013}, whereas \cref{app:LLM_eval} details the implementation of the LUCB algorithm for best-arm identification in LLM evaluation, and explains our experimental setup.

\subsubsection{Lower Upper Confidence Bounds algorithm}
\label{app:LUCB}

Imagine you have access to $n$ independent arms (or bandits); when played, these arms provide a reward sampled from an unknown distribution that characterizes them. The Best Arm Identification problem aims to identify the arm with the highest expected reward with high probability, using the minimum number of plays possible. This problem, known as the multi-armed bandit problem, is a classical challenge in decision theory.

The Lower Upper Confidence Bounds (LUCB) algorithm is a method designed to perform best arm identification in multi-armed bandit problems. It was introduced by \cite{Kalyanakrishnan2012} and refined by \cite{Kaufmann2013}. Unlike standard bandit algorithms which aim to maximize the cumulative reward while playing, LUCB is a purely exploratory algorithm. Its goal is strictly to identify the top-$m$ arms with high confidence using as few samples as possible. In our experiment, we focus only on best arm identification, but we present in this appendix the general algorithm that can perform best-$m$-arms identification. 

The central idea is that one does not need to know the precise mean of every arm, but only to identify the boundary between the top-$m$ arms and the remaining arms. To explain the algorithm, let us assume we have $n$ arms and want to identify the subset of size $m$ with the highest means. Consider a confidence level $\alpha \in (0,1)$. At any time step $t$ and for each arm $a$, we maintain:\begin{itemize}
    \item $\hat{\mu}_a(t)$, the empirical mean.
    \item $[L_a(t),U_a(t)]$ a $(1-\alpha)$-two-sided confidence sequence for the true mean of $a$.
   
\end{itemize}
Note that these values are not updated for every arm at every time step, but only for the arms that have been pulled. LUCB is agnostic to the methods used to derive the upper and lower confidence bounds, provided they offer proper statistical guarantees. While the choice of bound does not alter the correctness of the algorithm, it impacts its performance in terms of the number of sample pulls required for identification.

At each time step $t$, the algorithm sorts all arms based on their empirical means to form a tentative set of top-$m$ arms, denoted $J_t$. From this ranking, we can identify two critical arms that define the decision boundary:
\begin{enumerate}
    \item The Contender $h_t=\operatorname*{arg\,min}_{a\in J_t} L_a(t)$: this is the arm in $J_t$ that we are least likely to keep among the top-$m$.
    \item The Challenger $l_t=\operatorname*{arg\,max}_{a\notin J_t} U_a(t)$: this is the arm in $J_t^c$ that we are most likely to include among the top-$m$.
\end{enumerate}
We then draw samples only from the Contender and the Challenger to update their values. By doing so, we ignore the arms that are clearly safe winners or safe losers. The algorithm stops when:
$$
L_{h_t}(t) > U_{l_t}(t) - \epsilon
$$
This occurs when the confidence sets of these two critical arms no longer overlap up to a tolerance $\epsilon$ (which can be chosen equal to 0). The full algorithm is described in \cref{alg:lucb}. When the algorithm stops, we are $(1-n\alpha)$-confident that every single arm in the returned set is better than, or at most $\epsilon$-worse than, any arm rejected.

\begin{algorithm}[H]
   \caption{LUCB: Lower Upper Confidence Bound}
   \label{alg:lucb}
\begin{algorithmic}[1]
   \STATE {\bfseries Input:} Set of arms $\mathcal{K} = \{1, \dots, n\}$, target number of arms $m$, confidence parameter $\alpha \in (0, 1)$, tolerance $\epsilon \ge 0$.
   \STATE {\bfseries Initialization:} Sample each arm $a \in \mathcal{K}$ once. Set time $t \leftarrow n$.
   \STATE Calculate empirical means $\hat{\mu}_a(t)$ and confidence bounds $U_a(t), L_a(t)$.
   \WHILE{$U_{l_t}(t) - L_{h_t}(t) \ge \epsilon$}
      \STATE Draw one sample from arm $h_t$ and one sample from arm $l_t$.
      \STATE Update empirical means $\hat{\mu}$ and confidence bounds for $h_t$ and $l_t$.
      \STATE Sort arms based on empirical means: $\hat{\mu}_{(1)}(t) \ge \hat{\mu}_{(2)}(t) \ge \dots \ge \hat{\mu}_{(n)}(t)$.
      \STATE Define the set of current Top-$m$ candidates: $J_t$.
      \STATE Identify the \textit{contender} $h_t = \operatorname*{arg\,min}_{a \in J_t} L_a(t)$.
      \STATE Identify the \textit{challenger} $l_t = \operatorname*{arg\,max}_{a \notin J_t} U_a(t)$.
      \STATE $t \leftarrow t+2$.
   \ENDWHILE
   \STATE \textbf{return} $J_t$
\end{algorithmic}
\end{algorithm}

\subsubsection{Application of LUCB to LLM evaluation}
\label{app:LLM_eval}

For our experiments we apply the LUCB algorithm using several underlying CS methods.
While LUCB is designed to pull arms sequentially, we ensure each LUCB variant receives the same LLM outputs by first performing a large number of LLM calls and then simulating the online behavior for each LUCB variant. Note that in practice, when the goal is strictly best-arm identification rather than method comparison, the LLM calls would be performed iteratively. The four LLMs we compare are:

\begin{itemize}
\item \textbf{Gemma 3 1B:} \texttt{google/gemma-3-1b-it-qat} (\url{https://huggingface.co/google/gemma-3-1b-it-qat-q4_0-gguf} ~~\textit{license gemma})
\item \textbf{Qwen 3 1.7B:} \texttt{qwen3-1.7b}: (\url{https://huggingface.co/Qwen/Qwen3-1.7B-GGUF}~~\textit{license apache 2.0})
\item \textbf{Ministral 3 3B:} \texttt{ministral-3-3b-instruct-2512}: (\url{https://huggingface.co/mistralai/Ministral-3-3B-Instruct-2512-GGUF} ~~\textit{license apache 2.0})
\item \textbf{Llama 3.2 3B:} \texttt{llama-3.2-3b-instruct}: (\url{https://huggingface.co/lmstudio-community/Llama-3.2-3B-Instruct-GGUF} ~~\textit{license llama 3.2})
\end{itemize}

When available, reasoning abilities were disabled. Each time a LLM is queried, we generate a random array of integers between 0 and 100 without replacement, and the LLM is tasked with sorting the array using the following prompt: 

\begin{quote}
\ttfamily
You will receive 20 indexed items in the form:

i0: value
i1: value
...

Task:
Assign one numeric sorting score to each item.

Rules:
- Return exactly one numeric field for each item id
- Smaller score means smaller value
- Larger score means larger value
- Use all item ids exactly once as field names
- Do not return any extra fields

Example:
Input:
i0: 42
i1: 7
i2: 15

A correct response would be something like:
\{\{
  "i0": 3,
  "i1": 1,
  "i2": 2
\}\}

because 7 < 15 < 42.
\end{quote}

The LLMs were called using LM Studio version 0.4.25+1 locally on an Apple Silicon M4 Max Chip with 36 GB of RAM.
The temperature of the LLMs was set to 0.0.
We utilise a structured output format to ensure the LLMs return only the array ranking.
The performance of the LLM ordering is evaluated against the ground truth using the Spearman rank correlation coefficient, defined as
\begin{equation}
    \mathcal{S} = 1 - \frac{6 \sum_{i=1}^n d_i^2}{n(n^2 - 1)},
\end{equation}
where $d_i$ is the difference between the predicted and the target rank of the $i$-nth number in the array, which is returned as reward after an arm pull and rescaled to be in $[0,1]$ before being fed to the LUCB algorithm.

In all our experiments, LUCB is used with $\alpha=0.025$ and $\epsilon=0.1$.

Lastly, we emphasise that the goal of this experiment is to compare confidence-sequence methods within a best-arm identification procedure, rather than to benchmark the LLMs themselves. 
Accordingly, our primary outcome is the number of model queries required for LUCB to identify the best arm. 
In all replicates, and for all confidence-sequence methods and score metrics considered, Ministral 3 3B is selected as the best arm; the differences between methods therefore reflect sampling efficiency rather than disagreement about the final selected model.

\subsubsection{Prior specification}

For the Bayes-assisted CS methods, we elicit a prior for every LLM tested.
In particular, given an LLM, we evaluate its performance on the same sorting task, using the same prompt, for smaller arrays ($n \in \{16, 18, 19\}$), for each smaller array we perform 20 sorting tasks and fit an exponential decay function to extrapolate the predicted performance at $n=20$, denoted by $\hat{\mu}$. We also build an estimator $\hat{\nu}=\hat{\mu}(1-\hat{\mu})/\hat{\sigma}$ where $\hat{\sigma}$ is the empirical variance of the 20 data points at size 19. With this, we then set the following priors for the Bayes-assisted CS methods:
\begin{itemize}
   \item Parametric/MDP: $\rho \sim \mathrm{Beta}(500\hat{\mu}, 500(1-\hat{\mu})),\quad \nu \sim \mathrm{Gamma}(2, 2/\hat{\nu})$.
   \item BETEL/RETEL: $\mu \sim \mathrm{Beta}(500\hat{\mu}, 500(1-\hat{\mu}))$.
\end{itemize}

The prior on $\nu$ is chosen so its mean is $\hat{\nu}$ while maintaining the same relative uncertainty as the $\mathrm{Gamma}(2, 2)$ used in the synthetic experiment.

\subsection{Prediction-powered inference}
\label{app:ppi_details}

\subsubsection{Bounded-mean PPI construction}
\label{app:ppi_construction}

Prediction-powered inference (PPI) uses predictions from a fixed model to improve inference from a limited number of labelled observations and a larger pool of unlabelled observations \citep{angelopoulos2023prediction}. 
We use PPI only in the bounded-mean setting considered throughout the paper. 
Let $(X_i,Y_i)$ denote a labelled example, let $f$ be a fixed prediction rule, and let the target be
\[
    \theta^\star = \mathbb E[Y].
\]
PPI decomposes this target as
\[
    \theta^\star
    =
    \mathbb E[f(X)] + \mathbb E[Y-f(X)]
    =
    m^\star + \Delta^\star ,
\]
where $m^\star$ is the prediction-only term and $\Delta^\star$ is the rectifier correcting the bias of the prediction rule. 
Recent work has studied confidence sequences for this decomposition, either by constructing prediction-powered confidence sequences with asymptotic guarantees \citep{Kilian2025} or by building prediction-powered e-values in more general sequential settings \citep{Csillag2025}. 
Our use of PPI is narrower but fits directly into the nonasymptotic bounded-mean framework of the present paper: after fixing the unlabelled plug-in estimate of $m^\star$, the remaining inferential problem is an ordinary bounded-mean problem for the rectifier.

Specifically, given a pool of $N$ unlabelled examples, we estimate the prediction-only term by
\[
    \widehat m_N = \frac{1}{N}\sum_{j=1}^N f(X_j).
\]
Following the simplified setting of \citet{Kilian2025}, we treat $\widehat m_N$ as fixed and focus the sequential uncertainty on the labelled residuals
\[
    R_i = Y_i - f(X_i),
    \qquad
    \Delta^\star = \mathbb E[R_i].
\]
If $R_i\in[\ell,u]$, we rescale the residuals to the unit interval by
\[
    Z_i = \frac{R_i-\ell}{u-\ell}\in[0,1],
    \qquad
    \mu_{Z}^\star = \mathbb E[Z_i]
    =
    \frac{\Delta^\star-\ell}{u-\ell}.
\]
We then apply the bounded-mean confidence-sequence methods of the paper to the sequence $(Z_i)_{i\geq 1}$. 
If $C^Z_{\alpha,n}$ denotes the resulting confidence sequence for $\mu_{Z}^\star$, then the corresponding confidence sequence for the rectifier is
\[
    C^\Delta_{\alpha,n}
    =
    \ell + (u-\ell) C^Z_{\alpha,n},
\]
and the reported PPI confidence sequence for the original target is
\[
    C^\theta_{\alpha,n}
    =
    \widehat m_N + C^\Delta_{\alpha,n}.
\]
Thus all Bayes-assisted and prior-free bounded-mean methods are applied only to the rescaled residual sequence $Z_i$; the unlabelled predictions enter through the fixed shift $\widehat m_N$.

\subsubsection{Datasets}
\label{app:ppi_datasets}

\paragraph{Galaxies.}

The \textsc{galaxies} dataset contains a total of $16743$ images from the Galaxy Zoo 2 initiative \citep{willett2013galaxy}.
For each image, the dataset contains a binary label indicating whether the galaxy has spiral arms ($Y_i \in \{0,1\}$) and the predicted probability of a ResNet50 model \citep{he2016deep} for the same quantity ($f(X_i) \in [0, 1]$). 
The estimand is the population fraction of galaxies with spiral arms, i.e., $\theta^\star = \mathbb E[Y_i] \in [0, 1]$.
The PPI rectifier is based on residuals $R_i=Y_i-f(X_i)$, which lie in $[-1,1]$.
We therefore apply our bounded-mean methods to the rescaled residuals $Z_i=(R_i+1)/2\in[0,1]$.

\paragraph{AttributedQA.}

The \textsc{attributedQA} dataset evaluates attribution in question answering systems \citep{bohnet2022attributed}.
For each system output, the dataset reports an automatic attribution score from AutoAIS, an NLI-based evaluator, and for a subset of examples also reports human attribution labels.
We compare the retrieve-then-read systems RTR-10 and RTR-4, for which there are $1000$ examples with both human labels and AutoAIS scores, and $2610$ additional examples with AutoAIS scores only.
For each human-labelled example and system $k\in\{10,4\}$, let $Y_i^{(k)}\in\{0,1\}$ indicate whether the answer is judged by humans to be correctly attributed, and let $f^{(k)}(X_i)\in[0,1]$ denote the corresponding AutoAIS score.
We form paired differences
\[
    Y_i = Y_i^{(10)} - Y_i^{(4)} \in \{-1,0,1\},
    \qquad
    f(X_i)=f^{(10)}(X_i)-f^{(4)}(X_i)\in[-1,1].
\]
The estimand is the difference in human attribution reliability between RTR-10 and RTR-4, i.e.,
$\theta^\star=\mathbb E[Y_i]\in[-1,1]$.
The PPI rectifier is based on residuals $R_i=Y_i-f(X_i)$, which lie in $[-2,2]$.
We therefore apply our bounded-mean methods to the rescaled residuals $Z_i=(R_i+2)/4\in[0,1]$.

\subsubsection{Prior specification}
\label{app:ppi_prior_specification}

We choose the PPI priors to match the heuristic of \citet{Kilian2025}, who place a zero-centred prior on the scaled mean rectifier with variance $1/n_0$, where $n_0$ is a reference labelled sample size. 
Let $R_i=Y_i-f(X_i)\in[\ell,u]$ denote the rectifier residual, and let
\[
    Z_i=\frac{R_i-\ell}{u-\ell}\in[0,1],
    \qquad
    \mu_{Z}^\star=\mathbb E[Z_i].
\]
In both PPI datasets the residual bounds are symmetric around zero, so the prior belief $\mathbb E[R_i]\approx 0$ maps to $\mu_{Z}^\star\approx 1/2$. 
We therefore use a symmetric beta prior,
\[
    \mu_{Z}^\star \sim \mathrm{Beta}(\xi, \xi),
\]
with $\xi$ chosen by matching the variance on the original rectifier scale. 
Since
\[
    \mathrm{Var}\left(\ell+(u-\ell)\mu_{Z}^\star\right)
    =
    (u-\ell)^2\frac{1}{4(2\xi+1)},
\]
matching this variance to $1/n_0$ gives
\[
    \xi(n_0)
    =
    \frac12\left(
        \frac{n_0(u-\ell)^2}{4}-1
    \right),
\]
where $n_0$ is taken to be the largest labelled sample size considered in the PPI experiments, i.e., $n_0=1000$ for \textsc{galaxies} and $n_0=500$ for \textsc{attributedQA}.

This prior enters the different Bayes-assisted predictives following \cref{app:default_numerical_parameters}. 
For the parametric Bayesian and MDP methods, we use
\[
    \rho\sim\mathrm{Beta}(\xi(n_0),\xi(n_0)),
    \qquad
    \nu\sim\mathrm{Gamma}(7.5,1),
\]
where $\rho$ and $\nu$ are the mean and concentration parameters of the beta working likelihood, respectively. 
For BETEL and RETEL, we place the corresponding prior directly on the rescaled mean,
\[
    \mu\sim\mathrm{Beta}(\xi(n_0),\xi(n_0)).
\]
Thus all Bayes-assisted methods encode the same prior belief that the rectifier is close to zero, while differing in whether this belief enters through a beta working likelihood or directly through the empirical-likelihood pseudo-posterior.

\section{Additional experimental results}
\label{app:additional_results}

This section presents additional experimental results to complement the ones in \cref{sec:experiments}.

\subsection{Synthetic data}
\label{app:synthetic}

\subsubsection{Full baseline comparison}
\label{app:additional_baselines}

\Cref{fig:simulation_beta_mixture_baselines} shows the full-baseline version of the beta-mixture experiments in \cref{fig:simulation_beta_mixture_main}, while \cref{fig:simulation_bernoulli_baselines,fig:simulation_beta_baselines} show the full results on the Bernoulli and beta experiments, respectively.
\begin{figure}
    \centering
    \includegraphics[width=\textwidth]{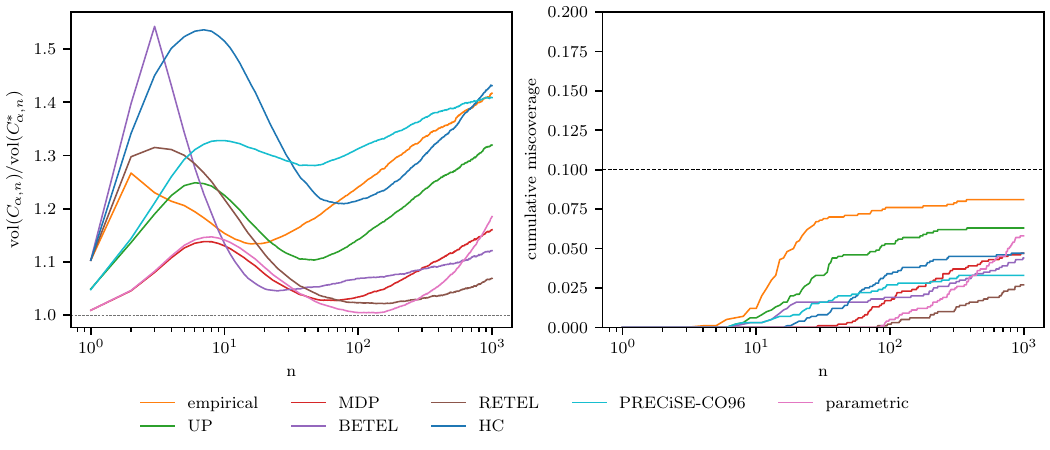}
    \caption{
        Full-baseline version of the synthetic beta-mixture experiment in \cref{fig:simulation_beta_mixture_main}. 
        The left and right panels show average CS length relative to the oracle and empirical cumulative miscoverage, respectively, over $1000$ replicates;
        the dashed line marks $\alpha = 0.1$. Width ratios use replicates where both the method and oracle CSs are non-empty. At $n=1000$, at most 2\% are empty (0.5\% on average).
    }
    \label{fig:simulation_beta_mixture_baselines}
\end{figure}
\begin{figure}
    \centering
    \includegraphics[width=\textwidth]{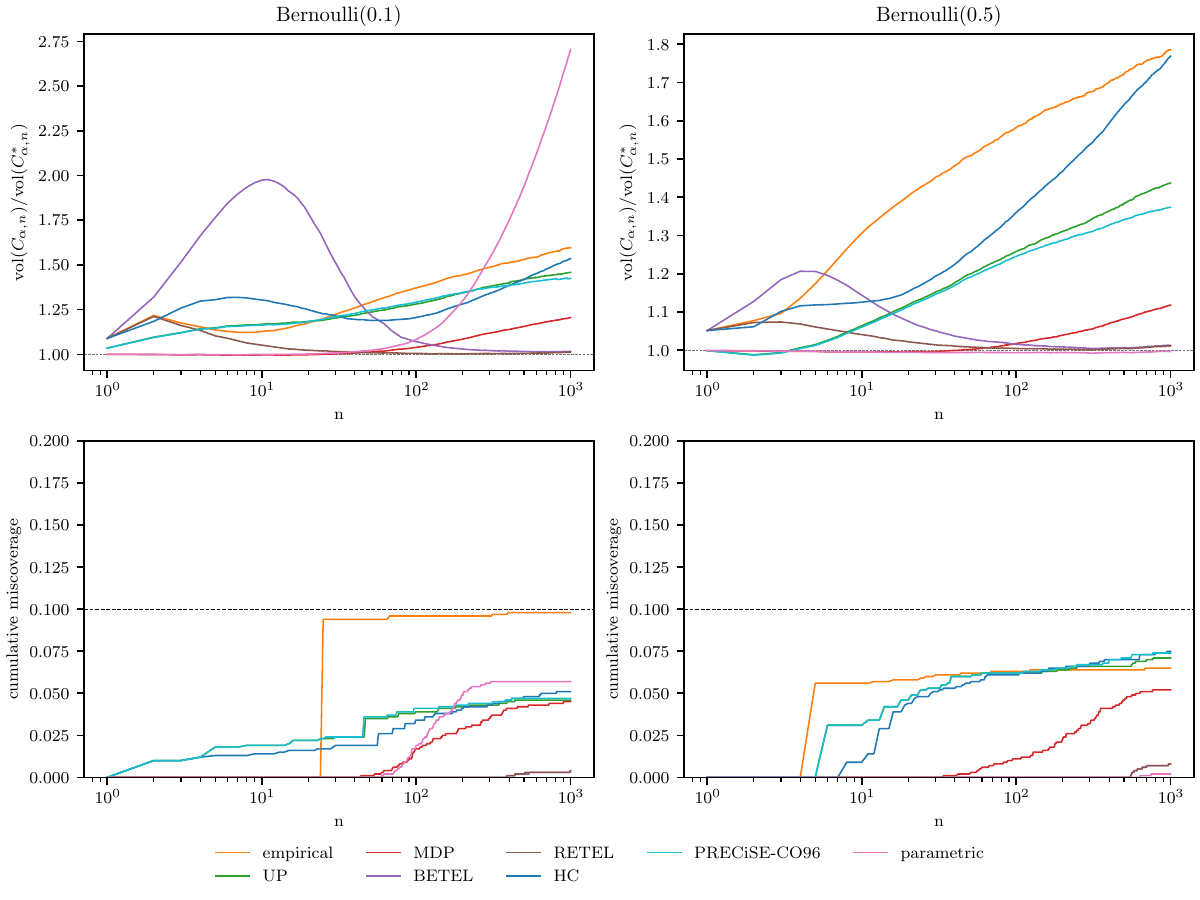}
    \caption{
        Synthetic Bernoulli experiments with all baselines. 
        The left and right columns correspond to $X_n\sim\mathrm{Bernoulli}(0.1)$ and $X_n\sim\mathrm{Bernoulli}(0.5)$, respectively. 
        The top and bottom rows show average CS length relative to the oracle and empirical cumulative miscoverage, respectively, over $1000$ replicates;
        the dashed line marks $\alpha = 0.1$. Width ratios use replicates where both the method and oracle CSs are non-empty. At $n=1000$, at most 4.4\% are empty (0.7\% on average).
    }
    \label{fig:simulation_bernoulli_baselines}
\end{figure}
\begin{figure}
    \centering
    \includegraphics[width=\textwidth]{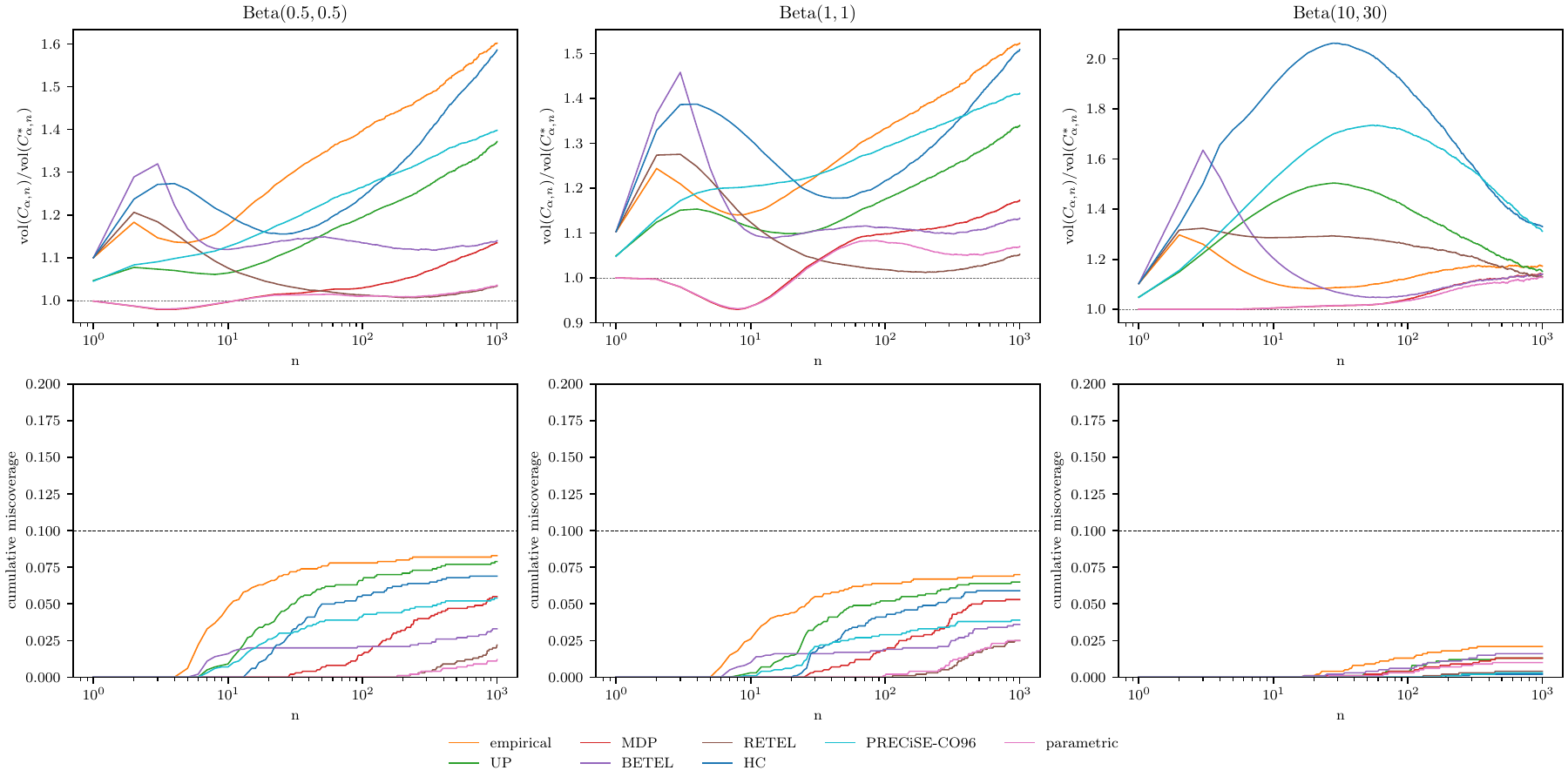}
    \caption{
        Synthetic beta experiments with all baselines. 
        The left, centre, and right columns correspond to $X_n\sim\mathrm{Beta}(0.5, 0.5)$, $X_n\sim\mathrm{Beta}(1, 1)$, and $X_n\sim\mathrm{Beta}(10, 30)$, respectively. 
        The top and bottom rows show average CS length relative to the oracle and empirical cumulative miscoverage, respectively, over $1000$ replicates;
        the dashed line marks $\alpha = 0.1$. Width ratios use replicates where both the method and oracle CSs are non-empty. At $n=1000$, at most 3.5\% are empty (0.4\% on average).
    }
    \label{fig:simulation_beta_baselines}
\end{figure}
Overall, the results are consistent with the main text: accurate predictive distributions can substantially reduce confidence-sequence width without sacrificing coverage.
It is interesting to see that the parametric Bayesian predictive severely underperforms the other Bayes-assisted methods in the Bernoulli experiment with $p = 0.1$, where the true distribution is far from the beta working model.

\subsubsection{Prior sensitivity}
\label{app:additional_priors}

\Cref{fig:simulation_beta_bayes,fig:simulation_bernoulli_bayes,fig:simulation_beta_bayes} compare the three prior regimes for the Bayes-assisted methods on the beta-mixture, Bernoulli, and beta simulations, respectively;
detailed prior specifications are provided in \cref{app:prior_specifications_synthetic}.
\begin{figure}
    \centering
    \includegraphics[width=\textwidth]{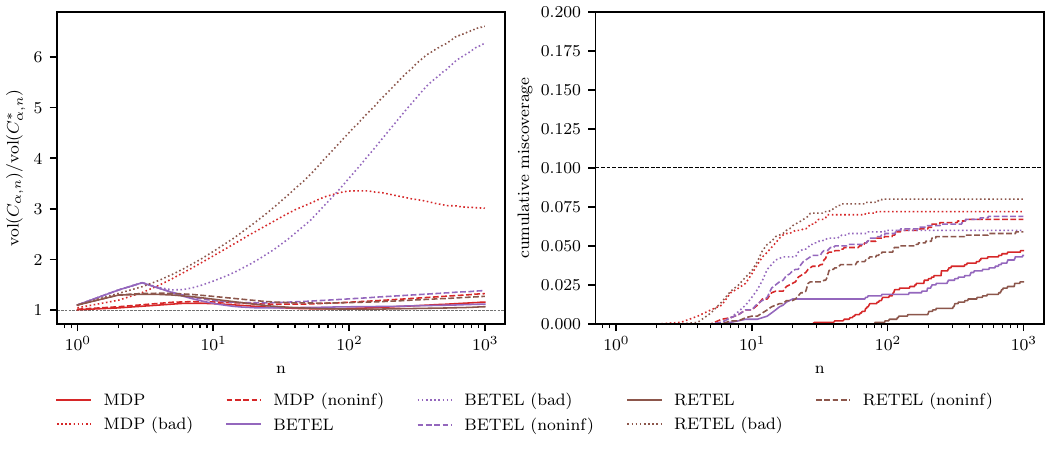}
    \caption{
        Synthetic beta-mixture experiment with different prior specifications for the same setting as \cref{fig:simulation_beta_mixture_main,fig:simulation_beta_mixture_baselines}. 
        The left and right panels show average CS length relative to the oracle and empirical cumulative miscoverage, respectively, over $1000$ replicates.
        Solid, dotted, and faint solid curves correspond to informative, misspecified, and non-informative priors, respectively;
        the dashed horizontal line marks $\alpha = 0.1$. Width ratios use replicates where both the method and oracle CSs are non-empty. At $n=1000$, at most 1.4\% are empty (0.6\% on average).
    }
    \label{fig:simulation_beta_mixture_bayes}
\end{figure}
\begin{figure}
    \centering
    \includegraphics[width=\textwidth]{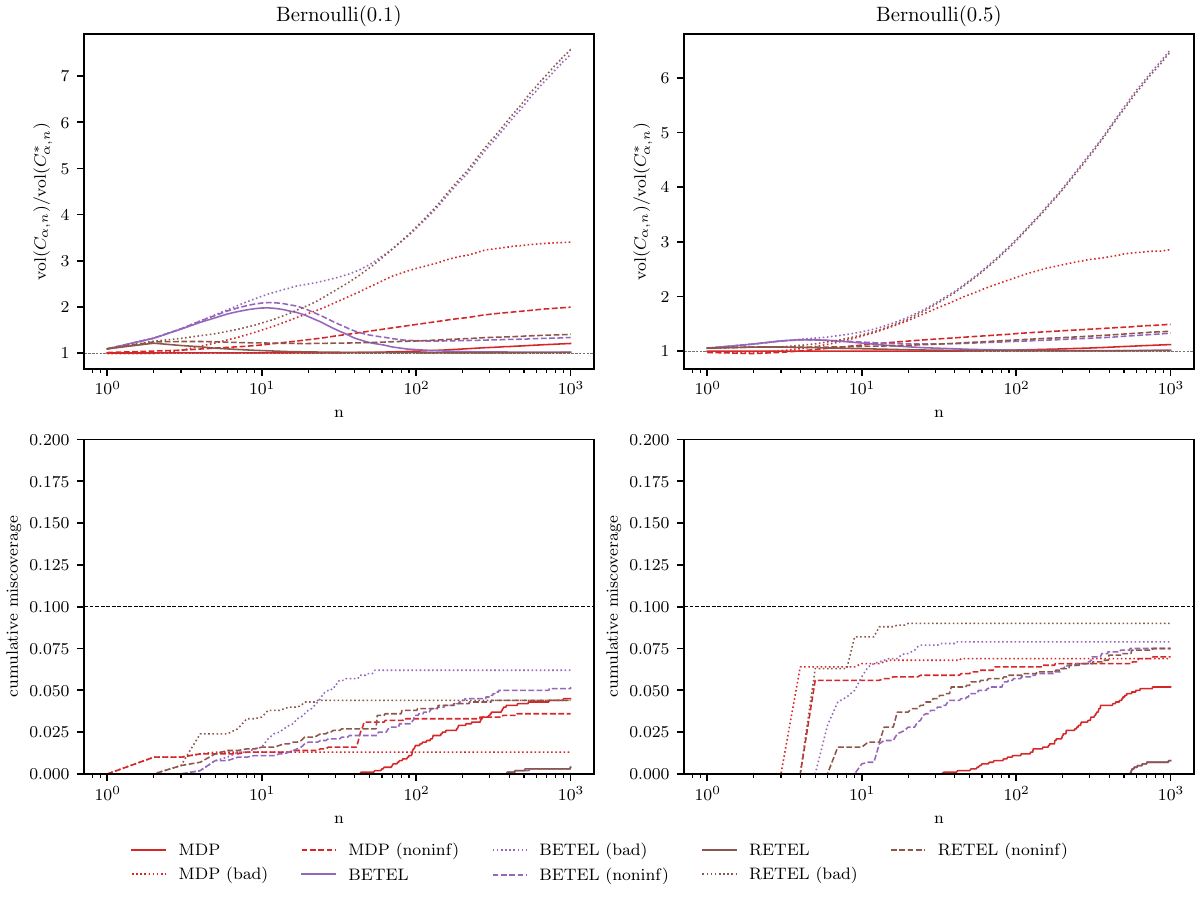}
    \caption{
        Synthetic Bernoulli experiment with different prior specifications for the same setting as in \cref{fig:simulation_bernoulli_baselines}. 
        The top and bottom rows show average CS length relative to the oracle and empirical cumulative miscoverage, respectively, over $1000$ replicates.
        Solid, dotted, and faint solid curves correspond to informative, misspecified, and non-informative priors, respectively;
        the dashed horizontal line marks $\alpha = 0.1$. Width ratios use replicates where both the method and oracle CSs are non-empty. At $n=1000$, at most 3.2\% are empty (0.8\% on average).
    }
    \label{fig:simulation_bernoulli_bayes}
\end{figure}
\begin{figure}
    \centering
    \includegraphics[width=\textwidth]{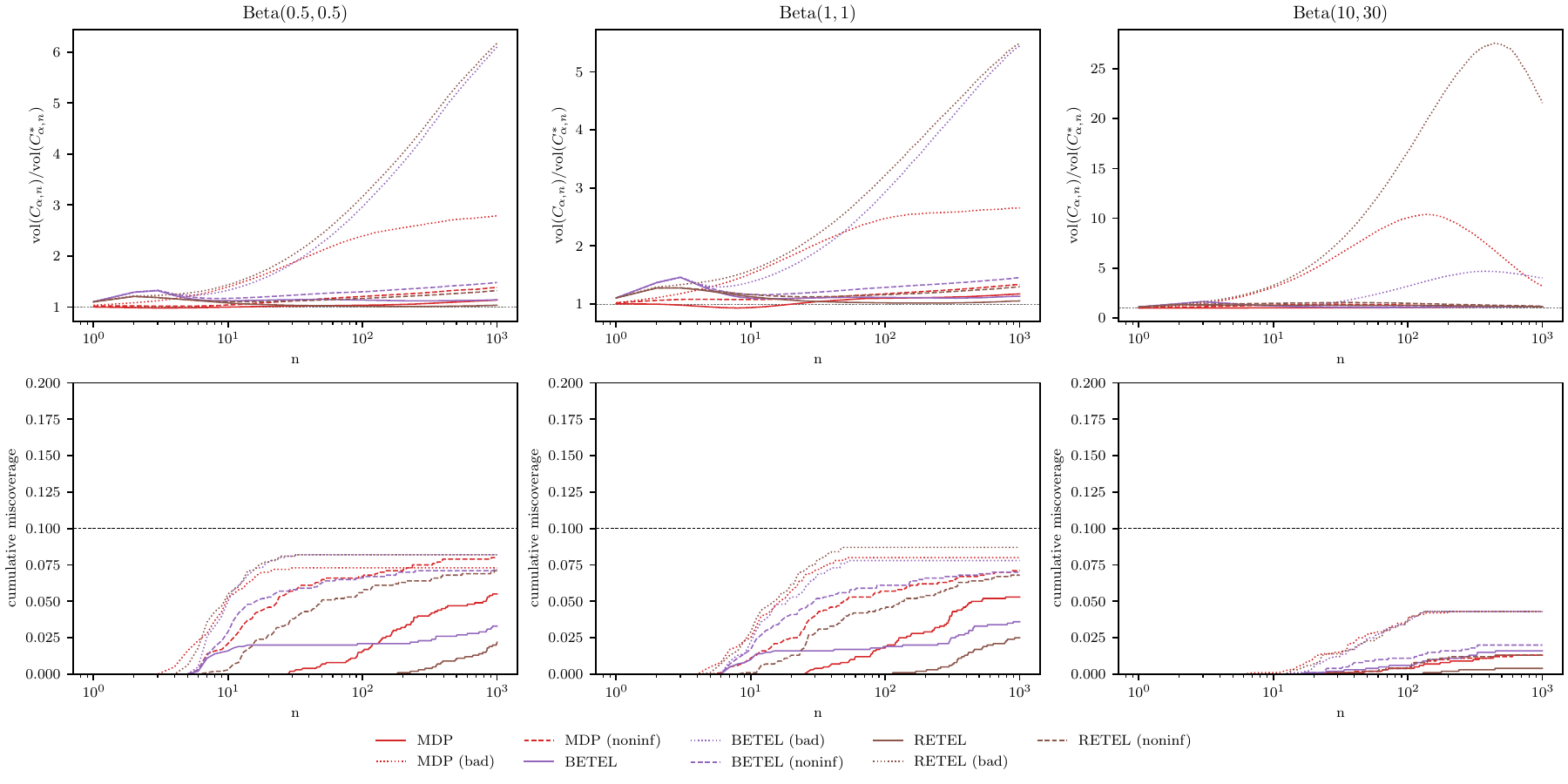}
    \caption{
        Synthetic beta experiment with different prior specifications for the same setting as in \cref{fig:simulation_beta_baselines}. 
        The top and bottom rows show average CS length relative to the oracle and empirical cumulative miscoverage, respectively, over $1000$ replicates.
        Solid, dotted, and faint solid curves correspond to informative, misspecified, and non-informative priors, respectively;
        the dashed horizontal line marks $\alpha = 0.1$. Width ratios use replicates where both the method and oracle CSs are non-empty. At $n=1000$, at most 2.4\% are empty (0.5\% on average).
    }
    \label{fig:simulation_beta_bayes}
\end{figure}
The overall message is clear: the quality of the prior information incorporated into the predictive distribution can have a significant impact on efficiency, with misspecified priors leading to wider confidence sequences, but coverage is preserved across all prior regimes (up to random fluctuations in the $\mathrm{Bernoulli}(0.5)$ example, where one of the misspecified Bayes-assisted methods slightly undercovers).

\subsubsection{Prior-informed universal portfolio}

We consider a prior-informed variant of the original Cover-Ordentlich Universal Portfolio (UP) \cite{Cover1996} construction, which we call Informed Universal Portfolio (I-UP). For each candidate null mean $\mu$, I-UP mixes the constant betting strategies
$$  
\lambda_{\mu, u}=\frac{u-\mu}{\mu(1-\mu)}, \quad u \in[0,1],
$$
using the prior $\pi(d u)$, giving the wealth process
$$
W_n^\pi(\mu)=\int \prod_{i=1}^n\left\{1+\lambda_{\mu, u}\left(X_i-\mu\right)\right\} \pi(d u),
$$
which is a valid test martingale for any fixed choice of $\pi$. The form of $\lambda_{\mu, u}$ arises from
$$
\lambda_{\mu, u} \in \operatorname{argmax}_{\lambda \in I_{\mu, 1}} \inf _{\mathbb{P}: \mathbb{E}_{\mathbb{P}} X=u} \mathbb{E}_{\mathbb{P}} \log \{1+\lambda(X-\mu)\}=\operatorname{argmax}_{\lambda \in I_{\mu, 1}} \mathbb{E}_{\mathbb{P}=\operatorname{Bernoulli}(u)} \log \{1+\lambda(X-\mu)\},
$$
and therefore represents a robust Kelly bet against the composite null $H_\mu: \mu^{\star}=\mu$ against the composite alternative $H_u: \mu^{\star}=u$, where robustness refers to maximising the worst-case expected log growth over all distributions on $[0,1]$ with mean $u$. It can be thought of as a worst-case analogue of the predictive objective in \cref{eq:bayes_assisted_lambda}. Standard UP is recovered by choosing $\pi(d u)=\operatorname{Beta}(1 / 2,1 / 2)$. By setting $\pi(d u)$ equal to the same Beta prior on the mean used for BETEL/RETEL, one can incorporate the same prior information through a structurally different mechanism: rather than using the prior-induced predictive to select the coefficient that maximises predictive expected log-growth, I-UP uses the prior as a fixed mixing distribution over constant, alternative-mean-indexed betting strategies.

To evaluate the performance of I-UP, we run additional experiments for all our synthetic scenarios. \cref{tab:up_prior_width} compares average widths over 100 replicates between the different methods; the Bayes-assisted methods are compared under the same priors.

The Bernoulli setting provides a positive control: because the least-favourable bounded distribution with mean $u$ is $\operatorname{Bernoulli}(u)$, the robust bet indexed by $u$ coincides with the oracle's. Accordingly, informative I-UP essentially matches the oracle and is roughly 15--20\% narrower than standard UP at $n=50$ and 25\% narrower at $n=200$. Informative MDP is similarly close to the oracle. Thus, I-UP can exploit prior information nearly optimally when its robust mean-indexed bets align with the data-generating distribution.

For the non-Bernoulli distributions, informative MDP remains close to the oracle and is narrower than standard UP. In contrast, informative I-UP is generally wider than standard UP; the exception is $\operatorname{Beta}(0.5,0.5)$, where the two are similar and I-UP is slightly narrower at $n=200$. While the prior is correctly centred on the true mean, I-UP concentrates on bets that are worst-case optimal over all bounded distributions with a given mean, whereas the coefficient optimal for the actual distribution can correspond to an index far from that mean. MDP can exploit this additional distributional information through its adaptive predictive formulation.

Finally, non-informative variants remain broadly comparable to their prior-free counterparts, while strongly misspecified priors substantially reduce the finite-sample efficiency of both approaches without affecting validity. Under misspecification, the efficiency loss is consistently smaller for MDP, particularly at $n=200$, because its predictive increasingly relies on the empirical distribution.

Overall, the comparison shows that the gains of our framework do not arise merely from access to prior information, but from translating that information into candidate-specific bets that adapt well to the data-generating process at hand.

\begin{sidewaystable}
\centering
\caption{Average CS width (standard error) over 100 replicates, at $n=50$ and $n=200$. Empty CSs are counted as width zero; across all CSs shown, at most 3\% are empty at $n=50$ or $n=200$ (0.1\% on average).}
\label{tab:up_prior_width}
\resizebox{\linewidth}{!}{%
\begin{tabular}{llcccccccccccc}
\toprule
 &  & \multicolumn{2}{c}{Bernoulli(0.1)} & \multicolumn{2}{c}{Bernoulli(0.5)} & \multicolumn{2}{c}{Beta(0.5, 0.5)} & \multicolumn{2}{c}{Beta(1, 1)} & \multicolumn{2}{c}{Beta(10, 30)} & \multicolumn{2}{c}{Beta mixture} \\
 &  & $n=50$ & $n=200$ & $n=50$ & $n=200$ & $n=50$ & $n=200$ & $n=50$ & $n=200$ & $n=50$ & $n=200$ & $n=50$ & $n=200$ \\
Method & Prior &  &  &  &  &  &  &  &  &  &  &  &  \\
\midrule
Oracle & - & 0.170 (0.0030) & 0.086 (0.0011) & 0.284 (0.0035) & 0.145 (0.0019) & 0.205 (0.0021) & 0.104 (0.0017) & 0.168 (0.0017) & 0.083 (0.0011) & 0.059 (0.0003) & 0.022 (0.0002) & 0.137 (0.0016) & 0.069 (0.0008) \\
UP & $\operatorname{Beta}(1 / 2,1 / 2)$ & 0.212 (0.0044) & 0.117 (0.0022) & 0.340 (0.0051) & 0.189 (0.0039) & 0.237 (0.0034) & 0.130 (0.0025) & 0.194 (0.0027) & 0.102 (0.0018) & 0.087 (0.0003) & 0.030 (0.0002) & 0.151 (0.0022) & 0.082 (0.0015) \\
I-UP & Informative & 0.170 (0.0030) & 0.086 (0.0011) & 0.284 (0.0033) & 0.145 (0.0016) & 0.241 (0.0011) & 0.120 (0.0007) & 0.228 (0.0007) & 0.112 (0.0004) & 0.184 (0.0001) & 0.088 (0.0001) & 0.216 (0.0007) & 0.106 (0.0003) \\
 & Non-informative & 0.219 (0.0036) & 0.119 (0.0021) & 0.325 (0.0042) & 0.179 (0.0034) & 0.228 (0.0029) & 0.124 (0.0022) & 0.189 (0.0023) & 0.098 (0.0016) & 0.100 (0.0002) & 0.033 (0.0002) & 0.154 (0.0019) & 0.080 (0.0012) \\
 & Misspecified & 0.490 (0.0035) & 0.424 (0.0035) & 0.646 (0.0147) & 0.575 (0.0128) & 0.516 (0.0074) & 0.440 (0.0055) & 0.478 (0.0048) & 0.404 (0.0037) & 0.648 (0.0009) & 0.562 (0.0005) & 0.564 (0.0039) & 0.484 (0.0029) \\
MDP & Informative & 0.173 (0.0032) & 0.094 (0.0014) & 0.286 (0.0028) & 0.151 (0.0016) & 0.210 (0.0020) & 0.110 (0.0014) & 0.184 (0.0026) & 0.094 (0.0013) & 0.060 (0.0003) & 0.024 (0.0002) & 0.141 (0.0015) & 0.074 (0.0009) \\
 & Non-informative & 0.252 (0.0065) & 0.152 (0.0033) & 0.352 (0.0071) & 0.193 (0.0045) & 0.239 (0.0038) & 0.131 (0.0027) & 0.195 (0.0027) & 0.103 (0.0019) & 0.080 (0.0004) & 0.028 (0.0003) & 0.153 (0.0026) & 0.082 (0.0017) \\
 & Misspecified & 0.418 (0.0043) & 0.270 (0.0039) & 0.575 (0.0147) & 0.367 (0.0099) & 0.436 (0.0074) & 0.269 (0.0050) & 0.381 (0.0048) & 0.222 (0.0035) & 0.478 (0.0008) & 0.223 (0.0005) & 0.422 (0.0037) & 0.225 (0.0030) \\
BETEL & Informative & 0.208 (0.0096) & 0.089 (0.0012) & 0.291 (0.0032) & 0.146 (0.0016) & 0.236 (0.0051) & 0.115 (0.0029) & 0.192 (0.0028) & 0.093 (0.0017) & 0.062 (0.0004) & 0.024 (0.0002) & 0.145 (0.0027) & 0.074 (0.0015) \\
 & Non-informative & 0.236 (0.0092) & 0.109 (0.0020) & 0.324 (0.0042) & 0.176 (0.0029) & 0.260 (0.0054) & 0.140 (0.0034) & 0.215 (0.0035) & 0.114 (0.0022) & 0.064 (0.0004) & 0.026 (0.0003) & 0.162 (0.0030) & 0.086 (0.0021) \\
 & Misspecified & 0.503 (0.0060) & 0.420 (0.0031) & 0.656 (0.0135) & 0.583 (0.0116) & 0.456 (0.0121) & 0.392 (0.0099) & 0.393 (0.0087) & 0.321 (0.0065) & 0.127 (0.0031) & 0.095 (0.0018) & 0.379 (0.0097) & 0.320 (0.0074) \\
RETEL & Informative & 0.171 (0.0029) & 0.086 (0.0011) & 0.286 (0.0032) & 0.145 (0.0017) & 0.210 (0.0020) & 0.105 (0.0014) & 0.172 (0.0015) & 0.084 (0.0009) & 0.075 (0.0003) & 0.028 (0.0002) & 0.142 (0.0014) & 0.071 (0.0007) \\
 & Non-informative & 0.210 (0.0041) & 0.113 (0.0022) & 0.329 (0.0045) & 0.180 (0.0036) & 0.233 (0.0032) & 0.126 (0.0023) & 0.193 (0.0025) & 0.100 (0.0017) & 0.090 (0.0003) & 0.031 (0.0003) & 0.155 (0.0020) & 0.081 (0.0013) \\
 & Misspecified & 0.490 (0.0035) & 0.424 (0.0035) & 0.654 (0.0136) & 0.583 (0.0118) & 0.502 (0.0081) & 0.419 (0.0068) & 0.437 (0.0052) & 0.340 (0.0044) & 0.633 (0.0010) & 0.525 (0.0006) & 0.497 (0.0044) & 0.368 (0.0030) \\
Parametric & Informative & 0.174 (0.0026) & 0.112 (0.0015) & 0.283 (0.0036) & 0.144 (0.0018) & 0.209 (0.0023) & 0.105 (0.0014) & 0.182 (0.0025) & 0.089 (0.0012) & 0.060 (0.0003) & 0.024 (0.0002) & 0.140 (0.0014) & 0.070 (0.0007) \\
Empirical & - & 0.218 (0.0084) & 0.124 (0.0039) & 0.418 (0.0091) & 0.229 (0.0058) & 0.279 (0.0060) & 0.153 (0.0039) & 0.218 (0.0042) & 0.118 (0.0026) & 0.065 (0.0005) & 0.026 (0.0003) & 0.161 (0.0034) & 0.087 (0.0023) \\
PRECiSE-CO96 & - & 0.211 (0.0047) & 0.115 (0.0021) & 0.337 (0.0051) & 0.187 (0.0039) & 0.251 (0.0032) & 0.136 (0.0023) & 0.213 (0.0024) & 0.111 (0.0015) & 0.101 (0.0002) & 0.036 (0.0002) & 0.174 (0.0022) & 0.093 (0.0012) \\
HC & - & 0.205 (0.0033) & 0.110 (0.0019) & 0.357 (0.0023) & 0.213 (0.0018) & 0.242 (0.0022) & 0.137 (0.0015) & 0.198 (0.0019) & 0.108 (0.0010) & 0.119 (0.0002) & 0.038 (0.0002) & 0.166 (0.0015) & 0.087 (0.0007) \\
\bottomrule
\end{tabular}
}
\end{sidewaystable}

\subsubsection{Truncation parameter $c$}
\label{app:c}

The truncation parameter $c$ in \cref{eq:I_mu_c} is common in betting-based confidence-sequence methods. This section explains its role, the choice $c=0.95$ for the empirical and Bayes-assisted experiments, and the sensitivity of the resulting confidence sequences to this choice.

The truncation parameter $c$ does not affect validity, although it can substantially affect efficiency. For a given $c\in(0,1)$ and candidate mean $\mu$, the set of admissible betting coefficients
$$
I_{\mu, c}=\left[-\frac{c}{1-\mu}, \frac{c}{\mu}\right]
$$
ensures that, for $x \in[0,1]$, each one-step factor in the wealth process satisfies
$$
1+\lambda(x-\mu) \geq 1-c>0
$$
thereby inducing a valid nonnegative test martingale.

The truncation level $c$ controls the amount of leverage available to the betting strategy, i.e., the largest fraction of the current wealth that can be lost in a single time step. This exposes a genuine trade-off. If $c_1<c_2$, then $I_{\mu, c_1} \subset I_{\mu, c_2}$; that is, increasing $c$ enlarges the oracle's feasible set and thus cannot decrease the optimal expected log growth. As a result, a larger $c$ can improve efficiency when the predictive distribution used to select the betting coefficient is accurate, but increases exposure to bad bets if that is not the case.

This also explains the dependence on $c$ through $L_{\mu, c}$ in the upper bound of \cref{thm:finitn}. \cref{thm:finitn} bounds the per-observation expected log-growth gap relative to the oracle in terms of the average predictive error measured in Wasserstein distance. As shown in the proof of \cref{thm:logoptimality}, $L_{\mu, c}$ is a Lipschitz constant in $x$, uniform over $\lambda \in I_{\mu, c}$, for the one-step log factor $\log\{1+\lambda(x-\mu)\}$. The latter's derivative with respect to $x$ becomes large when the linear factor $1+\lambda(x-\mu)$ approaches zero. Since the admissible coefficient range guarantees the lower bound $1+\lambda(x-\mu) \geq 1-c$, this worst-case sensitivity necessarily deteriorates as $c \rightarrow 1$, for fixed $\mu$. However, $L_{\mu, c}$ is obtained by taking the worst case uniformly over all $\lambda \in I_{\mu, c}$ and $x \in[0,1]$, and can therefore be conservative in a particular experiment. Nevertheless, its behaviour as $c \rightarrow 1$ captures the potential vulnerability of highly leveraged bets to predictive misspecification: an inaccurate predictive may select an aggressive yet suboptimal coefficient whose one-step log loss is extremely costly.

To assess this trade-off empirically, we run additional synthetic experiments with $c \in \{0.5,0.75,0.9,0.95,0.99\}$ under three contrasting data-generating processes: $\operatorname{Bernoulli}(0.5)$, which has maximal variance among distributions on $[0,1]$ and whose observations lie at the support boundaries; $\operatorname{Beta}(10,30)$, which is concentrated in the interior; and the asymmetric beta mixture, which departs from the beta working model. \cref{tab:truncation} reports the average widths at $n=100$ and $n=1000$ over 100 replicates.

The results reveal a clear leverage-robustness trade-off. For $\operatorname{Bernoulli}(0.5)$, neither the oracle nor informative Bayes-assisted methods benefit from increasing $c$, indicating that $c=0.5$ already contains the relevant optimal coefficients. Compared with $c=0.5$, at $c=0.95$ empirical is about $23\%$ wider at both horizons, HC is $36$--$56\%$ wider, and misspecified MDP and BETEL are $23$--$28\%$ wider. Because every Bernoulli observation occurs at a support endpoint, an overly aggressive bet is repeatedly exposed to the most adverse possible outcome.

For $\operatorname{Beta}(10,30)$, the pattern reverses. At $n=100$, increasing $c$ from 0.5 to 0.95 narrows substantially the oracle, empirical, and informative Bayes-assisted CSs. Here, the smaller admissible range excludes useful betting coefficients, while the concentration of the observations away from the endpoints makes greater leverage beneficial. The gains begin to saturate by $c=0.95$: at $n=100$, increasing $c$ further to 0.99 changes the oracle and informative Bayes-assisted widths by at most $2.2\%$, while at $n=1000$ the changes are negligible. The misspecified-prior results are more method- and sample-dependent, and therefore do not support a universal monotonic robustness claim; nevertheless, the accurate-prior results clearly show that a larger $c$ can exploit better oracle bets when they are available.

The beta-mixture results are intermediate: the oracle and most informative-prior methods are nearly insensitive to $c$, while empirical, HC, and BETEL widen modestly as $c$ increases.

Overall, we use $c=0.95$ for the empirical and Bayes-assisted experiments as an aggressive, nearly unconstrained experimental choice, not as a universal default. When the predictive is accurate, this choice captures most of the available oracle benefit when larger coefficients are useful, while not forcing their use when a smaller coefficient is optimal. Moderate choices such as $c \in[0.5,0.75]$ provide stronger regularisation when the predictive is uncertain or potentially misspecified. In these experiments, values closer to one provide little additional oracle improvement while reducing the worst-case one-step wealth floor from $0.05$ at $c=0.95$ to $0.01$ at $c=0.99$.

\begin{sidewaystable}
\centering
\caption{Average CS width at $n=100$ / $n=1000$ for each truncation level $c$. Widths are averaged over the non-empty CSs; across all CSs shown, at most 6\% are empty at $n=100$ or $n=1000$ (0.4\% on average).}
\label{tab:truncation}
\begin{subtable}{\linewidth}
\centering
\caption{Bernoulli(0.5)}
\label{tab:truncation_bernoulli_p0.5}
\resizebox{\linewidth}{!}{%
\begin{tabular}{lccccccccccccccc}
\toprule
Method & Oracle & \multicolumn{3}{c}{MDP} & \multicolumn{3}{c}{BETEL} & \multicolumn{3}{c}{RETEL} & \multicolumn{3}{c}{Parametric} & Empirical & HC \\
Prior & - & Informative & Non-informative & Misspecified & Informative & Non-informative & Misspecified & Informative & Non-informative & Misspecified & Informative & Non-informative & Misspecified & - & - \\
$c$ &  &  &  &  &  &  &  &  &  &  &  &  &  &  &  \\
\midrule
0.5 & 0.201 / 0.064 & 0.206 / 0.074 & 0.232 / 0.084 & 0.370 / 0.148 & 0.204 / 0.066 & 0.238 / 0.087 & 0.507 / 0.349 & 0.202 / 0.066 & 0.240 / 0.089 & 0.507 / 0.349 & 0.200 / 0.065 & 0.224 / 0.077 & 0.487 / 0.304 & 0.252 / 0.094 & 0.276 / 0.115 \\
0.75 & 0.201 / 0.064 & 0.206 / 0.074 & 0.250 / 0.092 & 0.426 / 0.172 & 0.204 / 0.066 & 0.239 / 0.086 & 0.580 / 0.415 & 0.202 / 0.066 & 0.244 / 0.089 & 0.578 / 0.415 & 0.200 / 0.065 & 0.242 / 0.085 & 0.555 / 0.362 & 0.276 / 0.102 & 0.318 / 0.140 \\
0.9 & 0.201 / 0.064 & 0.206 / 0.074 & 0.260 / 0.098 & 0.459 / 0.180 & 0.204 / 0.066 & 0.239 / 0.086 & 0.624 / 0.433 & 0.202 / 0.066 & 0.244 / 0.089 & 0.623 / 0.433 & 0.200 / 0.065 & 0.253 / 0.090 & 0.613 / 0.390 & 0.297 / 0.111 & 0.354 / 0.164 \\
0.95 & 0.201 / 0.064 & 0.206 / 0.074 & 0.261 / 0.099 & 0.472 / 0.184 & 0.204 / 0.066 & 0.239 / 0.086 & 0.624 / 0.433 & 0.202 / 0.066 & 0.244 / 0.089 & 0.623 / 0.433 & 0.200 / 0.065 & 0.254 / 0.091 & 0.633 / 0.398 & 0.313 / 0.117 & 0.375 / 0.180 \\
0.99 & 0.201 / 0.064 & 0.206 / 0.074 & 0.262 / 0.099 & 0.491 / 0.190 & 0.204 / 0.066 & 0.239 / 0.086 & 0.624 / 0.433 & 0.202 / 0.066 & 0.244 / 0.089 & 0.623 / 0.433 & 0.200 / 0.065 & 0.254 / 0.091 & 0.658 / 0.407 & 0.349 / 0.129 & 0.418 / 0.213 \\
\bottomrule
\end{tabular}
}
\end{subtable}
\vspace{1em}
\begin{subtable}{\linewidth}
\centering
\caption{Beta(10, 30)}
\label{tab:truncation_beta_a10.0_b30.0}
\resizebox{\linewidth}{!}{%
\begin{tabular}{lccccccccccccccc}
\toprule
Method & Oracle & \multicolumn{3}{c}{MDP} & \multicolumn{3}{c}{BETEL} & \multicolumn{3}{c}{RETEL} & \multicolumn{3}{c}{Parametric} & Empirical & HC \\
Prior & - & Informative & Non-informative & Misspecified & Informative & Non-informative & Misspecified & Informative & Non-informative & Misspecified & Informative & Non-informative & Misspecified & - & - \\
$c$ &  &  &  &  &  &  &  &  &  &  &  &  &  &  &  \\
\midrule
0.5 & 0.051 / 0.011 & 0.052 / 0.012 & 0.062 / 0.013 & 0.332 / 0.034 & 0.053 / 0.012 & 0.054 / 0.012 & 0.121 / 0.039 & 0.054 / 0.012 & 0.062 / 0.013 & 0.591 / 0.213 & 0.052 / 0.012 & 0.063 / 0.013 & 0.579 / 0.053 & 0.055 / 0.012 & 0.066 / 0.014 \\
0.75 & 0.040 / 0.010 & 0.041 / 0.012 & 0.050 / 0.012 & 0.347 / 0.034 & 0.041 / 0.012 & 0.043 / 0.012 & 0.115 / 0.040 & 0.045 / 0.012 & 0.053 / 0.012 & 0.589 / 0.221 & 0.041 / 0.011 & 0.053 / 0.012 & 0.581 / 0.057 & 0.044 / 0.012 & 0.053 / 0.014 \\
0.9 & 0.036 / 0.010 & 0.038 / 0.012 & 0.047 / 0.012 & 0.354 / 0.034 & 0.038 / 0.012 & 0.040 / 0.012 & 0.113 / 0.041 & 0.044 / 0.012 & 0.052 / 0.012 & 0.589 / 0.226 & 0.037 / 0.012 & 0.049 / 0.012 & 0.582 / 0.062 & 0.041 / 0.012 & 0.048 / 0.014 \\
0.95 & 0.036 / 0.010 & 0.037 / 0.012 & 0.046 / 0.012 & 0.356 / 0.034 & 0.037 / 0.012 & 0.039 / 0.012 & 0.112 / 0.042 & 0.044 / 0.012 & 0.052 / 0.012 & 0.589 / 0.227 & 0.037 / 0.012 & 0.049 / 0.012 & 0.582 / 0.063 & 0.040 / 0.012 & 0.047 / 0.013 \\
0.99 & 0.035 / 0.010 & 0.036 / 0.012 & 0.045 / 0.012 & 0.357 / 0.034 & 0.037 / 0.012 & 0.039 / 0.012 & 0.112 / 0.042 & 0.044 / 0.012 & 0.052 / 0.012 & 0.589 / 0.229 & 0.036 / 0.012 & 0.048 / 0.012 & 0.582 / 0.064 & 0.039 / 0.012 & 0.046 / 0.014 \\
\bottomrule
\end{tabular}
}
\end{subtable}
\vspace{1em}
\begin{subtable}{\linewidth}
\centering
\caption{Beta mixture}
\label{tab:truncation_beta_mixture_default}
\resizebox{\linewidth}{!}{%
\begin{tabular}{lccccccccccccccc}
\toprule
Method & Oracle & \multicolumn{3}{c}{MDP} & \multicolumn{3}{c}{BETEL} & \multicolumn{3}{c}{RETEL} & \multicolumn{3}{c}{Parametric} & Empirical & HC \\
Prior & - & Informative & Non-informative & Misspecified & Informative & Non-informative & Misspecified & Informative & Non-informative & Misspecified & Informative & Non-informative & Misspecified & - & - \\
$c$ &  &  &  &  &  &  &  &  &  &  &  &  &  &  &  \\
\midrule
0.5 & 0.095 / 0.031 & 0.098 / 0.036 & 0.107 / 0.040 & 0.293 / 0.072 & 0.097 / 0.033 & 0.109 / 0.040 & 0.355 / 0.167 & 0.098 / 0.033 & 0.109 / 0.039 & 0.408 / 0.171 & 0.096 / 0.037 & 0.108 / 0.045 & 0.480 / 0.169 & 0.109 / 0.040 & 0.118 / 0.046 \\
0.75 & 0.096 / 0.031 & 0.099 / 0.036 & 0.109 / 0.041 & 0.311 / 0.088 & 0.100 / 0.034 & 0.113 / 0.042 & 0.352 / 0.187 & 0.098 / 0.033 & 0.110 / 0.040 & 0.423 / 0.193 & 0.096 / 0.037 & 0.109 / 0.046 & 0.487 / 0.192 & 0.113 / 0.042 & 0.119 / 0.049 \\
0.9 & 0.096 / 0.031 & 0.099 / 0.036 & 0.110 / 0.041 & 0.320 / 0.095 & 0.103 / 0.034 & 0.116 / 0.043 & 0.352 / 0.198 & 0.098 / 0.033 & 0.110 / 0.040 & 0.432 / 0.206 & 0.096 / 0.037 & 0.110 / 0.046 & 0.492 / 0.205 & 0.116 / 0.044 & 0.122 / 0.052 \\
0.95 & 0.096 / 0.031 & 0.099 / 0.036 & 0.110 / 0.041 & 0.322 / 0.097 & 0.103 / 0.035 & 0.117 / 0.043 & 0.352 / 0.202 & 0.098 / 0.033 & 0.110 / 0.040 & 0.434 / 0.210 & 0.096 / 0.037 & 0.110 / 0.046 & 0.493 / 0.208 & 0.118 / 0.044 & 0.123 / 0.053 \\
0.99 & 0.096 / 0.031 & 0.099 / 0.036 & 0.110 / 0.042 & 0.324 / 0.098 & 0.104 / 0.035 & 0.118 / 0.043 & 0.352 / 0.204 & 0.098 / 0.033 & 0.110 / 0.040 & 0.436 / 0.213 & 0.096 / 0.037 & 0.110 / 0.046 & 0.494 / 0.210 & 0.119 / 0.044 & 0.125 / 0.053 \\
\bottomrule
\end{tabular}
}
\end{subtable}
\end{sidewaystable}

\subsubsection{MDP mixing parameter $\kappa$}
\label{app:kappa}

The MDP predictive takes the form (\cref{eq:dirichlet})

$$
\mathbb{Q}_{n+1 \mid n}=\frac{\kappa}{\kappa+n} \mathbb{H}_{n+1 \mid n}+\frac{n}{\kappa+n} \mathbb{P}_n,
$$

where $\kappa$ controls the prior-driven component's weight relative to the empirical distribution and hence how long the predictive remains prior-driven. For every fixed $\kappa$, this component vanishes asymptotically.

The distinction between empirical and MDP is therefore primarily finite-sample, which is relevant in sequential settings. At small sample sizes, the empirical predictive is supported on only a few observations and can yield noisy betting coefficients. When accurate external information is available, the Bayesian component can regularise this predictive, improve early evidence growth, and thus reduce stopping times. For example, in the AttributedQA sequential test in \cref{sec:experiments}, MDP requires 52.18 labels on average, compared with 65.24 for the empirical predictive, a reduction of about 20\%, and wins or ties against it in 98 of 100 paired replicates. Conversely, an inaccurate prior can reduce finite-sample efficiency relative to empirical.

To quantify this trade-off, we provide additional results for $\operatorname{Bernoulli}(0.5)$, $\operatorname{Beta}(10,30)$, and the beta mixture, varying $\kappa\in \{5, 20, 50, 200\}$. \cref{tab:kappa} compares average CS widths at $n=100$ and $n=1000$ for all three data-generating processes.

In the Bernoulli experiment, over jointly non-empty paired replicates, an accurate informative prior makes MDP about $27\%$ narrower than empirical for $\kappa=5$ and $34$--$37\%$ narrower for $\kappa=50$ across the two horizons. These results show that the finite-sample benefit can remain substantial well beyond the very earliest observations.

The same experiment also makes the cost of misspecification explicit. Larger $\kappa$ generally magnifies both the potential gain and the potential loss. Although the prior weight vanishes for fixed $\kappa$, early differences in accumulated wealth can remain visible at substantially larger sample sizes. In the low-variance Beta experiment, there is comparatively little room for an accurate prior to improve upon empirical, while a severely misspecified prior can cause a substantial efficiency loss. The contrast between Bernoulli and Beta is consistent with the intuition that prior regularisation is most useful when the empirical predictive is intrinsically noisy and takes time to adapt to the data.

A practical calibration is to specify the desired weight $w_{\mathrm{ref}}$ assigned to the prior-driven predictive at a reference sample size $n_{\mathrm{ref}}$, giving $\kappa=w_{\mathrm{ref}} /\left(1-w_{\mathrm{ref}}\right) n_{\mathrm{ref}}$. Outside this ablation, we use $\kappa=50$, corresponding to prior predictive weights of 0.5 at $n=50$ and approximately 0.05 at $n=1000$, so that the prior can materially affect the early sequential regime while the empirical predictive rapidly becomes dominant.

\begin{table}
\centering
\caption{Average CS width at $n=100$ / $n=1000$ for each MDP mixing parameter $\kappa$. Widths are averaged over the non-empty CSs; across all CSs shown, at most 6\% are empty at $n=100$ or $n=1000$ (0.9\% on average).}
\label{tab:kappa}
\begin{subtable}{\linewidth}
\centering
\caption{Bernoulli(0.5)}
\label{tab:kappa_bernoulli_p0.5}
\resizebox{\linewidth}{!}{%
\begin{tabular}{lccccc}
\toprule
Method & Oracle & \multicolumn{3}{c}{MDP} & Empirical \\
Prior & - & Informative & Non-informative & Misspecified & - \\
$\kappa$ &  &  &  &  &  \\
\midrule
5 & 0.201 / 0.064 & 0.229 / 0.084 & 0.278 / 0.107 & 0.328 / 0.119 & 0.313 / 0.117 \\
20 & 0.201 / 0.064 & 0.212 / 0.077 & 0.267 / 0.102 & 0.401 / 0.147 & 0.313 / 0.117 \\
50 & 0.201 / 0.064 & 0.206 / 0.074 & 0.261 / 0.099 & 0.472 / 0.184 & 0.313 / 0.117 \\
200 & 0.201 / 0.064 & 0.201 / 0.069 & 0.256 / 0.095 & 0.569 / 0.258 & 0.313 / 0.117 \\
\bottomrule
\end{tabular}
}
\end{subtable}
\vspace{1em}
\begin{subtable}{\linewidth}
\centering
\caption{Beta(10, 30)}
\label{tab:kappa_beta_a10.0_b30.0}
\resizebox{\linewidth}{!}{%
\begin{tabular}{lccccc}
\toprule
Method & Oracle & \multicolumn{3}{c}{MDP} & Empirical \\
Prior & - & Informative & Non-informative & Misspecified & - \\
$\kappa$ &  &  &  &  &  \\
\midrule
5 & 0.036 / 0.010 & 0.038 / 0.012 & 0.042 / 0.012 & 0.118 / 0.017 & 0.040 / 0.012 \\
20 & 0.036 / 0.010 & 0.037 / 0.012 & 0.044 / 0.012 & 0.244 / 0.025 & 0.040 / 0.012 \\
50 & 0.036 / 0.010 & 0.037 / 0.012 & 0.046 / 0.012 & 0.356 / 0.034 & 0.040 / 0.012 \\
200 & 0.036 / 0.010 & 0.037 / 0.012 & 0.048 / 0.012 & 0.494 / 0.049 & 0.040 / 0.012 \\
\bottomrule
\end{tabular}
}
\end{subtable}
\vspace{1em}
\begin{subtable}{\linewidth}
\centering
\caption{Beta mixture}
\label{tab:kappa_beta_mixture_default}
\resizebox{\linewidth}{!}{%
\begin{tabular}{lccccc}
\toprule
Method & Oracle & \multicolumn{3}{c}{MDP} & Empirical \\
Prior & - & Informative & Non-informative & Misspecified & - \\
$\kappa$ &  &  &  &  &  \\
\midrule
5 & 0.096 / 0.031 & 0.107 / 0.039 & 0.113 / 0.042 & 0.159 / 0.055 & 0.118 / 0.044 \\
20 & 0.096 / 0.031 & 0.102 / 0.038 & 0.111 / 0.042 & 0.243 / 0.075 & 0.118 / 0.044 \\
50 & 0.096 / 0.031 & 0.099 / 0.036 & 0.110 / 0.041 & 0.322 / 0.097 & 0.118 / 0.044 \\
200 & 0.096 / 0.031 & 0.097 / 0.035 & 0.110 / 0.042 & 0.425 / 0.137 & 0.118 / 0.044 \\
\bottomrule
\end{tabular}
}
\end{subtable}
\end{table}

\subsubsection{Cellwise-certified inversion}
\label{app:cellwise_experiments}

We compare the pointwise and cellwise implementations on the same length-$1000$ data sequence and method seed in each of 100 paired replicates of all six synthetic experiments. We use $\alpha=0.1$ and a uniform partition of $[0,1]$ into $G=500$ cells of width $1/500$, matching the pointwise grid resolution; all other method, prior, and running-intersection settings match the main synthetic suite. \Cref{tab:cellwise_certified_to_base_width_ratio_n100,tab:cellwise_certified_to_base_width_ratio_n1000} report mean-width ratios over replicates in which both CSs are non-empty: cellwise/pointwise for the predictive methods and, because UP uses the same grid in the main experiments to match their numerical resolution, direct-interval/grid for UP. For the predictive methods, these ratios range from 0.983 to 1.045 at $n=100$ and from 0.942 to 1.085 at $n=1000$, so the widths differ by at most 9\%. Overall, this close agreement indicates that the pointwise grid implementation used for the predictive methods in the other experiments has little effect on the reported widths in these settings.

\begin{table}
\caption{Cellwise-to-pointwise mean-width ratio at $n=100$. For UP, the analogous direct-interval/grid ratio is included for reference. Averages are taken over the replicates in which both CSs are non-empty (see \cref{tab:cellwise_nonempty}); across all CSs, at most 2\% are empty at $n=100$ (0.1\% on average).}
\label{tab:cellwise_certified_to_base_width_ratio_n100}
\begin{tabular}{lrrrrrr}
\toprule
 & Bernoulli(0.1) & Bernoulli(0.5) & Beta(0.5, 0.5) & Beta(1, 1) & Beta(10, 30) & Beta mixture \\
\midrule
Emp. & 0.998 & 0.996 & 0.996 & 0.996 & 0.984 & 0.994 \\
Oracle & 0.997 & 0.998 & 0.997 & 0.997 & 0.995 & 0.995 \\
Param. (good) & 1.003 & 0.998 & 0.998 & 0.999 & 0.990 & 0.994 \\
Param. (bad) & 0.999 & 1.003 & 0.998 & 0.998 & 0.998 & 0.998 \\
Param. (noninf) & 1.045 & 0.999 & 0.997 & 0.995 & 0.984 & 0.994 \\
MDP (good) & 0.996 & 0.998 & 0.996 & 0.998 & 0.983 & 0.991 \\
MDP (bad) & 1.005 & 1.001 & 0.999 & 0.996 & 0.997 & 0.997 \\
MDP (noninf) & 1.026 & 0.999 & 0.998 & 0.995 & 0.986 & 0.993 \\
BETEL (good) & 0.997 & 0.998 & 0.998 & 0.997 & 0.990 & 0.994 \\
BETEL (bad) & 1.021 & 1.004 & 0.998 & 0.997 & 0.991 & 0.998 \\
BETEL (noninf) & 0.996 & 0.998 & 0.997 & 0.997 & 0.988 & 0.996 \\
RETEL (good) & 0.995 & 0.997 & 0.998 & 0.996 & 0.985 & 0.994 \\
RETEL (bad) & 1.019 & 1.004 & 0.999 & 0.998 & 0.998 & 0.998 \\
RETEL (noninf) & 0.999 & 0.999 & 0.997 & 0.995 & 0.987 & 0.994 \\
UP & 0.987 & 0.992 & 0.989 & 0.985 & 0.959 & 0.982 \\
\bottomrule
\end{tabular}
\end{table}

\begin{table}
\caption{Cellwise-to-pointwise mean-width ratio at $n=1000$. For UP, the analogous direct-interval/grid ratio is included for reference. Averages are taken over the replicates in which both CSs are non-empty (see \cref{tab:cellwise_nonempty}); across all CSs, at most 6\% are empty at $n=1000$ (0.8\% on average).}
\label{tab:cellwise_certified_to_base_width_ratio_n1000}
\begin{tabular}{lrrrrrr}
\toprule
 & Bernoulli(0.1) & Bernoulli(0.5) & Beta(0.5, 0.5) & Beta(1, 1) & Beta(10, 30) & Beta mixture \\
\midrule
Emp. & 0.992 & 0.996 & 0.995 & 0.994 & 0.966 & 0.990 \\
Oracle & 0.988 & 0.992 & 0.994 & 0.988 & 0.942 & 0.981 \\
Param. (good) & 1.012 & 0.993 & 0.984 & 0.988 & 0.960 & 0.987 \\
Param. (bad) & 1.002 & 1.001 & 0.998 & 0.996 & 0.984 & 0.995 \\
Param. (noninf) & 1.085 & 0.997 & 0.992 & 0.991 & 0.949 & 0.991 \\
MDP (good) & 0.993 & 0.992 & 0.990 & 0.987 & 0.968 & 0.979 \\
MDP (bad) & 1.000 & 0.999 & 0.998 & 0.997 & 0.982 & 0.997 \\
MDP (noninf) & 1.003 & 0.997 & 0.992 & 0.991 & 0.952 & 0.990 \\
BETEL (good) & 0.987 & 0.992 & 0.992 & 0.984 & 0.966 & 0.979 \\
BETEL (bad) & 1.042 & 1.002 & 0.998 & 0.995 & 0.981 & 0.995 \\
BETEL (noninf) & 0.994 & 0.996 & 0.994 & 0.994 & 0.958 & 0.994 \\
RETEL (good) & 0.988 & 0.992 & 0.988 & 0.989 & 0.944 & 0.981 \\
RETEL (bad) & 1.040 & 1.002 & 0.998 & 0.996 & 0.997 & 0.995 \\
RETEL (noninf) & 0.992 & 0.996 & 0.993 & 0.991 & 0.949 & 0.984 \\
UP & 0.966 & 0.977 & 0.969 & 0.961 & 0.833 & 0.951 \\
\bottomrule
\end{tabular}
\end{table}

\begin{table}
\caption{Proportion of replicates in which the cellwise CS is empty. Methods and scenarios not listed have no empty cellwise CS; the corresponding pointwise rates differ by at most one percentage point. Across both implementations, at most 6\% are empty (0.5\% on average).}
\label{tab:cellwise_nonempty}
\begin{tabular}{llrrrrr}
\toprule
 &  & Bernoulli(0.1) & Bernoulli(0.5) & Beta(0.5, 0.5) & Beta(1, 1) & Beta mixture \\
\midrule
$n=100$ & Emp. & 0.00 & 0.02 & 0.01 & 0.00 & 0.00 \\
 & Param. (noninf) & 0.00 & 0.01 & 0.00 & 0.00 & 0.00 \\
 & MDP (noninf) & 0.00 & 0.01 & 0.00 & 0.00 & 0.00 \\
 & BETEL (good) & 0.00 & 0.00 & 0.02 & 0.00 & 0.01 \\
 & BETEL (bad) & 0.00 & 0.00 & 0.02 & 0.00 & 0.00 \\
$n=1000$ & Emp. & 0.03 & 0.05 & 0.03 & 0.01 & 0.06 \\
 & Param. (bad) & 0.00 & 0.02 & 0.00 & 0.00 & 0.00 \\
 & Param. (noninf) & 0.00 & 0.05 & 0.02 & 0.00 & 0.02 \\
 & MDP (bad) & 0.01 & 0.02 & 0.01 & 0.00 & 0.03 \\
 & MDP (noninf) & 0.02 & 0.05 & 0.02 & 0.00 & 0.03 \\
 & BETEL (good) & 0.00 & 0.00 & 0.03 & 0.00 & 0.02 \\
 & BETEL (bad) & 0.00 & 0.00 & 0.02 & 0.00 & 0.00 \\
 & BETEL (noninf) & 0.00 & 0.00 & 0.02 & 0.00 & 0.04 \\
 & RETEL (noninf) & 0.02 & 0.02 & 0.01 & 0.00 & 0.02 \\
 & UP & 0.02 & 0.02 & 0.01 & 0.00 & 0.03 \\
\bottomrule
\end{tabular}
\end{table}

\subsection{Best-arm identification for LLM evaluation}
\label{app:LLM_eval_results}

\subsubsection{Full baseline comparison}

\Cref{tab:method_comparison_baselines} reports the full results for the best-arm identification experiment in \cref{sec:experiments}.
\begin{table}[H]
    \centering
    \caption{Full-baseline version of the LLM best-arm identification experiment. Across all methods, no CS is empty at the stopping time.}
    \label{tab:method_comparison_baselines}
    \begin{tabular}{lccc}
        \toprule
        Method & Avg Pulls (s.e.) & Avg rank (s.e.) & W/T/L vs Empirical \\
        \midrule
        MDP & 91.68 (3.90) & 1.95 (0.13) & 38/5/7 \\
        Parametric & 95.92 (4.75) & 2.66 (0.26) & 37/1/12 \\
        BETEL & 99.36 (3.62) & 3.28 (0.17) & 23/7/20 \\
        Empirical & 101.64 (4.03) & 3.58 (0.21) & -- \\
        RETEL & 109.76 (3.63) & 4.95 (0.23) & 18/2/30 \\
        UP & 110.68 (3.53) & 5.18 (0.16) & 4/3/43 \\
        PRECiSE-CO96 & 123.84 (4.16) & 6.54 (0.13) & 0/0/50 \\
        HC & 147.96 (3.51) & 7.86 (0.09) & 0/0/50 \\
        \bottomrule
    \end{tabular}
\end{table}

Overall, the results presented in the main text stand, with the parametric Bayes-assisted methods typically outperforming the prior-free baselines on this task.
The additional prior-free baselines, PRECiSE-CO96 and HC, decisively underperform all methods reported in the main text. The nonparametric Bayesian predictive performs similarly to the empirical predictive. All this suggests that the Beta working model is a good fit for this example.

\subsubsection{Paired significance comparisons}
\label{app:LLM_eval_significance}

To establish the statistical significance of the performance disparities among the evaluated methods in \cref{tab:method_comparison_spear,tab:method_comparison_baselines}, we conduct pairwise comparisons across all unique combinations of algorithms.
For every unordered pair, denoted as $\{A,B\}$, we employ a two-sided Wilcoxon signed-rank test \citep{Wilcoxon1945} to detect structural differences in the number of pulls required to converge.
To rigorously control the family-wise error rate stemming from this testing matrix, the resulting $p$-values are adjusted utilizing the Holm-Bonferroni step-down procedure \citep{Holm1979}, bounding the global significance level at $\alpha=0.05$.
When a statistically significant difference is established, the superior algorithm is subsequently identified as the one exhibiting the lower mean number of pulls. Ties between paired samples are handled via the zero-split method. The results are presented in \cref{tab:sign_spear}.
\begin{table}[H]
\centering
\begin{tabular}{l|cccccccc}
\hline
 & \rotatebox{90}{MDP} & \rotatebox{90}{Parametric} & \rotatebox{90}{BETEL} & \rotatebox{90}{Empirical} & \rotatebox{90}{RETEL} & \rotatebox{90}{UP} & \rotatebox{90}{PRECiSE-CO96} & \rotatebox{90}{HC} \\
\hline
MDP & - & \textbf{Yes} & \textbf{Yes} & \textbf{Yes} & \textbf{Yes} & \textbf{Yes} & \textbf{Yes} & \textbf{Yes} \\
Parametric & No & - & \textbf{Yes} & \textbf{Yes} & \textbf{Yes} & \textbf{Yes} & \textbf{Yes} & \textbf{Yes} \\
BETEL & No & No & - & ? & \textbf{Yes} & \textbf{Yes} & \textbf{Yes} & \textbf{Yes} \\
Empirical & No & No & ? & - & ? & \textbf{Yes} & \textbf{Yes} & \textbf{Yes} \\
RETEL & No & No & No & ? & - & ? & \textbf{Yes} & \textbf{Yes} \\
UP & No & No & No & No & ? & - & \textbf{Yes} & \textbf{Yes} \\
PRECiSE-CO96 & No & No & No & No & No & No & - & \textbf{Yes} \\
HC & No & No & No & No & No & No & No & - \\
\hline
\end{tabular}
\caption{Pairwise statistical comparison: Is Method A (row) significantly faster than Method B (column)? Across all methods, no CS is empty at the stopping time.}
\label{tab:sign_spear}
\end{table}

\subsubsection{Prior misspecification}
\label{app:LLM_bad_prior}

Misspecification of the working prior does not affect statistical validity, but it can reduce efficiency.
For the LUCB experiment of \cref{sec:experiments} and \cref{app:LLM_eval_results}, we centre the priors at the well-specified mean minus one standard deviation of the calibration data used to compute that mean, while leaving everything else unchanged. \Cref{tab:method_comparison_bad_prior} reports average LUCB stopping times and paired ranks for accurate and inaccurate priors. All methods identify the same LLM, but the inaccurate priors require considerably more LLM calls.

\begin{table}[H]
    \centering
    \caption{LLM best-arm identification experiment with informative and misspecified priors. Across all methods, no CS is empty at the stopping time.}
    \label{tab:method_comparison_bad_prior}
    \begin{tabular}{lcc}
        \toprule
        Method & Avg Pulls (s.e.) & Avg rank (s.e.) \\
        \midrule
        MDP & 91.68 (3.90) & 1.67 (0.09)\\
        Parametric & 95.92 (4.75) & 2.16 (0.13)\\
        BETEL & 99.36 (3.62) & 2.73 (0.15) \\
        RETEL & 109.76 (3.63) & 3.90 (0.16) \\
        MDP (bad) & 149.64 (7.26) & 5.00 (0.11) \\
        BETEL (bad) & 171.40 (10.76) & 5.85 (0.08) \\
        Parametric (bad) & 246.04 (19.21) & 7.19 (0.16)\\
        RETEL (bad) & 229.84 (14.54) & 7.50 (0.08) \\
        \bottomrule
    \end{tabular}
\end{table}

\subsection{Prediction-powered inference}\label{app:ppi_results}

\subsubsection{Full baseline comparison}
\FloatBarrier

\Cref{fig:ppi_baselines} shows the full-baseline version of the prediction-powered inference experiment in \cref{fig:ppi_main}.
\begin{figure}
    \centering
    \includegraphics[width=\textwidth]{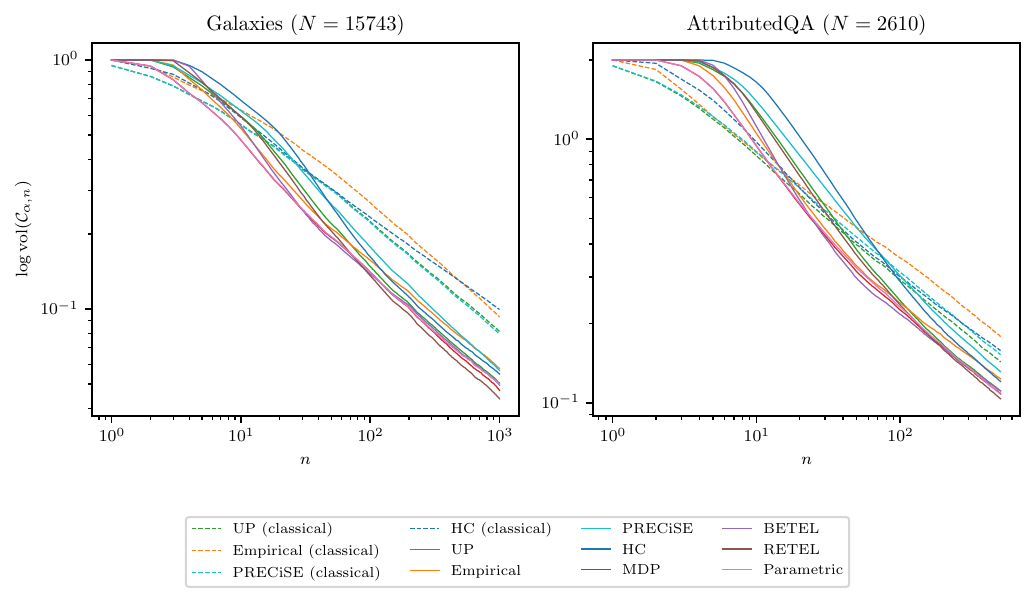}
    \caption{
        Full-baseline version of prediction-powered mean estimation experiment in \cref{fig:ppi_main}.
        The left and right panels show average CS length over 100 replicates for the \textsc{galaxies} and \textsc{attributedQA} datasets.
        Curves labelled ``classical'' use only labelled observations, while the remaining curves use the PPI construction.
        Empty CSs are counted as width zero; across all CSs shown, at most 6\% are empty at the final sample size (0.5\% on average).
    }
    \label{fig:ppi_baselines}
\end{figure}
While the results remain consistent with the main text, it is interesting to see that the parametric Bayesian predictive performs well on this task, especially on the \textsc{attributedQA} dataset, where the working model seems to be a good fit for the data distribution.

\Cref{tab:ppi_attributedqa_test_full} reports the full comparison for the sequential test of $H_0:\theta^\star\leq 0$ on the \textsc{attributedQA} dataset, including additional classical and PPI baselines not reported in the main text.
\begin{table}
    \centering
    \caption{
        Full comparison for the AttributedQA sequential test of $H_0:\theta^\star\leq 0$.
        Avg labels is the average number of labelled examples required to reject the null over $100$ replicates.
        Avg rank ranks methods by stopping time within each replicate.
        W/T/L gives paired wins, ties, and losses against the PPI empirical predictive baseline.
        Across all methods, at most 1\% of CSs are empty at the stopping time (0.1\% on average).
    }
    \label{tab:ppi_attributedqa_test_full}
    \begin{tabular}{lccc}
        \toprule
        Method & Avg labels (s.e.) & Avg rank (s.e.) & W/T/L vs Empirical \\
        \midrule
        Parametric & 49.43 (3.23) & 2.91 (0.16) & 80 / 18 / 2 \\
        BETEL & 50.25 (3.37) & 3.33 (0.18) & 71 / 15 / 14 \\
        MDP & 52.18 (3.50) & 3.49 (0.18) & 82 / 16 / 2 \\
        RETEL & 55.03 (2.49) & 4.24 (0.22) & 62 / 8 / 30 \\
        Empirical & 65.24 (4.12) & 5.88 (0.25) & -- \\
        UP & 69.73 (3.24) & 6.48 (0.16) & 38 / 4 / 58 \\
        HC & 87.82 (2.94) & 8.42 (0.18) & 18 / 5 / 77 \\
        PRECiSE-CO96 & 89.87 (3.71) & 8.90 (0.16) & 9 / 1 / 90 \\
        UP (classical) & 106.27 (7.48) & 7.24 (0.32) & 30 / 0 / 70 \\
        PRECiSE-CO96 (classical) & 112.23 (7.46) & 8.43 (0.32) & 27 / 0 / 73 \\
        HC (classical) & 114.01 (9.02) & 7.78 (0.33) & 25 / 2 / 73 \\
        Empirical (classical) & 188.91 (10.06) & 10.93 (0.30) & 11 / 0 / 89 \\
        \bottomrule
    \end{tabular}
\end{table}
While remaining consistent with the comments in the main text, the full comparison highlights again the strong performance of the parametric Bayesian predictive on this task, where it competes closely with MDP.

\FloatBarrier
\subsubsection{Prior misspecification}
\label{app:ppi_bad_prior}

For the AttributedQA experiment described in \cref{sec:experiments} and \cref{app:ppi_details}, we shift the prior means from 0.5 to 0.45 while leaving the variances unchanged. Although the change in the mean is modest, these priors are strongly concentrated and therefore encode a substantially shifted belief about the quality of the black-box predictor. The table below reports the average capped number of labels to reject $H_0 : \theta^* \leq0$ and paired ranks across 100 replicates. All four methods reject within the 500-label budget in every replicate. The shifted priors nevertheless delay rejection substantially for some methods. MDP exhibits the smallest degradation among the prior-assisted methods, consistent with its increasing reliance on the empirical component as additional observations are collected.

\begin{table}[H]
    \centering
    \caption{AttributedQA PPI experiment with informative and misspecified priors. Across all methods, no CS is empty at the stopping time.}
    \label{tab:ppi_bad_prior}
    \begin{tabular}{lcc}
        \toprule
        Method & Avg labels (s.e.) & Avg rank (s.e.) \\
        \midrule
        Parametric & 49.43 (3.23) & 2.08 (0.07) \\
        BETEL & 50.25 (3.37) & 2.46 (0.11) \\
        MDP & 52.18 (3.50) & 2.63 (0.11) \\
        RETEL & 55.03 (2.49) & 3.18 (0.13) \\
        MDP (bad) & 77.16 (3.01) & 5.11 (0.06) \\
        BETEL (bad) & 87.07 (4.18) & 5.79 (0.12) \\
        Parametric (bad) & 107.92 (3.36) & 7.18 (0.05) \\
        RETEL (bad) & 112.75 (2.79) & 7.59 (0.06) \\
        \bottomrule
    \end{tabular}
\end{table}

%%%%%%%%%%%%%%%%%%%%%%%%%%%%%%%%%%%%%%%%%%%%%%%%%%%%%%%%%%%%
%\clearpage
%\newpage

%\input{checklist}
\end{document}